\documentclass{article}
\usepackage{palatino,epsfig,latexsym}

\usepackage{dsfont,amsmath,amssymb}
\usepackage{xspace}
\usepackage[latin1]{inputenc}
\usepackage{epsfig,wrapfig}

\usepackage{subfigure}          
\usepackage[margin=0pt,font=small,labelfont=bf,format=plain,indention=.5cm]{caption}  

\usepackage{balance}
\usepackage{url}
\usepackage{tikz}

\usepackage[numbers,sort&compress,longnamesfirst,sectionbib]{natbib}
\newtheorem{lemma}{Lemma}

\newtheorem{definition}{Definition}
\newtheorem{theorem}[lemma]{Theorem}

\newcommand{\eg}{e.\,g.\xspace}
\newcommand{\R}{\mathbb{R}}
\newcommand{\Oh}[1]{\mathcal{O}(#1)}

\newcommand{\defref}[1]{Definition~\ref{def:#1}}
\newcommand{\thmref}[1]{Theorem~\ref{thm:#1}}
\newcommand{\lemref}[1]{Lemma~\ref{lem:#1}}

\newcommand{\secref}[1]{Section~\ref{sec:#1}}

\newcommand{\figref}[1]{Figure~\ref{fig:#1}}
\newcommand{\figrefs}[2]{Figures~\ref{fig:#1} and~\ref{fig:#2}}

\newcommand{\cX}{\mathcal{X}}

\newcommand{\ee}{\varepsilon}
\renewcommand{\epsilon}{\ee}

\newcommand{\ie}{i.\,e.\xspace}
\newcommand{\wlo}{w.\,l.\,o.\,g.\xspace}

\renewcommand{\P}{\textup{\textbf{P}}\xspace}
\newcommand{\NP}{\textup{\textbf{NP}}\xspace}

\newcommand{\HYP}{\operatorname{HYP}}
\newcommand{\Vol}[1]{\textsc{vol}(#1)}

\newcommand{\VolBigg}[1]{\textsc{vol}\Bigg(#1\Bigg)}

\newcommand{\fasy}{\ensuremath{f^{\textup{asy}}}}
\newcommand{\fsym}{\ensuremath{f^{\textup{sym}}}}

\newcommand{\APP}[1]{\ensuremath{\operatorname{APP}(#1)}}
\newcommand{\OPTAPP}[2]{\ensuremath{X_{\mathrm{opt}}^{\mathrm{APP}}(#1,#2)}}
\newcommand{\OPTAPPr}[2]{\ensuremath{\widehat{X}_{\mathrm{opt}}^{\mathrm{APP}}(#1,#2)}}

\newcommand{\OPTHYP}[2]{\ensuremath{X_{\mathrm{opt}}^{\mathrm{HYP}}(#1,#2)}}
\newcommand{\OPTHYPr}[2]{\ensuremath{\widehat{X}_{\mathrm{opt}}^{\mathrm{HYP}}(#1,#2)}}

\newcommand{\OPT}[2]{\ensuremath{X_{\mathrm{opt}}(#1,#2)}}

\usepackage{setspace}
\usepackage{color}   

\newcommand{\ratio}{\delta}

\def\argmax{\operatornamewithlimits{argmax}}
\def\argmin{\operatornamewithlimits{argmin}}

\newcommand{\distfigwidth}{.21\textwidth}
\newcommand{\doubledistfigwidth}{.42\textwidth}
\newcommand{\subfigwidth}{.18\textwidth}

\makeatletter
\DeclareRobustCommand{\qed}{\ifmmode\mathqed\else\leavevmode\unskip\penalty9999\hbox{}\nobreak\hfill\quad\hbox{\qedsymbol}\fi}
\let\QED@stack\@empty
\let\qed@elt\relax
\newcommand{\pushQED}[1]{\toks@{\qed@elt{#1}}\@temptokena\expandafter{\QED@stack}\xdef\QED@stack{\the\toks@\the\@temptokena}}
\newcommand{\popQED}{\begingroup\let\qed@elt\popQED@elt \QED@stack\relax\relax\endgroup}
\def\popQED@elt#1#2\relax{#1\gdef\QED@stack{#2}}
\newcommand{\qedhere}{\begingroup \let\mathqed\math@qedhere\let\qed@elt\setQED@elt \QED@stack\relax\relax \endgroup}
\newif\ifmeasuring@
\newif\iffirstchoice@ \firstchoice@true
\def\setQED@elt#1#2\relax{\ifmeasuring@\else \iffirstchoice@ \gdef\QED@stack{\qed@elt{}#2}\fi\fi#1}
\newcommand{\mathqed}{\quad\hbox{\qedsymbol}}
\def\linebox@qed{\hfil\hbox{\qedsymbol}\hfilneg}
\def\math@qedhere{\@ifundefined{\@currenvir @qed}{\qed@warning\quad\hbox{\qedsymbol}}{\@xp\aftergroup\csname\@currenvir @qed\endcsname}}
\def\displaymath@qed{\relax\ifmmode\ifinner\aftergroup\linebox@qed\else\eqno\let\eqno\relax \let\leqno\relax \let\veqno\relax\hbox{\qedsymbol}\fi\else\aftergroup\linebox@qed\fi}
\@xp\let\csname equation*@qed\endcsname\displaymath@qed
\def\equation@qed{
  \iftagsleft@\hbox{\phantom{\quad\qedsymbol}}\gdef\alt@tag{\rlap{\hbox to\displaywidth{\hfil\qedsymbol}}\global\let\alt@tag\@empty}
  \else\gdef\alt@tag{\global\let\alt@tag\@empty\vtop{\ialign{\hfil####\cr\tagform@\theequation\cr\qedsymbol\cr}}\setbox\z@}
  \fi
}
\def\qed@tag{\global\tag@true \nonumber&\omit\setboxz@h {\strut@ \qedsymbol}\tagsleft@false\place@tag@gather\kern-\tabskip\ifst@rred \else \global\@eqnswtrue \fi \global\advance\row@\@ne \cr}
\def\split@qed{\def\endsplit{\crcr\egroup \egroup \ctagsplit@false \rendsplit@\aftergroup\align@qed}}
\def\align@qed{\ifmeasuring@ \tag*{\qedsymbol}\else \let\math@cr@@@\qed@tag\fi}
\@xp\let\csname align*@qed\endcsname\align@qed
\@xp\let\csname gather*@qed\endcsname\align@qed
\def\@tempb#1 v#2.#3\@nil{#2}
\ifnum\@xp\@xp\@xp\@tempb\csname ver@amsmath.sty\endcsname v0.0\@nil<\tw@\def\@tempa{TT}\else\def\@tempa{TF}\fi
\if\@tempa\renewcommand{\math@qedhere}{\quad\hbox{\qedsymbol}}\fi
\newcommand{\openbox}{\leavevmode\hbox to.77778em{\hfil\vrule\vbox to.675em{\hrule width.6em\vfil\hrule}\vrule\hfil}}
\DeclareRobustCommand{\textsquare}{\begingroup\usefont{U}{msa}{m}{n}\thr@@\endgroup}
\providecommand{\qedsymbol}{\openbox}
\newenvironment{proof}[1][\proofname]{\par\pushQED{\qed}\normalfont\topsep6\p@\@plus6\p@\relax\trivlist\item[\hskip\labelsep\itshape #1\@addpunct{.}]\ignorespaces}{\popQED\endtrivlist\@endpefalse}
\providecommand{\proofname}{Proof}
\makeatother

\begin{document}

\title{\bf Multiplicative Approximations,\\Optimal Hypervolume Distributions,\\and the Choice of the Reference Point}  

\author{
Tobias Friedrich\\
Lehrstuhl f{\"u}r Theoretische Informatik I\\
Fakult{\"a}t f{\"u}r Mathematik und Informatik\\ 
Friedrich-Schiller-Universit{\"a}t Jena, Germany
\and 
Frank Neumann\\
Optimisation and Logistics\\
School of Computer Science\\
The University of Adelaide, Australia
\and 
Christian Thyssen\\
 Lehrstuhl~2, Fakult{\"a}t f{\"u}r Informatik\\
Technische Universit{\"a}t Dortmund, Germany
}

\maketitle


{

\begin{abstract}
    Many optimization problems arising in applications have to consider several 
    objective functions at the same time. 
Evolutionary algorithms seem to be a very natural choice for dealing with multi-objective problems as the population of such an algorithm can be used to represent the trade-offs with respect to the given objective functions.
   In this paper, we contribute to the 
    theoretical understanding of  evolutionary algorithms for multi-objective problems. 
    We consider indicator-based algorithms whose goal is to maximize the hypervolume for a 
    given problem by distributing $\mu$ points on the Pareto front.
    To gain new theoretical insights into the behavior of hypervolume-based algorithms we compare their optimization goal to the goal of achieving an optimal multiplicative approximation ratio. Our studies are carried out for different Pareto front shapes of bi-objective problems.
    For 
    the class of linear fronts and a class of convex fronts, we prove that 
    maximizing the hypervolume gives the best possible approximation ratio when assuming that the extreme points have to be included in both distributions of the points on the Pareto front. Furthermore, we investigate the choice of the reference point on the approximation behavior of hypervolume-based approaches and 
    examine Pareto fronts of different shapes by numerical calculations.
\end{abstract}



\section{Introduction}


Multi-ob\-ject\-ive optimization \cite{ehrg2005a} deals with the task of optimizing several objective 
functions at the same time. Here,
several attributes of a given problem are employed as objective functions and
are used to define a partial order, called preference order, on the
solutions, for which the set of minimal
(maximal) elements is sought. 
 Usually, the objective functions are conflicting, which means 
that improvements with respect to one function can only be achieved when 
impairing the solution quality with respect to another objective function. 
Due to this, such problems usually do not have a single optimal function value. Instead, there is a set of optimal objective vectors which represents the different trade-offs of the different objective functions.
Solutions that cannot be improved with respect to any function without impairing 
another one are called \emph{Pareto-optimal solutions}. The objective vectors 
associated with these solutions are called \emph{Pareto-optimal objective vectors} and 
the set of all these objective vectors constitutes the \emph{Pareto front}.

In contrast to single-objective optimization, in multi-objective optimization the task is not to compute a single optimal solution but a set of solutions representing the different trade-offs with respect to the given objective functions.
Most of the best-known single-objective polynomially solvable
problems like shortest path or minimum spanning tree become
NP-hard when at least two weight functions have to be
optimized at the same time.  In this sense,
multi-objective optimization is generally considered as more 
difficult than single-objective optimization.


Another, more promising, approach to deal with multi-objective optimization problems is to apply general stochastic search algorithms that evolve a set of possible solutions into a set of solutions that represent the trade-offs with respect to the objective functions. Well-known approaches in this field are evolutionary algorithms \cite{bfm1997a} and ant colony optimization \cite{Dorigo2004}. Especially, multi-objective evolutionary algorithms (MOEAs) have been shown to be very successful when dealing with multi-objective problems~\cite{cvl2002a,deb2001a}. Evolutionary algorithms work with a set of solutions called population which is evolved over time by applying crossover and mutation operators to produce new possible solutions for the underlying multi-objective problem. 
Due to this population-based approach, they  are in a natural way well-suited for 
dealing with multi-ob\-ject\-ive optimization problems.

A major problem when dealing with multi-objective optimization problems is that the number of different trade-offs may be too large. This implies that not \emph{all} trade-offs can be computed efficiently, \ie, in polynomial time.  In the discrete case the 
Pareto front may grow exponentially with respect to the problem size  and may be even infinite in the continuous case. 
In such a case, 
it is not possible to compute the whole Pareto front efficiently and the goal is to 
compute a good approximation consisting of a not too large set of Pareto-optimal solutions.
It has been observed empirically that MOEAs are able to obtain good approximations for a wide range of multi-objective optimization problems.

The aim of this paper is to contribute to the theoretical understanding of MOEAs in particular with respect to their approximation behavior.
Many researchers have worked on how to use evolutionary algorithms for multi-ob\-ject\-ive optimization problems and 
how to find solutions being close to the Pareto front and covering all parts of the Pareto front.
However, often the optimization goal remains rather unclear as it is not stated explicitly how to measure the quality of an approximation that a proposed algorithm should achieve. 

One popular approach to achieve the mentioned objectives is to use the hypervolume indicator \cite{ZitzlerT99} for 
measuring the quality of a population. This approach has gained increasing interest in recent years (see \eg\ \cite{BeumeNE07,IgelHR07,KnowlesCorne03,ZitzlerBT07}). The hypervolume indicator implicitly defines an 
optimization goal for the population of an evolutionary algorithm.
Unfortunately, this optimization goal is rarely understood from a theoretical point of view. Recently, it has been shown in \cite{Augetal2009} that the slope of the front determines which objective 
vectors maximize the value of the hypervolume when dealing with continuous 
Pareto fronts. The aim of this paper is to further increase the theoretical 
understanding of the hypervolume indicator and examine its approximation 
behavior.

As multi-ob\-ject\-ive optimization problems often involve a vast number of Pareto-optimal 
objective vectors, multi-ob\-ject\-ive evolutionary algorithms use a population 
of fixed size and try to evolve the population into a good 
approximation of the Pareto front. However, often it is not stated explicitly 
what a good approximation for a given problem is. One approach that allows a 
rigorous evaluation of the approximation quality is to measure the quality of a 
solution set with respect to its approximation ratio~\cite{PapadimitriouY00}.
We follow this approach and examine the approximation ratio
of a population with respect to all objective vectors of the 
Pareto front.

The advantage of the approximation ratio is that it gives a meaningful scalar value
which allows us to compare the quality of solutions between different
functions, different population sizes, and even different dimensions.
This is not the case for the hypervolume indicator.  A specific dominated volume
does not give a priori any information how well a front is approximated.
Also, the hypervolume measures the space relative to an arbitrary reference
point (cf.~\secref{hyp}).  This (often unwanted) freedom of choice not only
changes the distribution
of the points, but also makes the hypervolumes of different solutions measured
relative to a (typically dynamically changing) reference point very hard to compare.

Our aim is to examine whether a given solution set of $\mu$ search points maximizing 
the hypervolume (called the optimal hypervolume distribution) gives a good approximation measured with respect to the approximation ratio. 
We do this by investigating two classes of objective functions having two objectives each and analyze the optimal distribution for the hypervolume indicator and the one achieving the optimal approximation ratio.

In a first step, we assume that both sets of $\mu$ points have to include both optimal points regarding the given two single objective functions. 
We point out situations where maximizing the hypervolume provably 
leads to the best approximation ratio achievable by choosing $\mu$ 
Pareto-optimal solutions.
After these theoretical investigations, we carry out numerical investigations to see 
how the shape of the Pareto front influences the approximation behavior of the 
hypervolume indicator and point out where the approximation given by 
the hypervolume differs from the best one achievable by a solution set of 
$\mu$ points. These initial theoretical and experimental results investigating the correlation between the hypervolume indicator and multiplicative approximations have been published as a conference version in~\cite{HypGECCO2009}.

This paper extends its conference version in \secref{depref} to the case where the optimal hypervolume distribution is dependent on the chosen reference point. 
The reference point is a crucial parameter when applying hypervolume-based algorithms. It determines the area in the objective space where the algorithm focuses its search. As the hypervolume indicator itself, it is hard to understand the impact of the choice of the reference point. Different studies have been carried out on this topic and initial results on the optimal hypervolume distribution in the dependence of the reference point have been obtained in \cite{Augetal2009} and \cite{Brockhoff2010}.
We provide new insights into how the choice of the reference point may affect the approximation behavior of hypervolume-based algorithms. In our studies, we relate the optimal hypervolume distribution with respect to a given reference to the optimal approximation ratio obtainable when having the freedom to choose the $\mu$ points arbitrarily.

The rest of the paper is structured as follows. In Section~\ref{sec2}, we introduce the hypervolume indicator and our notation of approximations. Section~\ref{sec3} gives analytic results for the approximation achievable by the hypervolume indicator under the assumption that both extreme points have to be included in the two distributions and reports on our numerical investigations  Pareto fronts having different shapes. In \secref{depref}, we generalize our results and study the impact of the reference point on the optimal hypervolume distribution and relate this choice to the best possible overall approximation ratio when choosing $\mu$~points. Finally, we finish with some concluding remarks.


\section{The Hypervolume Indicator and Multiplicative Approximations}\label{sec2}

In this paper, we consider bi-objective maximization problems $P \colon \mathcal{S} \to \R^{2}$
for an arbitrary decision space $\mathcal{S}$.
We are interested in the so-called Pareto front of $P$,
which consists of all maximal elements of $P(\mathcal{S})$ with respect to the weak Pareto dominance relation.
We restrict ourselves to problems with a Pareto front that can be written as
$\{(x,f(x)) \mid x \in [x_{\min},x_{\max}]\}$ where
$f \colon [x_{\min},x_{\max}] \to \R$ is a continuous, differentiable, and strictly monotonically decreasing function.
This allows us to denote with $f$ not only the actual function $f \colon [x_{\min},x_{\max}] \to \R$,
but also the front $\{(x,f(x)) \mid x \in [x_{\min},x_{\max}]\}$ itself.
We assume further that $x_{\min}>0$ and $f(x_{\max})>0$ hold.

We intend to find a solution set $X^* = \{x^*_{1},x^*_{2},\dots,x^*_{\mu}\}$ of
$\mu$ Pareto-optimal search points $(x^*_{i},f(x^*_{i}))$ that constitutes a good approximation of the front $f$.

\subsection{Hypervolume indicator}
\label{sec:hyp}

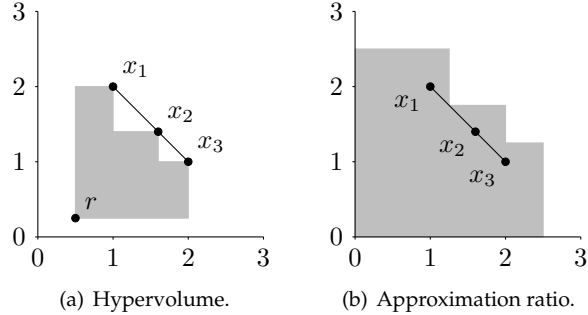
\begin{figure}[tb]
	\centering
	\subfigure[Hypervolume.]{
		\begin{tikzpicture}
			\tikzstyle{point}=[circle,inner sep=1,outer sep=0,draw,fill];
			
			\draw [draw=gray!50,fill=gray!50] (0.5,0.25) -- (0.5,2) -- (1,2) -- (1,1.4) -- (1.6,1.4) -- (1.6,1) -- (2,1) -- (2,0.25) -- cycle;
			
			\path (0.5,0.25) node (R)  [point] {} node [above right] {$r$};
			\path (1,2)      node (X1) [point] {} node [above right] {$x_{1}$};
			\path (1.6,1.4)  node (X2) [point] {} node [above right] {$x_{2}$};
			\path (2,1)      node (X3) [point] {} node [above right] {$x_{3}$};
			
			\draw (1,2) -- (2,1);
			
			\draw (0,0) -- (3,0);
			\foreach \x in {0,1,2,3}
				\draw (\x cm,0) -- (\x cm,-1pt) node [anchor=north] {$\x$};
			\draw (0,0) -- (0,3);
			\foreach \y in {0,1,2,3}
				\draw (0,\y cm) -- (-1pt,\y cm) node [anchor=east] {$\y$};
		\end{tikzpicture}
		\label{fig:sketch1}
	}
	\hspace*{.01\textwidth}
	\subfigure[Approximation ratio.]{
		\begin{tikzpicture}
			\tikzstyle{point}=[circle,inner sep=1,outer sep=0,draw,fill];
			
			\draw [draw=gray!50,fill=gray!50] (0,0) -- (0,2.5) -- (1.25,2.5) -- (1.25,1.75) -- (2,1.75) -- (2,1.25) -- (2.5,1.25) -- (2.5,0) -- cycle;
			
			\path (1,2)      node (X1) [point] {} node [below left] {$x_{1}$};
			\path (1.6,1.4)  node (X2) [point] {} node [below left] {$x_{2}$};
			\path (2,1)      node (X3) [point] {} node [below left] {$x_{3}$};
			
			\draw (1,2) -- (2,1);
			
			\draw (0,0) -- (3,0);
			\foreach \x in {0,1,2,3}
				\draw (\x cm,0) -- (\x cm,-1pt) node [anchor=north] {$\x$};
			\draw (0,0) -- (0,3);
			\foreach \y in {0,1,2,3}
				\draw (0,\y cm) -- (-1pt,\y cm) node [anchor=east] {$\y$};
		\end{tikzpicture}
		\label{fig:sketch2}
	}
	\caption{Point distribution $X=\{1,1.6,2\}$ for the linear front $f \colon [1,2] \to [1,2]$ with $f(x)=3-x$, which achieves a hypervolume of $\HYP(X)=1.865$ with respect to the reference point $r=(0.5,0.25)$ and an approximation ratio of $\APP{X}=1.25$. The shaded areas show the dominated portion of the objective space and the approximated portion of the objective space, respectively.}
	\label{fig:sketch}
\end{figure}

The hypervolume ($\HYP$) measures the volume of the dominated portion of the objective space.
It was first introduced for performance assessment
in multi-ob\-ject\-ive optimization by Zitzler and Thiele~\cite{ZitzlerT99}.
Later on it was used to guide the search in various hypervolume-based evolutionary
optimizers~\cite{BeumeNE07,EmmerichBN05,IgelHR07,KnowlesCF2003,ZitzlerBT07,ZitzlerK04}.

Geometrically speaking, the hypervolume indicator measures the volume
of the dominated space of all solutions contained in a solution set $X \subseteq \R^{d}$.
This space is truncated at a fixed footpoint called the \emph{reference point $r=(r_1,r_2,\dots,r_d)$}.
The \emph{hypervolume $\HYP_r(Y)$ of a solution set $Y$} in dependence of a given reference point $r=(r_1, r_2, \dots, r_d)$ is then defined as
\begin{equation*}
	\HYP_r(Y) := \VolBigg{\bigcup_{(y_1,\dots,y_d) \in Y} [r_1,y_1] \times \dots \times [r_d,y_d]}
\end{equation*}
with $\Vol{\cdot}$ being the usual Lebesgue measure (see \figref{sketch1} for an illustration).

The hypervolume indicator is a popular
second-level sorting criterion in many recent multi-objective
evolutionary algorithms for several reasons.
Besides having a very intuitive interpretation, 
it is also the only common indicator that is strictly Pareto-compliant~\cite{ZitzlerTLFF03}.
Strictly Pareto-compliant means that given two solution sets $A$ and~$B$ the indicator
values $A$ higher than $B$ if the solution set $A$ dominates the solution set $B$.
It has further been shown by \citet{approxjournal}
that the worst-case approximation factor of all possible Pareto fronts
obtained by any hypervolume-optimal set of fixed
size $\mu$ is asymptotically equal
to the best worst-case approximation factor achievable by any set of size $\mu$.

In the last years,
the hypervolume has become very popular and several algorithms have been developed to calculate it.
The first one was the Hypervolume by Slicing Objectives (HSO) algorithm, which
was suggested independently by Zitzler~\cite{Zitzler01} and Knowles~\cite{Knowles02}.
For $d\leq3$ it can be solved in (asymptotically optimal)
time $\Oh{n\log n}$~\cite{FonsecaPLI05}.
The currently best asymptotic runtime for $d\in\{4,5,6\}$
is $\Oh{n^{(d-1)/2}\,\log n}$~\cite{YildizSuri12}.
The best known bound for
large dimensions $d\geq 7$
is $\Oh{n^{(d+2)/3}}$~\cite{Bringmann12}.

On the other hand,
Bringmann and Friedrich~\cite{BF08CGTA} proved that 
all hypervolume algorithms
must have a superpolynomial runtime in the number of objectives (unless $\P=\NP$).
Assuming the widely accepted exponential time hypothesis,
the runtime must even be at least $n^{\Omega(d)}$~\cite{2013GECCO}.
As this dashes the hope for fast and exact hypervolume algorithms,
there are several estimation algorithms~\cite{BaderDZ08TR,BF08CGTA,BF09b}
for approximating the hypervolume based on Monte Carlo sampling.

\subsection{Approximations}

In the following, we define our notion of approximation in a formal way.
Let $X = \{x_{1},\dots,x_{\mu}\}$ be a solution set and $f$ a function that describes the Pareto front.
We call a Pareto front convex if the function defining the Pareto front is a convex function.
Otherwise, we call the Pareto front concave. Note that this differs from the notation used in~\cite{HypGECCO2009}.


The \emph{approximation ratio $\APP{X}$ of a solution set $X$} with respect to $f$ is defined according to~\cite{PapadimitriouY00} as follows.

\begin{definition}\label{def:ratio}
Let $f \colon [x_{\min},x_{\max}] \to \R$ and $X = \{x_{1},\linebreak[0]x_{2},\linebreak[0]\dots,\linebreak[0]x_{\mu}\}$.
The solution set $X$ is a \emph{$\ratio$-ap\-proxi\-ma\-tion of $f$} iff for each $x \in [x_{\min},x_{\max}]$ there is an $x_{i} \in X$ with
\begin{equation*}
	x \leq \ratio \cdot x_i ~~\text{and}~~ f(x) \leq \ratio \cdot f(x_i)
\end{equation*}
where $\ratio \in \R$, $\ratio \geq 1$.
The \emph{approximation ratio of $X$} with respect to $f$ is defined as
\begin{equation*}
	\APP{X} := \min\{\ratio \in \R \mid \text{$X$ is a $\ratio$-approximation of $f$}\}.
\end{equation*}
\end{definition}

\figref{sketch2} shows the area of the objective space that a certain solution set $X$ $\ratio$-approximates for $\ratio = 1.25$.
Note that this area covers the entire Pareto front $f$.
Since the objective vector $(1.25,1.75)$ is not $\ratio$-approximated for all $\ratio < 1.25$, the approximation ratio of $X$ is $1.25$.

Our definition of approximation is similar to the definition of multiplicative $\ee$-dominance given in \cite{Lauetal2002}. In this paper, an algorithmic framework for discrete multi-objective optimization is proposed which converges to a $(1+\ee)$-approximation of the Pareto front.


\section{Results independent of the reference point}\label{sec3}


The goal of this paper is to relate the above definition of approximation to the optimization goal implicitly defined by the hypervolume indicator. 
Using the hypervolume, the choice of the reference point decides which parts of the front are covered. 
In this section we avoid the additional influence of the reference point 
by considering only solutions where both extreme points have to be included.
The influence of the reference point is studied in \secref{depref}.

All the functions that we consider in this paper have positive and bounded domains and codomains. Furthermore, the functions that are under consideration don't have infinite or zero derivative at the extremes.
Hence, choosing the reference point $r = (r_1, r_2)$ for appropriate $r_1,r_2\leq 0$ ensures that the points
$x_{\min}$ and $x_{\max}$ are contained in an optimal hypervolume distribution.
A detailed calculation on how to choose the reference point such that $x_{\min}$ and $x_{\max}$ are contained in an optimal hypervolume distribution is given in~\cite{Augetal2009}.
Assuming that $x_{\min}$ and $x_{\max}$ have to be included in the optimal hypervolume distribution, the value of the volume is in this section independent of the choice of the reference point. Therefore, we write  $\HYP(X)$ instead of $\HYP_r(X)$.


Consider a Pareto front $f$.  There is an infinite number of possible solution sets of fixed size $\mu$.
To make this more formal, let $\cX(\mu,f)$ be the set of all subsets of 
\[
	\left\{(x,f(x)) \mid x \in [x_{\min},x_{\max}]\right\}
\]
of cardinality $\mu$
which contain
$(x_{\min},f(x_{\min}))$ and $(x_{\max},f(x_{\max}))$.  We want to compare two specific solution sets from $\cX$
called \emph{optimal hypervolume distribution} and \emph{optimal approximation distribution}
defined as follows.

\begin{definition}\label{def:OPT}
    The optimal hypervolume distribution
    \begin{equation*}
    	\OPTHYP{\mu}{f} := \argmax_{X\in\cX(\mu,f)}{\HYP(X)}
    \end{equation*}
    consists of $\mu$ points that maximize the hypervolume with respect to $f$.
    The optimal approximation distribution
    \begin{equation*}
    	\OPTAPP{\mu}{f} := \argmin_{X\in\cX(\mu,f)}{\APP{X}}
    \end{equation*}
    consists of $\mu$ points that minimize the approximation ratio with respect to $f$.
    For brevity, we will
    also use $\OPT{\mu}{f}$ in Figures~\ref{fig:x3}--\ref{fig:p} as a short form to refer to both sets \OPTHYP{\mu}{f} and \OPTAPP{\mu}{f}.
\end{definition}

Note that ``optimal hypervolume distributions'' are also called
``optimal $\mu$\nobreakdash-distributions''~\cite{Augetal2009,Brockhoff2010}
or ``maximum hypervolume set''~\cite{approxjournal} in the literature.

We want to investigate the approximation ratio obtained by a solution set
maximizing the hypervolume indicator in comparison to an optimal one.
For this, we first examine conditions for an optimal approximation distribution $\OPTAPP{\mu}{f}$.
Later on, we consider two classes of functions $f$ on which 
the optimal hypervolume distribution $\OPTHYP{\mu}{f}$ is equivalent to the
optimal approximation distribution $\OPTAPP{\mu}{f}$ and therefore
provably leads to the best achievable approximation ratio.


\subsection{Optimal approximations}

We now consider the optimal approximation ratio that can be achieved placing 
$\mu$ points on the Pareto front given by the function $f$. The following lemma 
states a condition which allows to check whether a given set consisting of $\mu$ 
points achieves an optimal approximation ratio for a given function~$f$.

\begin{lemma}\label{lem:optapprox}
    Let $f \colon [x_{\min},x_{\max}] \to \R$ be a Pareto front
    and $X = \{x_1,\dots,x_\mu\}$ be an arbitrary solution set
    with $x_{1}=x_{\min}$, $x_{\mu}=x_{\max}$, and $x_i \leq x_{i+1}$ for all $1\leq i<\mu$.
    If there is a constant $\ratio > 1$ and a set $=Z=\{z_{1},\dots, z_{\mu-1}\}$
    with $x_{i} \leq z_{i} \leq x_{i+1}$
    and $\ratio = \frac{z_i}{x_i} = \frac{f(z_i)}{f(x_{i+1})}$ for all $1 \leq i < \mu$,
    then $X=\OPTAPP{\mu}{f}$ is the optimal approximation distribution
    with approximation ratio $\ratio$.
\end{lemma}

\begin{proof}
We assume that a better approximation ratio than $\ratio$ can be achieved by choosing a different set of solutions $X' = \{x_1',\dots,x_\mu'\}$ with $x_1'=x_{\min}$, $x_{\mu}'=x_{\max}$, and $x_i' \leq x_{i+1}'$, $1 \leq i < \mu$, and show a contradiction.

The points $z_i$, $1 \leq i \leq \mu-1$, are the points that are worst approximated by the set $X$. Each $z_i$ is approximated by a factor of $\delta$. Hence, in order to obtain a better approximation than the one achieved by the set $X$, the points $z_i$ have to be approximated within a ratio of less than $\ratio$.
We now assume that there is a point $z_i$ for which a better approximation is achieved by the set $X'$.
Getting a better approximation of $z_i$ than $\ratio$ means that there is at least one point $x_j' \in X'$ with $x_i< x_j' < x_{i+1}$ as otherwise $z_i$ is approximated within a ratio of at least  $\frac{z_i}{x_i} = \frac{f(z_i)}{f(x_{i+1})} = \ratio$.

We assume \wlo\ that $j \leq i+1$ and show that there is at least one point $z$ with $z \leq z_i$ that is not approximated by a factor of $\ratio$ or that $x_1' > x_{\min}$ holds.
To approximate all points $z$ with $x_{i-1} \leq z \leq x_j'$ by a factor of $\ratio$, the inequality $x_{i-1} < x_{j-1}' < x_i$ has to hold as otherwise $z_{i-1}$ is approximated within a ratio of more than $\ratio$ by $X'$.
We iterate the arguments.  In order to approximate all points in $x_{i-s} \leq y \leq x_{i-s+1}$, $x_{i-s} < x_{j-s}' < x_{i-s+1}$ has to hold as otherwise $z_{i-s}$ is not approximated within a ratio of $\ratio$ by $X'$.
Considering $s=j-1$ either one of the points $z$, $x_{i-j+1} \leq y \leq x_{i-j+2}$ is not approximated within a ratio of $\ratio$ by $X'$ or $x_{\min}=x_1 \leq x_{i-j+1} < x_1'$ holds, which contradicts the assumption that $X'$ includes $x_{\min}$ and constitutes an approximation better than $\ratio$.

The case $j > i+1$ can be handled symmetrically, by showing that either $x_n' < x_{\max}$ or there is a point $z \geq z_{i+1}$ that is not approximated within a ratio of $\ratio$ by $X'$. This completes the proof.
\end{proof}

We will use this lemma in the rest of the paper to check whether an approximation obtained by the hypervolume indicator is optimal as well as use these ideas to identify sets of points that achieve an optimal approximation ratio.


\subsection{Analytic results for linear fronts}

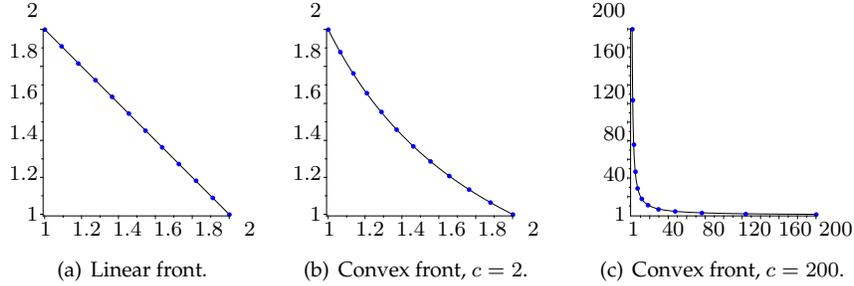
\begin{figure}[tb]
    \centering
    \subfigure[Linear front.]{
\begin{tikzpicture}
\node (full) at (0,0) [inner sep=0pt,above right]
            {\includegraphics[width=\distfigwidth,viewport=0.42in 0.36in 7.5in 7.5in
            ,clip]{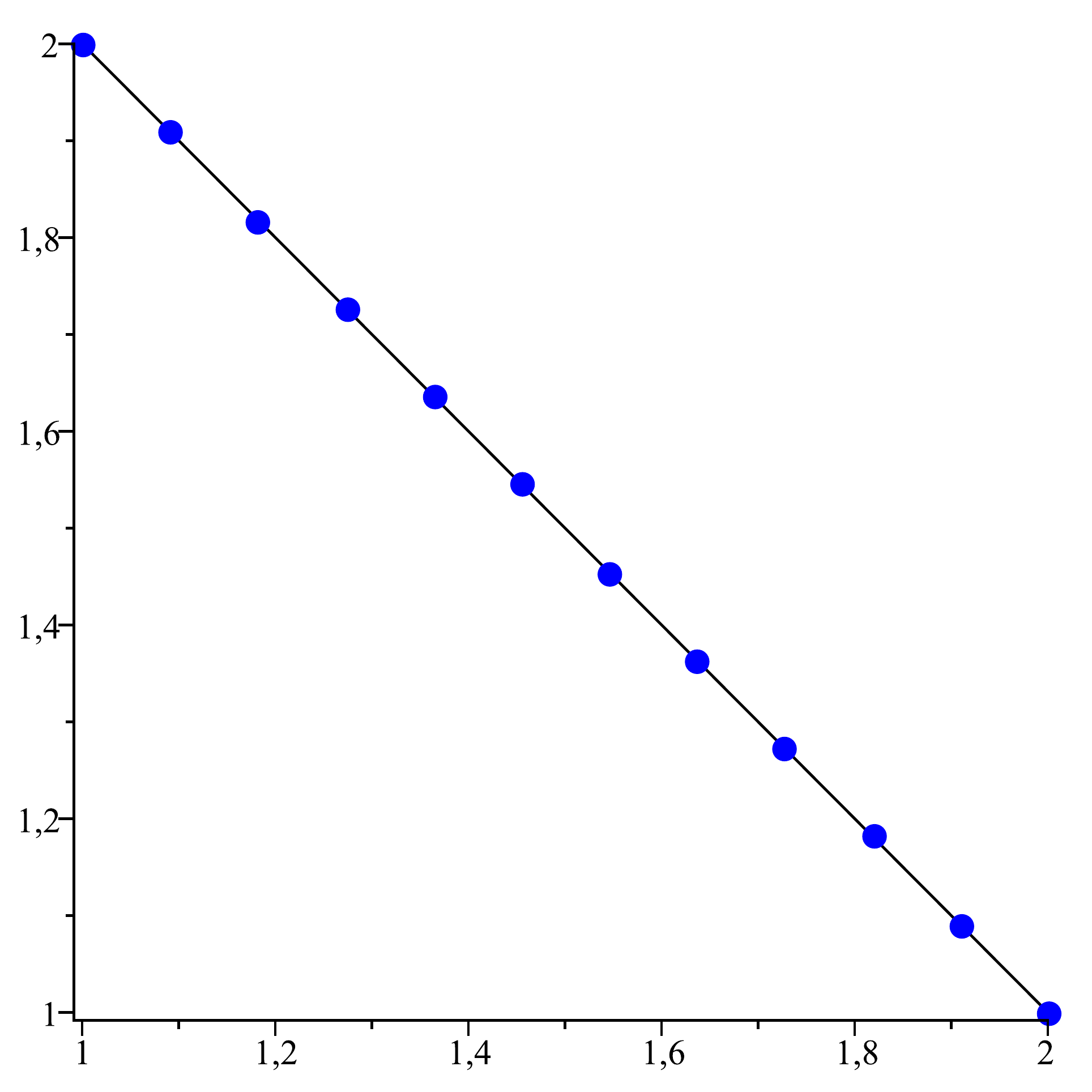}}; 
\foreach \x in {1,1.2,...,2}{
\pgfmathparse{0.08+2.71*(\x-1)}
\node at (\pgfmathresult,.1) [below] {\footnotesize\pgfmathprintnumber{\x}};
}
\foreach \y in {1,1.2,...,2}{
\pgfmathparse{0.08+2.71*(\y-1)}
\node at (.1,\pgfmathresult) [left] {\footnotesize\pgfmathprintnumber{\y}};
}
\end{tikzpicture}
}
    \subfigure[Convex front, $c=2$.]{
\begin{tikzpicture}
\node (full) at (0,0) [inner sep=0pt,above right]
            {\includegraphics[width=\distfigwidth,viewport=0.42in 0.36in 7.5in 7.5in
            ,clip]{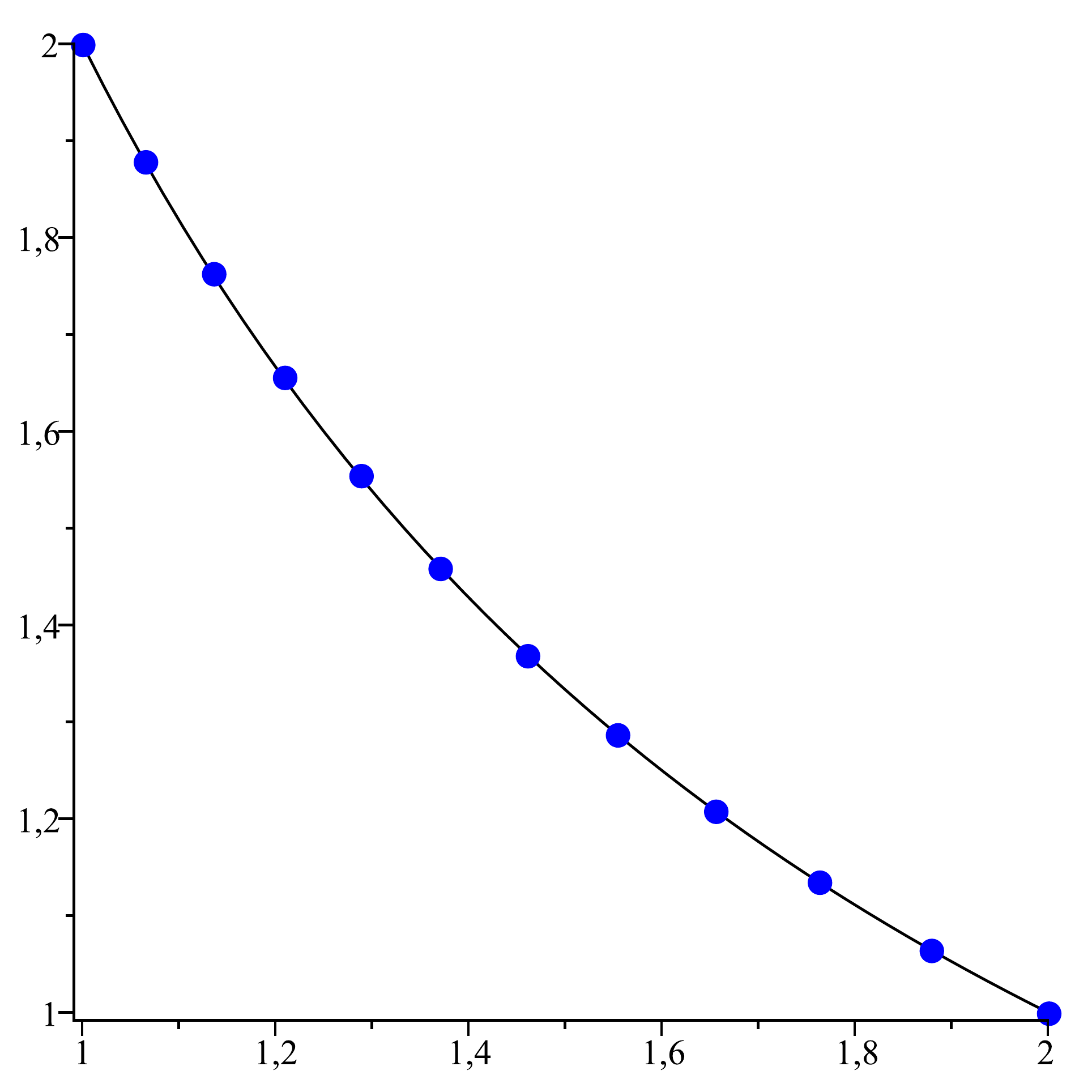}}; 
\foreach \x in {1,1.2,...,2}{
\pgfmathparse{0.08+2.71*(\x-1)}
\node at (\pgfmathresult,.1) [below] {\footnotesize\pgfmathprintnumber{\x}};
}
\foreach \y in {1,1.2,...,2}{
\pgfmathparse{0.08+2.71*(\y-1)}
\node at (.1,\pgfmathresult) [left] {\footnotesize\pgfmathprintnumber{\y}};
}
\end{tikzpicture}
        \label{fig:con2}
    }
    \hspace*{.01\textwidth}
    \subfigure[Convex front, $c=200$.]{
\begin{tikzpicture}
\node (full) at (0,0) [inner sep=0pt,above right]
            {\includegraphics[width=\distfigwidth,viewport=0.484in 0.36in 7.55in 7.5in
            ,clip]{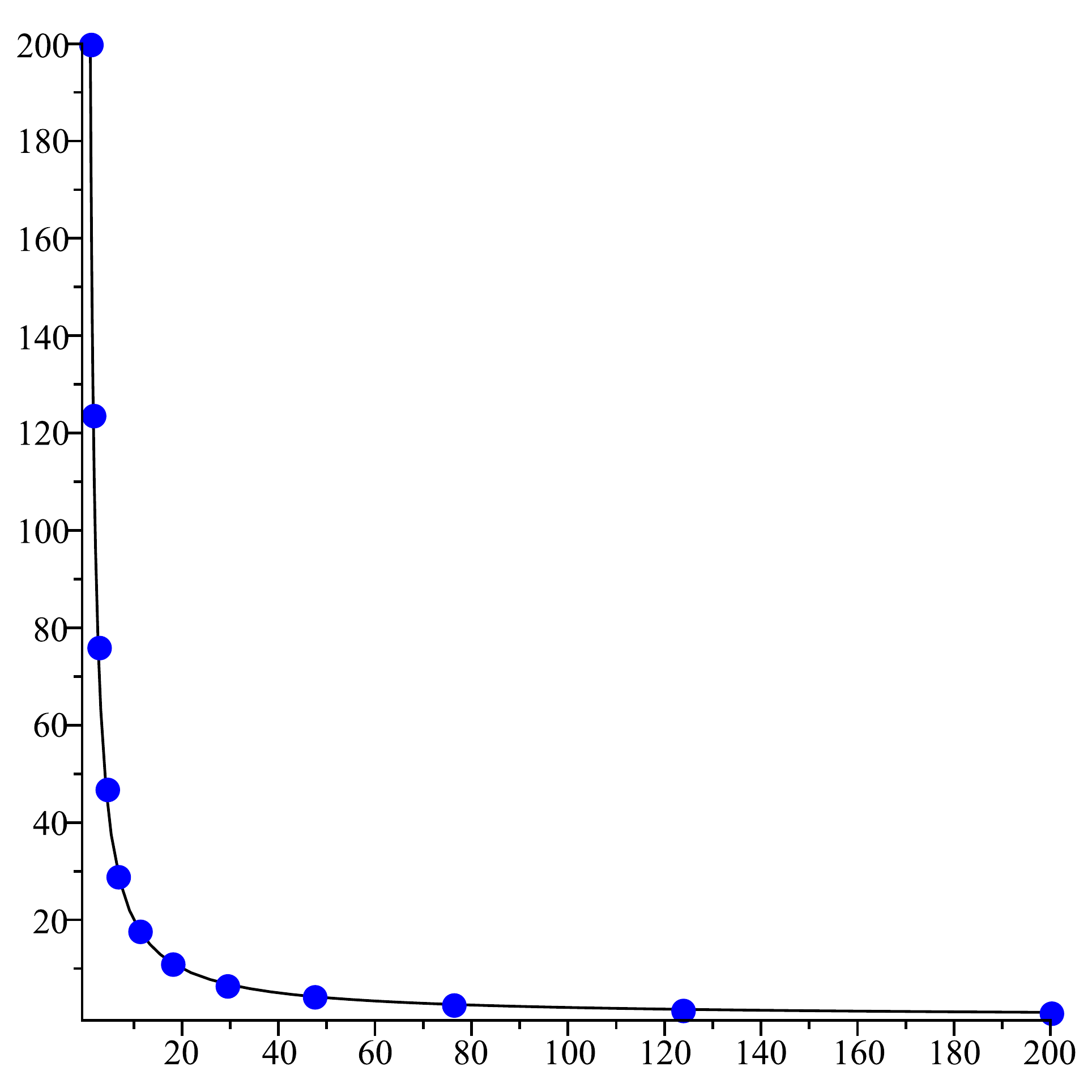}}; 
\foreach \x in {1,40,80,...,200}{
\pgfmathparse{0.08+0.01352*(\x-1)}
\node at (\pgfmathresult,.1) [below] {\footnotesize\pgfmathprintnumber{\x}};
}
\foreach \y in {1,40,80,...,200}{
\pgfmathparse{0.08+0.01367*(\y-1)}
\node at (.1,\pgfmathresult) [left] {\footnotesize\pgfmathprintnumber{\y}};
}
\end{tikzpicture}
        \label{fig:con200}
    }
    \caption{Optimal point distribution $\OPTHYP{12}{f}=\OPTAPP{12}{f}$ for 
             (a) the linear front
            $f \colon [1,2] \to [1,2]$ with  $f(x)=3-x$
            and
            (b,c) the convex fronts
            $f \colon [1,c] \to [1,c]$ with $f(x) = c/x$.
The respective optimal hypervolume distributions and optimal approximation distributions are equivalent in all three cases.}
    \label{fig:con}
    \label{fig:lin}
\end{figure}

The distribution of points maximizing the hypervolume for linear fronts has already been investigated in \cite{Augetal2009,EmmerichDB07}. Therefore, we start by considering the hypervolume indicator with respect to the approximation it 
achieves when the Pareto front is given by a linear function
\begin{equation*}
	f \colon [1,(1-d)/c] \to [1,c+d] ~~\text{with}~~ f(x) = c \cdot x + d
\end{equation*}
where $c < 0$ and $d > 1 - c$ are arbitrary constants.

Auger et al.~\cite{Augetal2009} and Emmerich et al.~\cite{EmmerichDB07} have shown that the maximum hypervolume of $\mu$ 
points on a linear front is reached when the points are distributed in an 
equally spaced manner. We assume that the reference point is 
chosen such that the extreme points of the Pareto front are included in the 
optimal distribution of the $\mu$ points on the Pareto front, that is, $x_1 = 
x_{\min} = 1$ and $x_{\mu} = x_{\max} = (1-d)/c$ hold. The maximal hypervolume 
is achieved by choosing
\begin{align}\label{eqn:hyp}
	x_i &= x_{\min} + \frac{i-1}{\mu-1} \cdot \left(x_{\max} - x_{\min}\right)\nonumber\\
	&= 1 + \frac{i-1}{\mu-1} \cdot \left(\frac{1-d}{c} - 1\right)
\end{align}
due to Theorem~6 in \cite{Augetal2009}.

The following theorem shows that the optimal approximation distribution coincides with the optimal hypervolume distribution.

\begin{theorem}\label{thm:lin}
	Let $f \colon [1,(1-d)/c] \to [1,c+d]$ 
	be a linear function $f(x) = c \cdot x + d$ where $c < 0$ and $d > 1 - c$ are arbitrary constants.
	Then
	\begin{equation*}
		\OPTHYP{\mu}{f} = \OPTAPP{\mu}{f}.
	\end{equation*}
\end{theorem}
\begin{proof}
We determine the approximation ratio that the optimal hypervolume distribution $\OPTHYP{\mu}{f} = \{x_{1}, \dots, x_{\mu}\}$ using $\mu$ points achieves.
Let $\tilde{x}$, $x_i < \tilde{x} < x_{i+1}$, be a Pareto-optimal $x$-coordinate. The approximation given by $x_i$ and $x_{i+1}$ is 
\begin{equation*}
	\min \left\{\frac{\tilde{x}}{x_i}, \frac{f(\tilde{x})}{f(x_{i+1})} \right\}
\end{equation*}
as $f$ is monotonically decreasing.

Furthermore, as $f$ is monotonically decreasing,
the worst-case approximation is attained for a point $\tilde{x}$, $x_i < \tilde{x} < x_{i+1}$, if
\begin{equation}\label{eqn:approx}
	\frac{\tilde{x}}{x_i} = \frac{f(\tilde{x})}{f(x_{i+1})}
\end{equation}
holds.

Using $$x_i = 1 + \frac{i-1}{\mu-1} \cdot \left(\frac{1-d}{c} - 1\right)$$ and 
resolving the Equation~\ref{eqn:approx} with respect to $\tilde{x}$, we get
\begin{equation*}
    \tilde{x} = \frac{d \, ((d + c - 1) \, i - c \mu - d + 1)}
                     {c \, ((\mu - 2) \, d - c + 1)}.
\end{equation*}
For the approximation ratio we get
\begin{equation*}
    \frac{\tilde{x}}{x_i} = \frac{f(\tilde{x})}{f(x_{i+1})} = \frac{d (\mu-1)}
                                                                   {d (\mu-2) - c + 1}.
\end{equation*}
Hence, the worst-case approximation is independent of the choice of $i$ and the same for all intervals $[x_i, x_{i+1}]$ of the Pareto front.
Lemma~\ref{lem:optapprox} implies that the hypervolume achieves the best possible approximation ratio on the class of linear fronts.
\end{proof}

\figref{lin}~(a) shows the optimal distribution for $f(x)=3-x$ and $\mu = 12$.


\newcommand{\renumberx}[1]{
\hspace*{-3mm}
\begin{tikzpicture}
	\node (full) at (0,0) [inner sep=0pt,above right]
	            {\includegraphics[width=\subfigwidth,viewport=0.63in 0.362in 7.52in 7.5in
	            ,clip]{pdfs/#1-eps-converted-to.pdf}}; 
	\foreach \x in {2,10,20,...,50}{
		\pgfmathparse{0.038+0.0443*(\x-2)}
		\node at (\pgfmathresult,.1) [below] {\footnotesize\pgfmathprintnumber{\x}};
	}
\end{tikzpicture}
\hspace*{-3mm}
}

\subsection{Analytic results for a class of convex fronts}
\label{sec:confixed}

We now consider the distribution of $\mu$ points on a convex front maximizing 
the hypervolume. In contrast to the class of linear functions where an optimal 
approximation can be achieved by distributing the $\mu$ points in an equally 
spaced manner along the front, the class of functions considered in this section 
requires that the points are distributed exponentially to obtain an optimal approximation.

As already argued we want to make sure that optimal hypervolume distribution includes $x_{\min}$ and $x_{\max}$. For the class of convex fronts that we consider, this can be achieved by choosing the reference point $r=(0,0)$.

The hypervolume of a set of points $X=\{x_1,\dots,x_{\mu}\}$, where \wlo\ $x_1 \leq x_2 \leq \dots \leq x_{\mu}$, is then given by
\begin{align*}
	\HYP(X) = {} & x_1 \cdot f(x_1) + x_2 \cdot f(x_2) - x_1 \cdot f(x_2) + \dots + x_{\mu} \cdot f(x_{\mu}) - x_{\mu-1} \cdot f(x_{\mu})\\
	= {} & x_1 \cdot f(x_1) + x_2 \cdot f(x_2) + \dots + x_{\mu} \cdot f(x_{\mu}) - (x_1 \cdot f(x_2) + \dots + x_{\mu-1} \cdot f(x_{\mu})).
\end{align*}

We consider a Pareto front given by the function
\begin{equation*}
	f \colon [1,c] \to [1,c] ~~\text{and}~~ f(x) = c/x
\end{equation*}
where $c > 1$ is an arbitrary constant. Then we get
\begin{equation*}
	\HYP(X) = c \cdot \mu - c \cdot \left(\frac{x_1}{x_2} + \frac{x_2}{x_3} + \dots + \frac{x_{\mu-2}}{x_{\mu-1}} + \frac{x_{\mu-1}}{x_{\mu}}\right).
\end{equation*}

\newcommand{\renumbermu}[1]{
\hspace*{-4mm}
\begin{tikzpicture}
	\node (full) at (0,0) [inner sep=0pt,above right]
	            {\includegraphics[width=\subfigwidth,viewport=0.63in 0.362in 7.52in 7.5in
	            ,clip]{pdfs/#1-eps-converted-to.pdf}}; 
	\foreach \x in {3,5,10,...,20}{
		\pgfmathparse{0.038+0.1374*(\x-3)}
		\node at (\pgfmathresult,.1) [below] {\footnotesize\pgfmathprintnumber{\x}};
	}
\end{tikzpicture}
\hspace*{-4mm}
}
\newcommand{\renumbermuu}[1]{
\hspace*{-4mm}
\begin{tikzpicture}
	\node (full) at (0,0) [inner sep=0pt,above right]
	            {\includegraphics[width=\subfigwidth,viewport=0.513in 0.362in 7.52in 7.5in
	            ,clip]{pdfs/#1-eps-converted-to.pdf}}; 
	\foreach \x in {3,5,10,...,20}{
		\pgfmathparse{0.038+0.1374*(\x-3)}
		\node at (\pgfmathresult,.1) [below] {\footnotesize\pgfmathprintnumber{\x}};
	}
\end{tikzpicture}
\hspace*{-4mm}
}

\newcommand{\renumberp}[1]{
\hspace*{-3mm}
\begin{tikzpicture}
	\node (full) at (0,0) [inner sep=0pt,above right]
	            {\includegraphics[width=\subfigwidth,viewport=0.63in 0.362in 7.52in 7.5in
	            ,clip]{pdfs/#1-eps-converted-to.pdf}}; 
	\foreach \x/\y in {1/{1/4},2/{1/2},3/1,4/2,5/4}{
		\pgfmathparse{0.038+0.584*(\x-1)}
		\node at (\pgfmathresult,.1) [below] {\footnotesize\y};
	}
\end{tikzpicture}
\hspace*{-3mm}
}

\newcommand{\renumberpp}[1]{
\hspace*{-3mm}
\begin{tikzpicture}
	\node (full) at (0,0) [inner sep=0pt,above right]
	            {\includegraphics[width=\subfigwidth,viewport=0.513in 0.362in 7.52in 7.5in
	            ,clip]{pdfs/#1-eps-converted-to.pdf}}; 
	\foreach \x/\y in {1/{1/4},2/{1/2},3/1,4/2,5/4}{
		\pgfmathparse{0.038+0.584*(\x-1)}
		\node at (\pgfmathresult,.1) [below] {\footnotesize\y};
	}
\end{tikzpicture}
\hspace*{-3mm}
}

\noindent
Hence, to maximize the hypervolume we have to find $\mu$ points minimizing
\begin{equation*}
	h(x_1, \dots, x_{\mu}) := \left(\frac{x_1}{x_2} + \dots + \frac{x_{\mu-1}}{x_{\mu}}\right).
\end{equation*}
Setting $x_1 = 1$ and $x_{\mu} = c$ minimizes $h$, since $x_{1}$ and $x_{\mu}$ occur just in the first and last term of $h$, respectively.
Furthermore, we have $1=x_1 < x_2 < \ldots < x_\mu=c$ as the equality of two points implies that one of them can be exchanged for another unchosen point on the Pareto front and thereby increases the hypervolume. 

We work under these assumptions and aim to find a set of points $X$ that minimizes the function $h$.
To do this, we consider the gradient vector given by the partial derivatives
\begin{align*}
	&h'(x_1,\dots,x_{\mu})= \left(\frac{1}{x_2},-\frac{x_1}{x_2^2} +
	  \frac{1}{x_3},\dots,-\frac{x_{\mu-2}}{x_{\mu-1}^2} +
	  \frac{1}{x_{\mu}},-\frac{x_{\mu-1}}{x_{\mu}^2}\right).
\end{align*}
This implies that $h$ can be minimized by setting
\begin{equation*}
	\begin{array}{lclcl}
		x_3 & = & x_2^2/x_1 & = & x_2^2, \\
		x_4 & = & x_3^2/x_2 & = & x_2^3, \\
		 & \vdots &  & \vdots &  \\
		x_\mu & = & x_{\mu-1}^2/x_{\mu-2} & = & x_2^{\mu-1}.
	\end{array}
\end{equation*}
From the last equation we get
\begin{equation}
	\begin{array}{lclcl}
		x_2 & = & x_\mu^{1/(\mu-1)} & = & c^{1/(\mu-1)}, \\
		x_3 & = & x_2^2 & = & c^{2/(\mu-1)}, \\
		 & \vdots &  & \vdots &  \\
		x_{\mu-1} & = & x_2^{\mu-2} & = & c^{(\mu-2)/(\mu-1)}.
	\end{array}
	\label{eqn:hypcon}
\end{equation}

The following theorem shows that the optimal approximation distribution coincides with the optimal hypervolume distribution.

\begin{theorem}\label{thm:con}
	Let $f \colon [1,c] \to [1,c]$ be a convex front with $f(x) = c/x$ where $c > 1$ is an arbitrary constant.
	Then
	\begin{equation*}
		\OPTHYP{\mu}{f} = \OPTAPP{\mu}{f}.
	\end{equation*}
\end{theorem}
\begin{proof}
We determine the approximation ratio that the optimal hypervolume distribution $\OPTHYP{\mu}{f} = \{x_{1}, \dots, x_{\mu}\}$ using $\mu$ points achieves.
As $f$ is monotonically decreasing, the worst-case approximation is attained for a point $x$, $x_i < x < x_{i+1}$, if 
\begin{equation*}
	\frac{x}{x_i} = \frac{f(x)}{f(x_{i+1})}
\end{equation*}
holds.
Substituting the coordinates and function values, we get
\begin{equation*}
	\frac{x}{x_i} = \frac{x}{c^{(i-1)/ (\mu-1)}} ~~\text{and}~~ \frac{f(x)}{f(x_{i+1})} = \frac{c/x}{c/c^{i/ (\mu-1)}} = \frac{c^{i/(\mu-1)}}{x}.
\end{equation*}
Therefore,
\begin{equation*}
	x^2 = c^{i/(\mu-1)} \cdot c^{(i-1)/ (\mu-1)}= c^{(2i-1)/(\mu-1)},
\end{equation*}
which implies
\begin{equation*}
	x = c^{(2i-1)/(2\mu-2)}.
\end{equation*}
Hence, the set of search points maximizing the hypervolume achieves an approximation ratio of
\begin{equation}
	\frac{c^{(2i-1)/(2\mu-2)}}{c^{(i-1)/(\mu-1)}} = c^{1/(2\mu-2)}.
	\label{eqn:conapp}
\end{equation}

We have seen that the requirements of Lemma~\ref{lem:optapprox} are fulfilled.
Hence, an application of Lemma~\ref{lem:optapprox} shows that the hypervolume indicator achieves an 
optimal approximation ratio when the Pareto front is given by $f \colon [1,c] 
\to [1,c]$ with $f(x) = c/x$ where $c \in \R_{> 1}$ is any constant.
\end{proof}

\figref{con} shows the optimal distribution for $\mu = 12$ and $c=2$ as well as $c=200$.


\subsection{Numerical evaluation for fronts of different shapes}
\label{sec4}

\begin{figure}[t]
    \centering
    \subfigure[Set $\OPTHYP{12}{\fsym_2}$.]{
	\begin{tikzpicture}
	\node (full) at (0,0) [inner sep=0pt,above right]
	            {\includegraphics[width=\distfigwidth,viewport=0.42in 0.362in 7.5in 7.5in
	            ,clip]{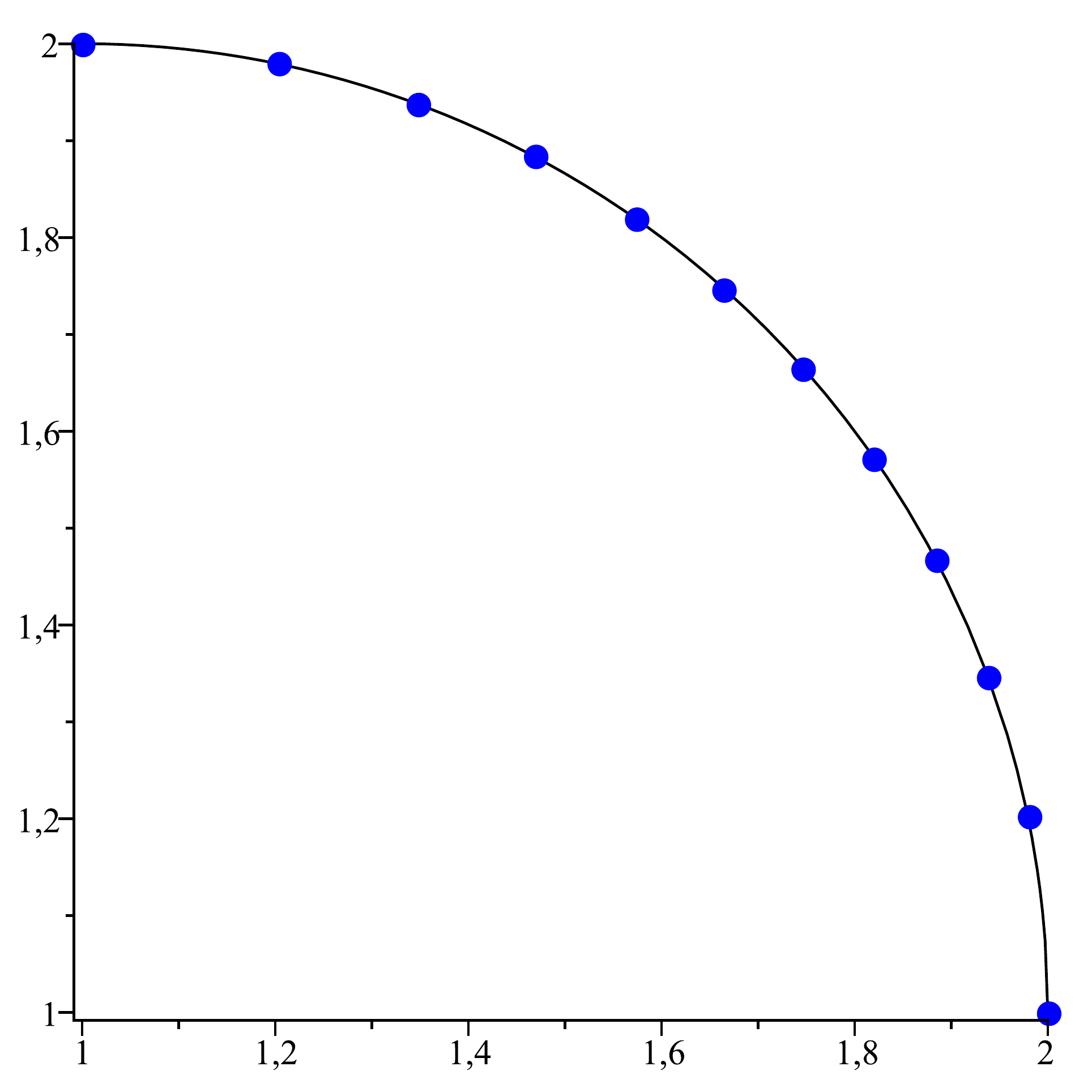}}; 
	\foreach \x in {1,1.2,...,2}{
	\pgfmathparse{0.08+2.71*(\x-1)}
	\node at (\pgfmathresult,.1) [below] {\footnotesize\pgfmathprintnumber{\x}};
	}
	\foreach \y in {1,1.2,...,2}{
	\pgfmathparse{0.08+2.71*(\y-1)}
	\node at (.1,\pgfmathresult) [left] {\footnotesize\pgfmathprintnumber{\y}};
	}
	\end{tikzpicture}
        \label{fig:symhyp}
    }
    \hspace*{.01\textwidth}
    \subfigure[Set $\OPTAPP{12}{\fsym_2}$.]{
    	\begin{tikzpicture}
	\node (full) at (0,0) [inner sep=0pt,above right]
	            {\includegraphics[width=\distfigwidth,viewport=0.42in 0.362in 7.5in 7.5in
	            ,clip]{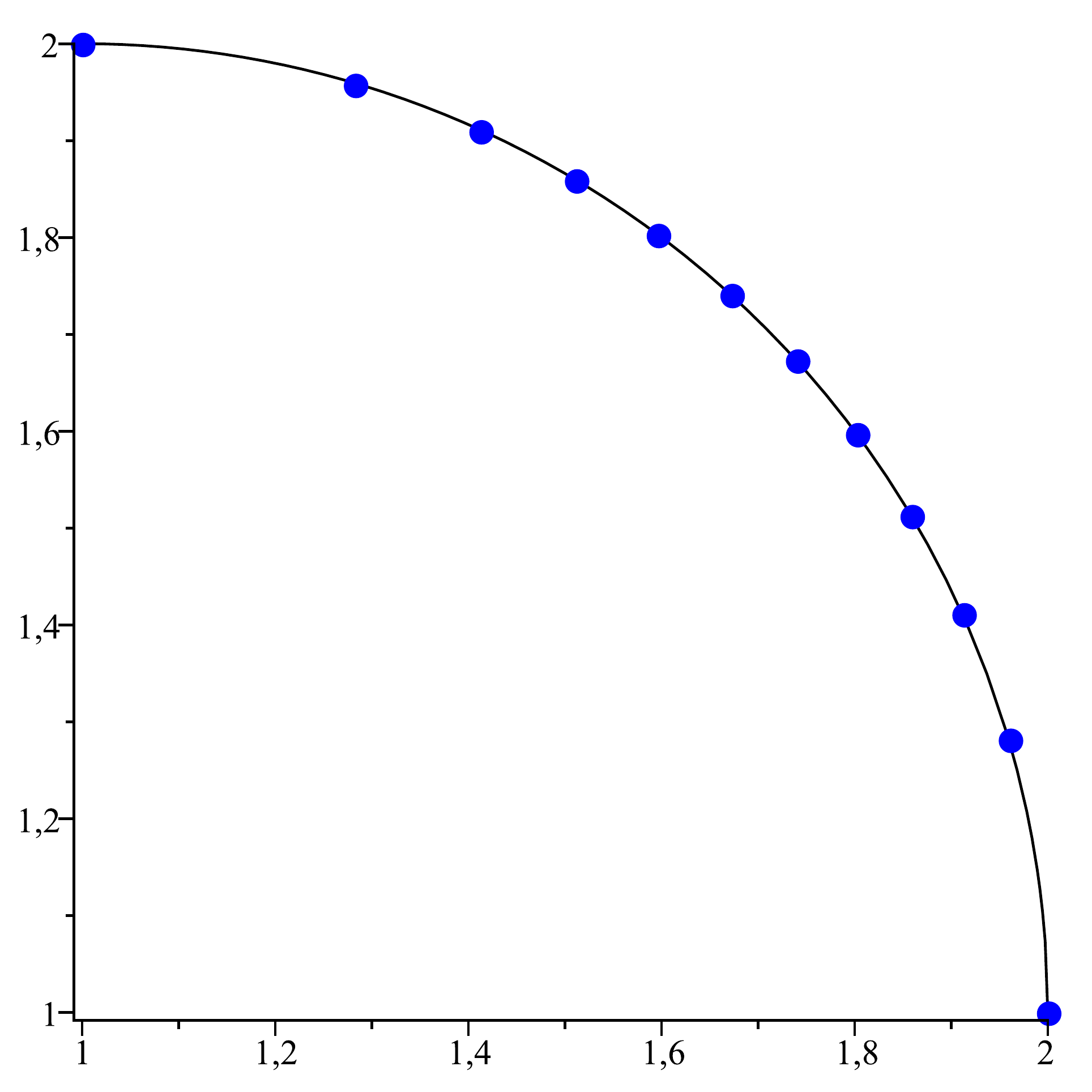}}; 
	\foreach \x in {1,1.2,...,2}{
	\pgfmathparse{0.08+2.71*(\x-1)}
	\node at (\pgfmathresult,.1) [below] {\footnotesize\pgfmathprintnumber{\x}};
	}
	\foreach \y in {1,1.2,...,2}{
	\pgfmathparse{0.08+2.71*(\y-1)}
	\node at (.1,\pgfmathresult) [left] {\footnotesize\pgfmathprintnumber{\y}};
	}
	\end{tikzpicture}
        \label{fig:symapp}
    }
    \caption{Optimal point distributions for symmetric front $\fsym_2$.
             Note that the optimal hypervolume distribution and the
             optimal approximation distribution differ in this case.
             The set of points maximizing the hypervolume yields an
             approximation ratio of $\APP{\OPTHYP{12}{\fsym_2}}\approx1.025$,
             which is $0.457\%$ larger than the optimal approximation ratio
             $\APP{\OPTAPP{12}{\fsym_2}}\approx1.021$.
             }
    \label{fig:sym}
\end{figure}

\begin{figure}[h]
    \centering
    \subfigure[Set $\OPTHYP{12}{\fasy_2}$.]{
	\begin{tikzpicture}
	\node (full) at (0,0) [inner sep=0pt,above right]
	            {\includegraphics[width=\distfigwidth,viewport=0.42in 0.362in 7.5in 7.5in
	            ,clip]{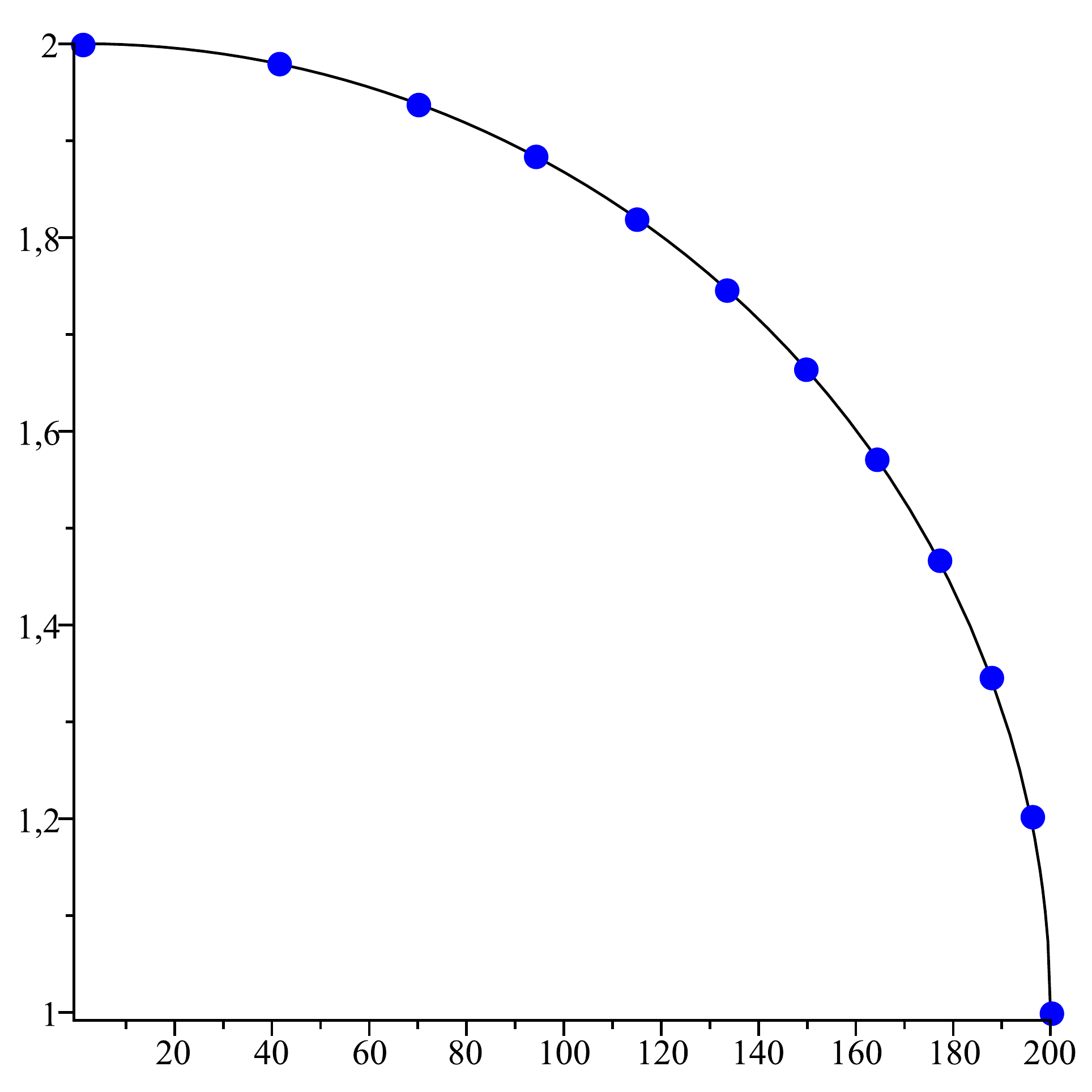}}; 
	\foreach \x in {1,40,80,...,200}{
	\pgfmathparse{0.08+0.01358*(\x-1)}
	\node at (\pgfmathresult,.1) [below] {\footnotesize\pgfmathprintnumber{\x}};
	}
	\foreach \y in {1,1.2,...,2}{
	\pgfmathparse{0.08+2.71*(\y-1)}
	\node at (.1,\pgfmathresult) [left] {\footnotesize\pgfmathprintnumber{\y}};
	}
	\end{tikzpicture}
        \label{fig:asyhyp}
    }
    \hspace*{.01\textwidth}
    \subfigure[Set $\OPTAPP{12}{\fasy_2}$.]{
	\begin{tikzpicture}
	\node (full) at (0,0) [inner sep=0pt,above right]
	            {\includegraphics[width=\distfigwidth,viewport=0.42in 0.362in 7.5in 7.5in
	            ,clip]{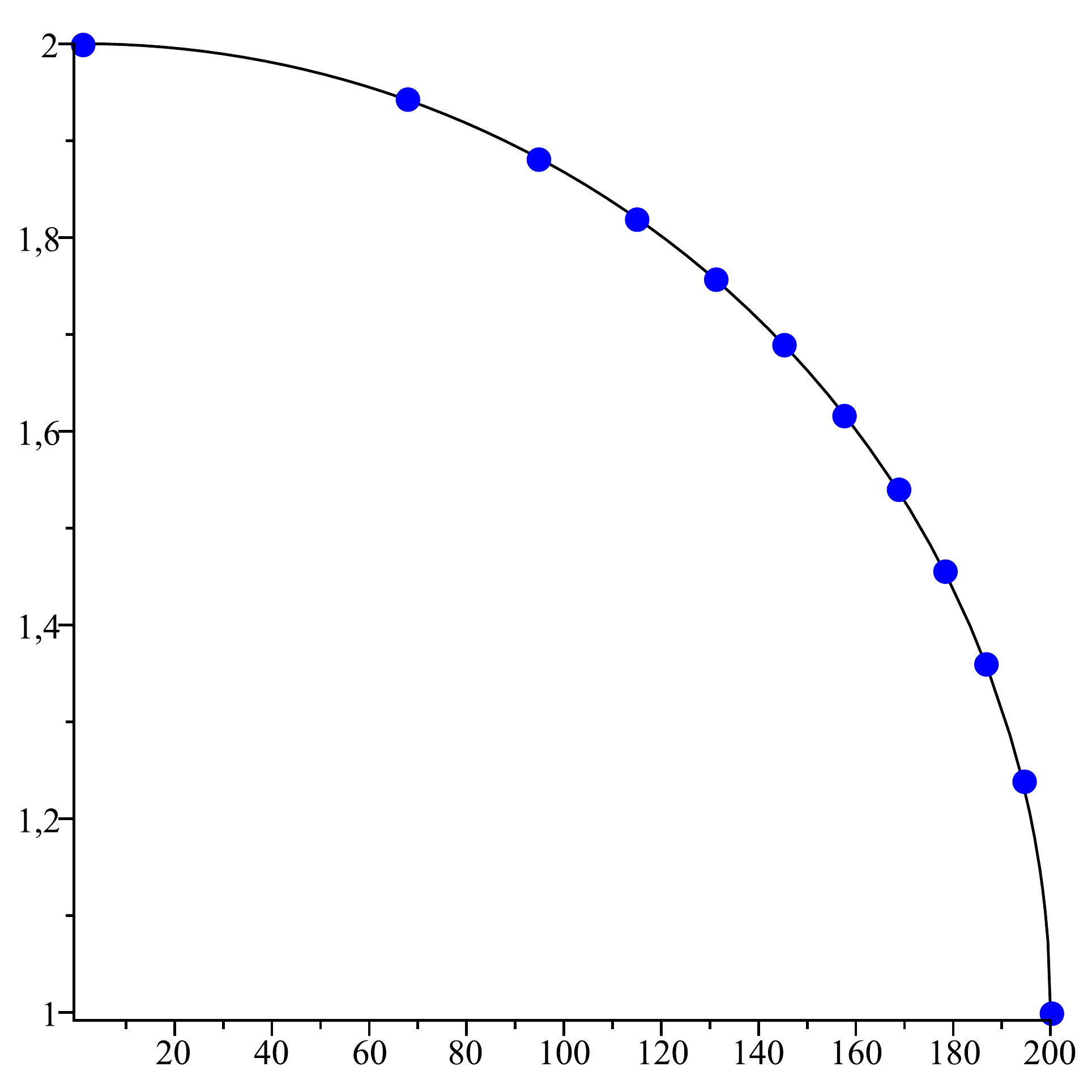}}; 
	\foreach \x in {1,40,80,...,200}{
	\pgfmathparse{0.08+0.01358*(\x-1)}
	\node at (\pgfmathresult,.1) [below] {\footnotesize\pgfmathprintnumber{\x}};
	}
	\foreach \y in {1,1.2,...,2}{
	\pgfmathparse{0.08+2.71*(\y-1)}
	\node at (.1,\pgfmathresult) [left] {\footnotesize\pgfmathprintnumber{\y}};
	}
	\end{tikzpicture}
        \label{fig:asyapp}
    }
    \caption{Optimal point distributions for asymmetric front $\fasy_2$. 
             Note that the optimal hypervolume distribution and the
             optimal approximation distribution differ in this case.
             The set of points maximizing the hypervolume yields an
             approximation ratio of $\APP{\OPTHYP{12}{\fasy_2}}\approx1.038$,
             which is $0.839\%$ larger than the optimal approximation ratio
             $\APP{\OPTAPP{12}{\fasy_2}}\approx1.030$.
             }
    \label{fig:asy}
\end{figure}

The analysis of the distribution of an optimal set of search points tends to be hard or is impossible for more complex functions.
Hence, resorting to numerical analysis methods constitutes a possible escape from this dilemma.
This section is dedicated to the numerical analysis of a larger class of functions.


Our goal is to study the optimal hypervolume distribution for different shapes of Pareto fronts and investigate how the shape of such a front influences the approximation behavior of the hypervolume indicator.
We examine a family of fronts of the shape $x^p$ where $p>0$ is a parameter that determines the degree of the polynomial describing the Pareto front. Furthernmore, we allow scaling in both dimensions.

The Pareto fronts that we consider can be defined by a function of the form $f_p \colon [x_1,x_\mu] \to [y_\mu,y_1]$ with
\begin{equation*}
	f_p(x) := y_\mu - (y_\mu-y_1) \cdot \left( 1 - \left( \frac{x-x_1}{x_\mu-x_1} \right) ^p \right)^{1/p}.
\end{equation*}
We use the notation $y_i = f(x_i)$ for the function value $f(x_i)$ of a point $x_i$.
As we assume the reference point to be sufficiently negative, the leftmost point $(x_1,y_1)$
and the rightmost point $(x_\mu,y_\mu)$ are always contained in the optimal hypervolume distribution as well as
in the optimal approximation. 
We will mainly concentrate on two parameter sets of $f_p$, that is,
\begin{itemize}
\setlength{\itemsep}{0pt}
\setlength{\parskip}{0pt}
	\item the \emph{symmetric} front $\fsym_p \colon [1,2] \to [1,2]$ and
	\item the \emph{asymmetric} front $\fasy_p \colon [1,201] \to [1,2]$.
\end{itemize}

Note, that choosing $p=1$ corresponds to the well-known test function DTLZ1~\cite{DTLZ02}.
For $p=2$ the shape of the front corresponds to functions DTLZ2, DTLZ3, and DTLZ4.


Our goal is to study the optimal hypervolume distribution for our parametrized family of Pareto fronts and relate it to an optimal multiplicative approximation. 
Therefore, we calculate for different functions $f_p$ and $\mu \geq 3$
\begin{itemize}
\setlength{\itemsep}{0pt}
\setlength{\parskip}{0pt}
	\item the set of $\mu$ points $\OPTHYP{\mu}{f_p}$ which maximizes the dominated hypervolume, and
	\item the set of $\mu$ points $\OPTAPP{\mu}{f_p}$ which minimizes the multiplicative approximation ratio.
\end{itemize}

As in Section~\ref{sec3}, we assume that both extreme have to be included in both distributions.
For the optimal hypervolume distribution, it suffices to find the $x_2,x_3,\dots,x_{\mu-1}$ that maximize the dominated hypervolume, that is, the solutions of
\begin{align*}
	\argmax_{x_2,\dots,x_{\mu-1}}\Bigg((x_2-x_1) \cdot (f(x_2)-f(x_\mu)) 
	 + \sum_{i=3}^{\mu-1} (x_i-x_{i-1}) \cdot (f(x_i)-f(x_\mu))\Bigg)
\end{align*}
We solve the arising nonlinear continuous optimization problem numerically by means of sequential quadratic programming~\cite{SQP}.

In the optimal multiplicative approximation, we have to solve the following system of nonlinear equations
\begin{align*}
	\frac{z_1}{x_1} = \frac{z_2}{x_2} &= \dots = \frac{z_{\mu-1}}{x_{\mu-1}}=\\
	= \frac{f(z_1)}{f(x_2)} = \frac{f(z_2)}{f(x_3)} &= \dots = \frac{f(z_{\mu-1})}{f(x_\mu)}
\end{align*}
with auxiliary variables $z_1,\dots,z_{\mu-1}$ due to \lemref{optapprox}.
The numerical solution of this system of equations can be determined easily
by any standard computer algebra system. We used the Optimization package
of Maple~15.



In the following, we present the results that have been obtained by our numerical investigations.
We first examine the case of $f_2$.
\figrefs{sym}{asy} show different point distributions for $f_2$. It can be observed that the hypervolume distribution differs from the optimal distribution.
\figrefs{symhyp}{symapp} show the distributions for the symmetric front
\begin{equation*}
	f_2(x)=1+\sqrt{1-(x-1)^2}
\end{equation*}
with
$(x_1,y_1)=(1,2)$ and $(x_\mu,y_\mu)=(2,1)$.
\figrefs{asyhyp}{asyapp} show the asymmetric front
\begin{equation*}
	f_2(x)=1+\sqrt{1-(x/200-1/200)^2}
\end{equation*}
with $(x_1,y_1)=(1,2)$ and $(x_\mu,y_\mu)=(201,1)$.

It can be observed that the relative positions of the hypervolume points stay the same in \figrefs{symhyp}{asyhyp}
while the relative positions achieving an optimal approximation change with scaling
(cf.\ \figrefs{symapp}{asyapp}).
Hence, the relative position of the points maximizing the hypervolume
is robust with respect to scaling.
But as the optimal point distribution for a multiplicative approximation
is dependent on the scaling, the hypervolume cannot achieve the best possible
approximation quality.

\begin{figure*}[tb]
    \centering
    \subfigure[\APP{\OPT{3}{f_{1/3}}}.]{
	\renumberx{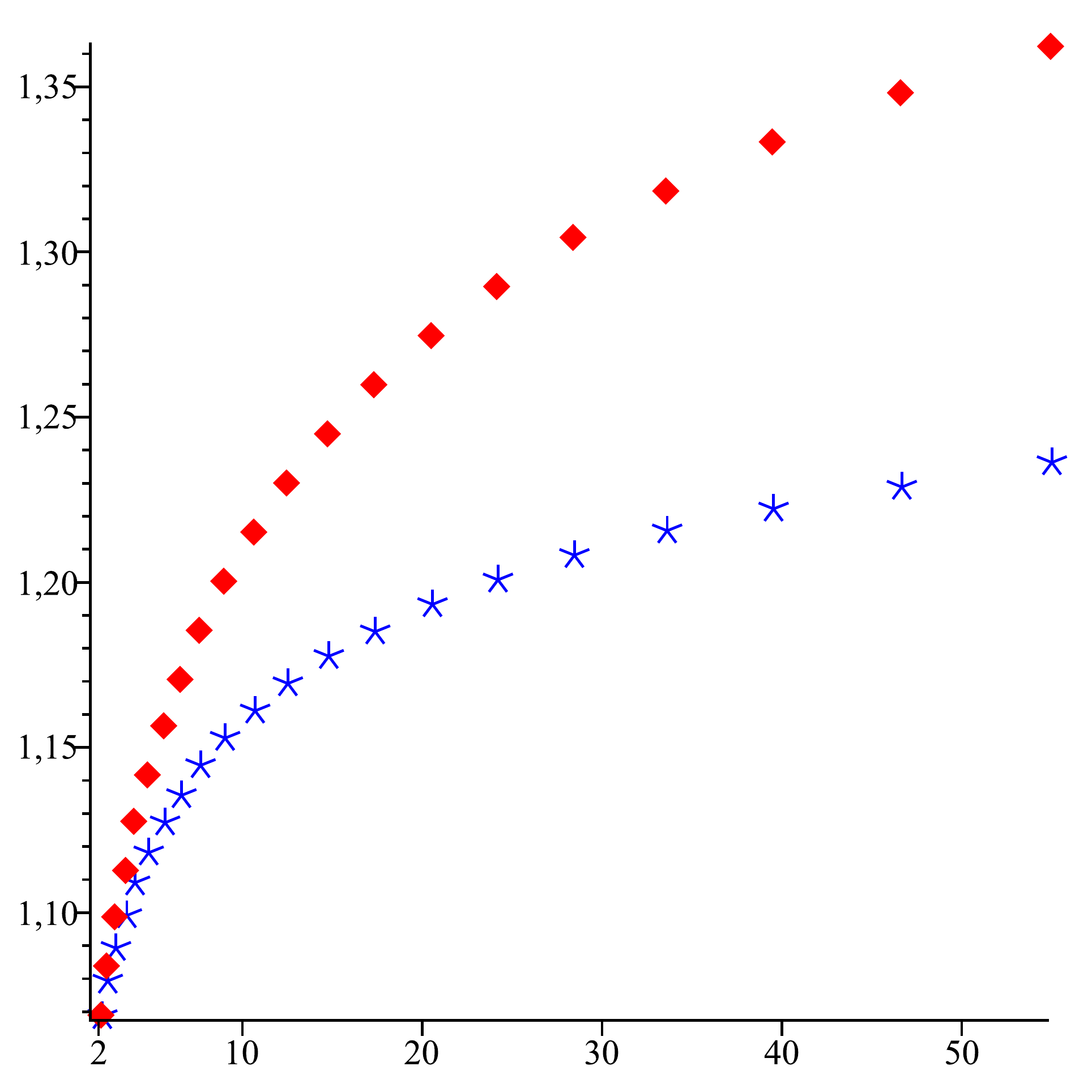}
    }
    \subfigure[\APP{\OPT{3}{f_{1/2}}}.]{
    	\renumberx{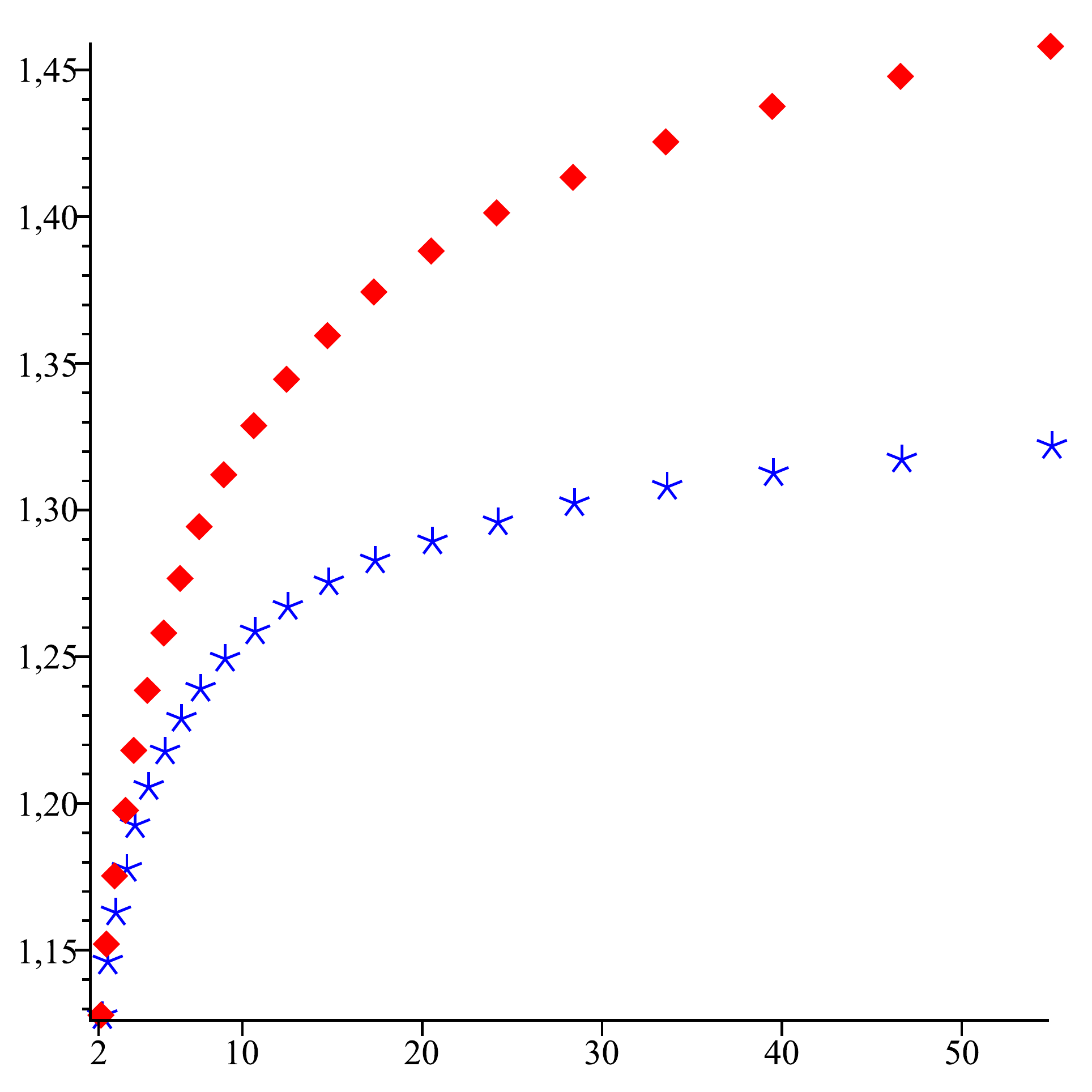}    }
    \subfigure[\APP{\OPT{3}{f_{1}}}.]{
    	\renumberx{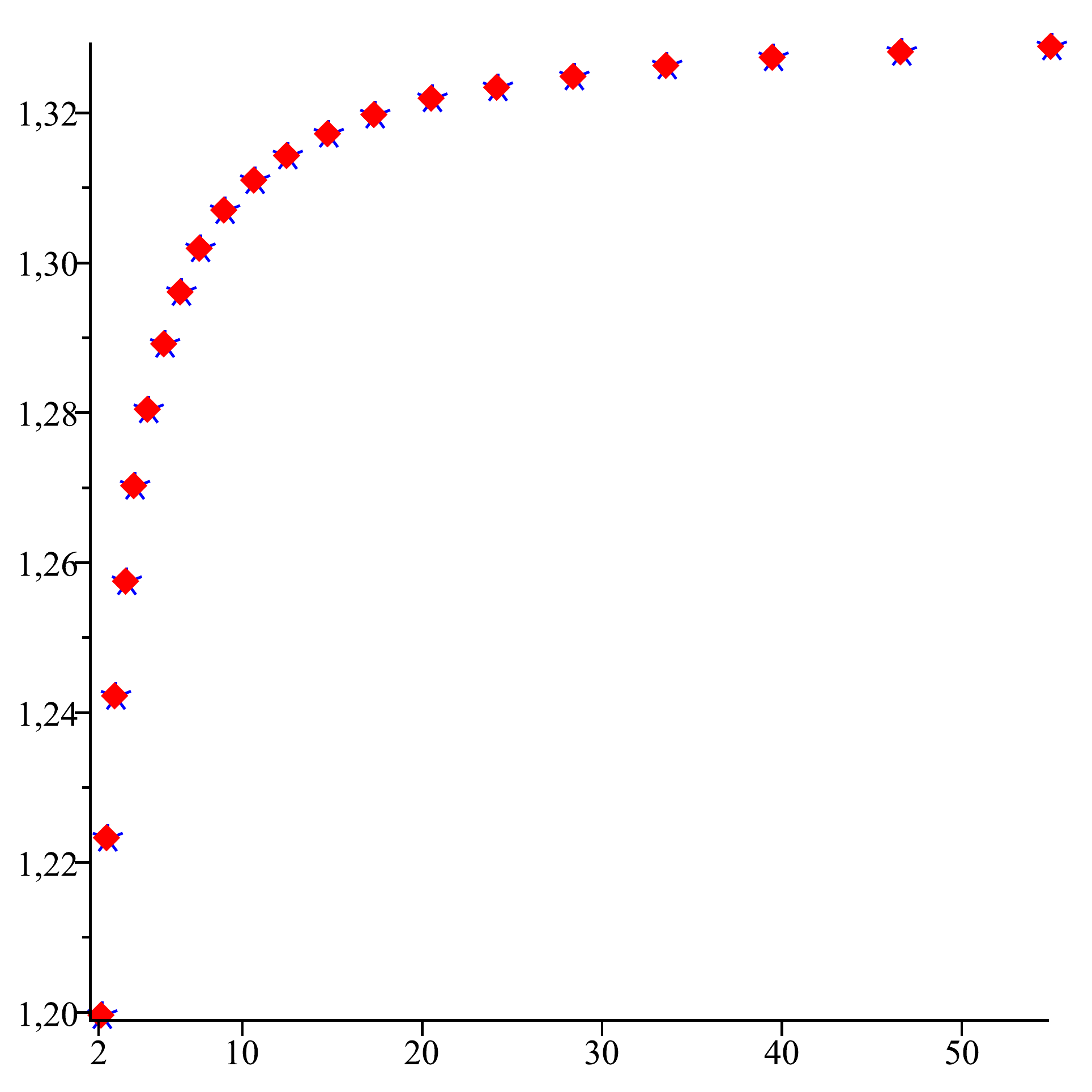}
    }
    \subfigure[\APP{\OPT{3}{f_{2}}}.]{
    	\renumberx{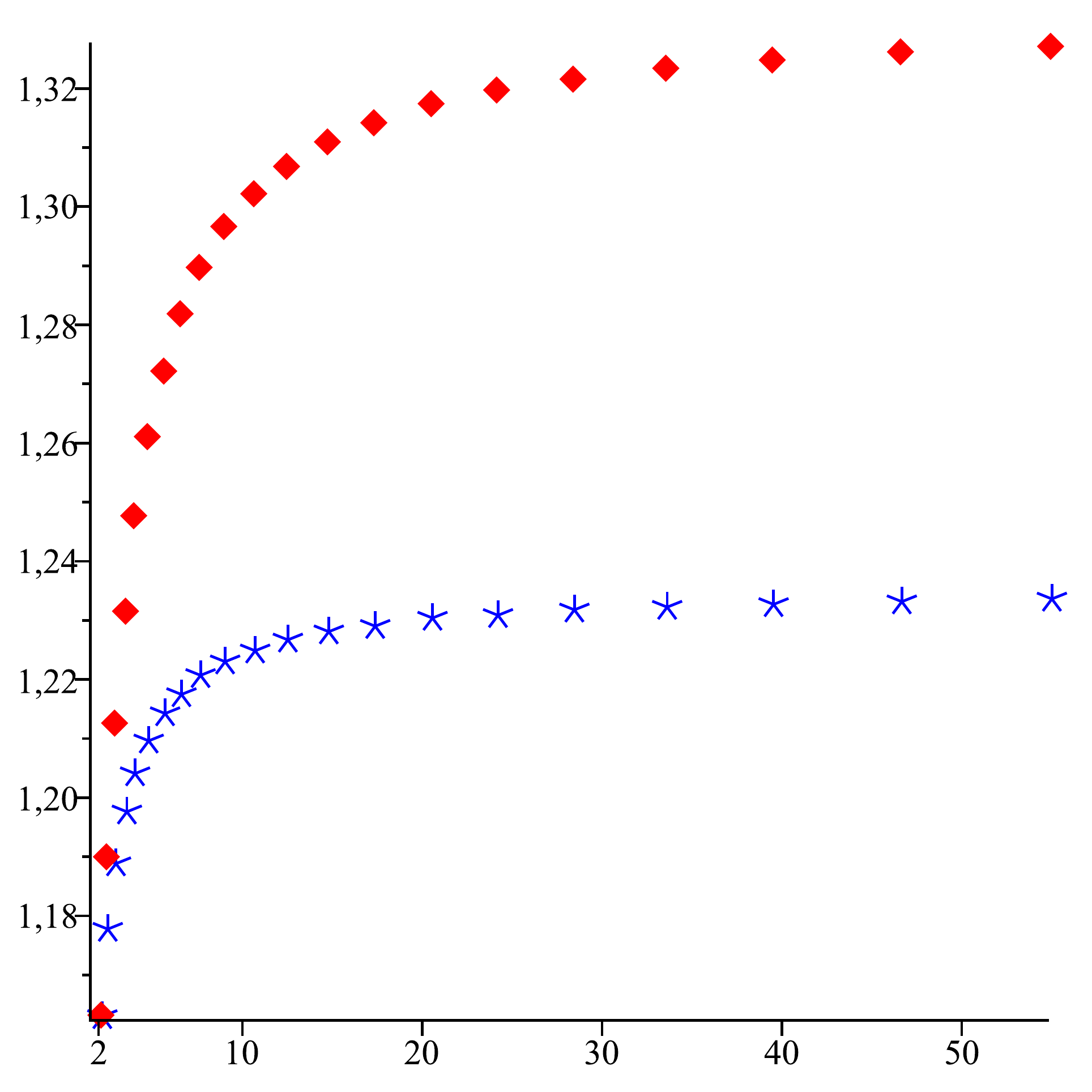}
    }
    \subfigure[\APP{\OPT{3}{f_{3}}}.]{
    	\renumberx{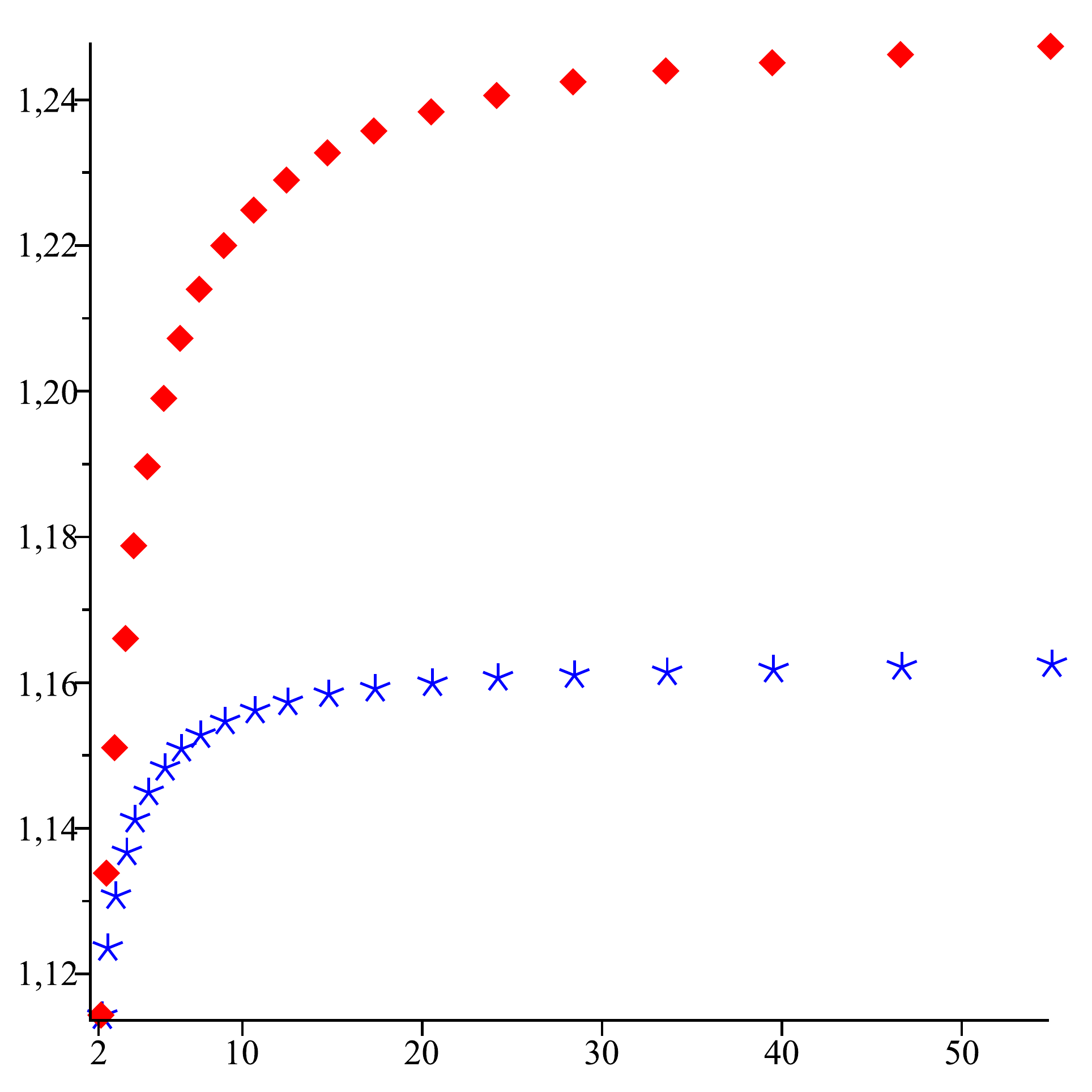}
    }
    \caption{Approximation ratio 
             of the optimal hypervolume distribution
             (\protect\includegraphics[width=2mm,clip]{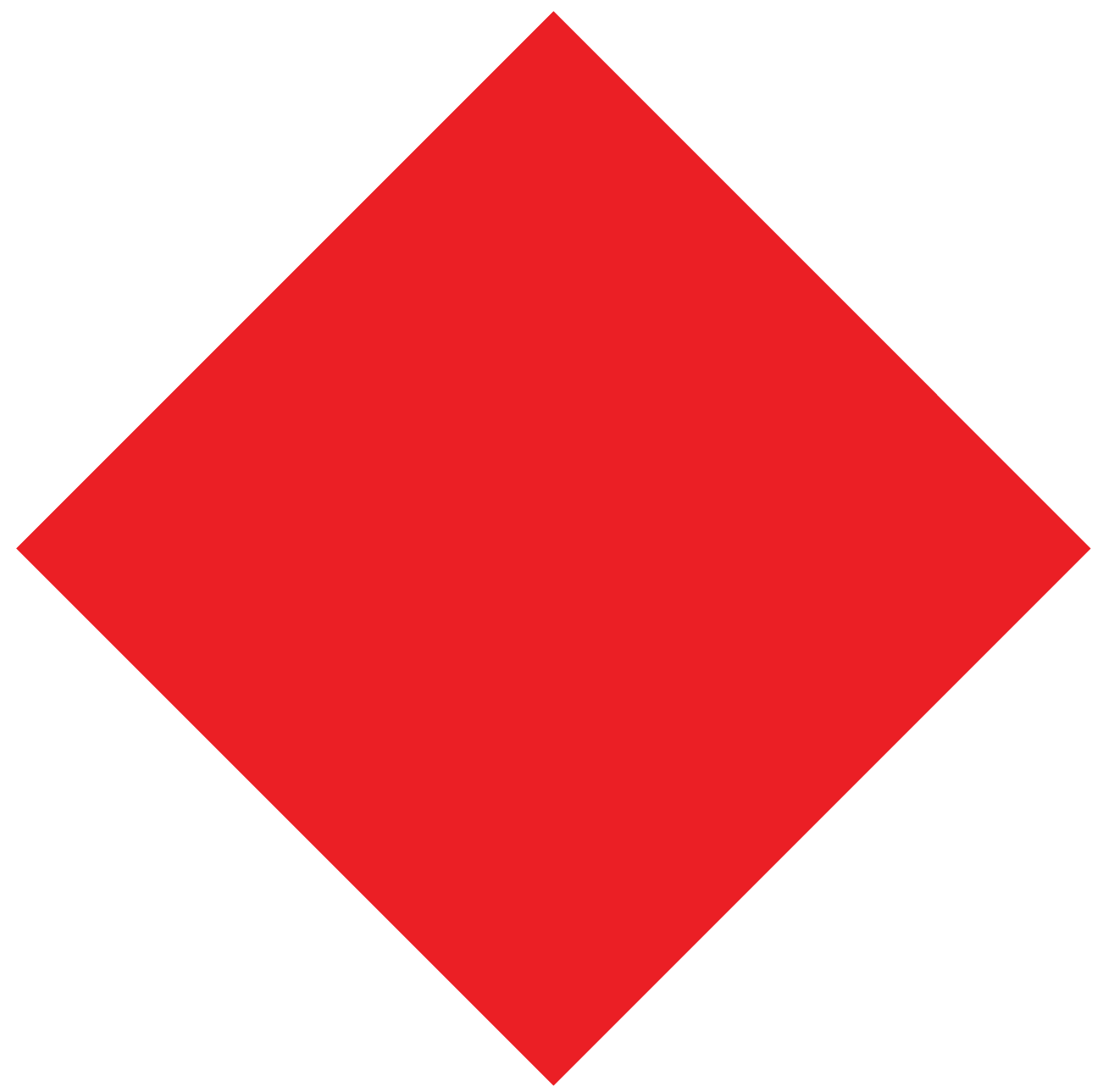})
             and  
             the optimal approximation distribution
             (\protect\includegraphics[width=2mm,clip]{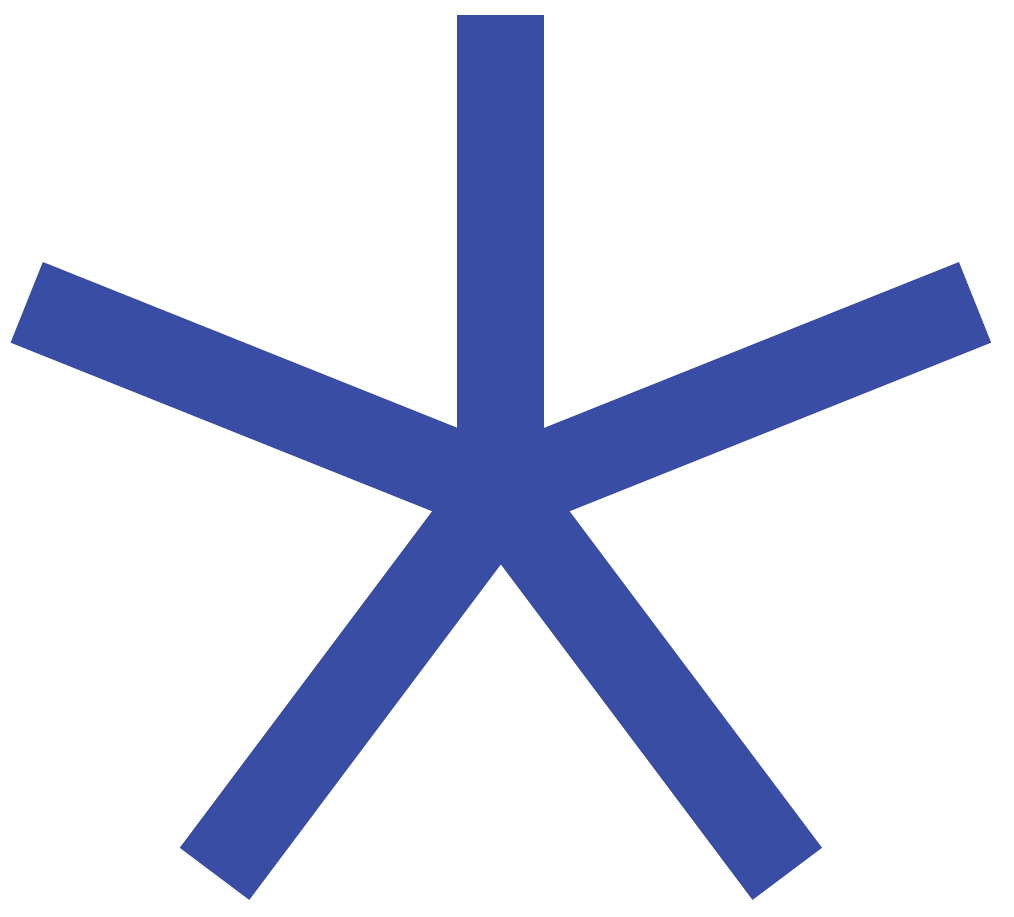})
             \emph{depending on the scaling $x_\mu$}
             of the fronts $f_p$ (cf.\ \defref{OPT}).
             We omit the values of the $y$-axis as we are
             only interested in the relative comparison 
             (\protect\includegraphics[width=2mm,clip]{pdfs/icon-hyp-eps-converted-to.pdf}
             vs.~\protect\includegraphics[width=2mm,clip]{pdfs/icon-app-eps-converted-to.pdf})
             for each front $f_p$.
             Note that as analytically predicted in \thmref{lin},
             both curves coincide in (c) for the linear function $f_1$
             independent of the scaling.
             }
    \label{fig:x3}
\end{figure*}

In the example of \figrefs{sym}{asy} the optimal multiplicative approximation 
factor for the symmetric and asymmetric case is
1.021 (\figref{symapp}) 
and 1.030 (\figref{asyapp}), 
respectively,
while the hypervolume only achieves an approximation of
1.025 (\figref{symhyp}) 
and 1.038 (\figref{asyhyp}), 
respectively.
Therefore in the symmetric and asymmetric case of $f_2$ the hypervolume is not calculating
the set of points with the optimal multiplicative approximation.

\begin{figure*}[tbp]
    \centering
    \subfigure[\APP{\OPT{\mu}{\fsym_{1/3}}}.]{
        \renumbermu{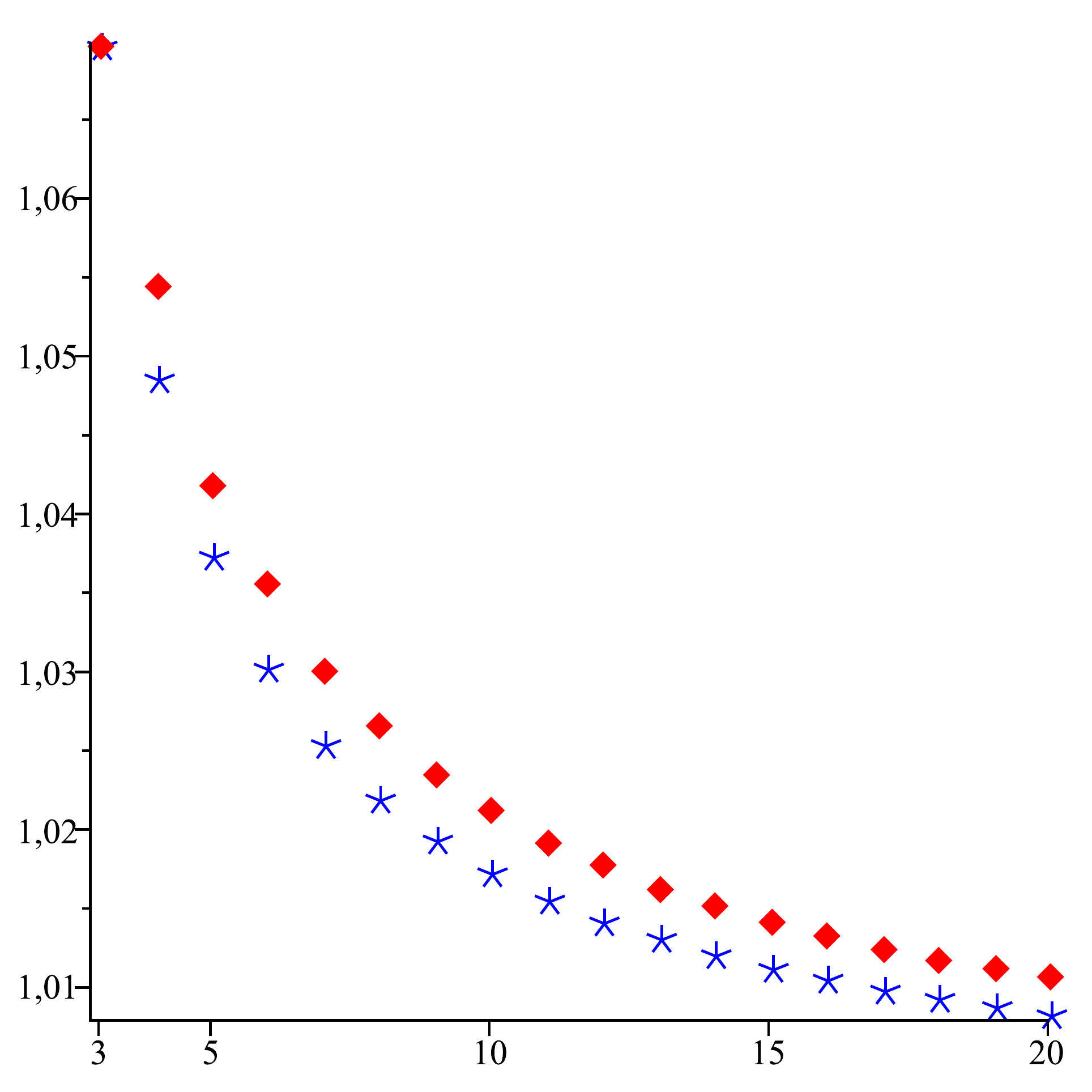}
    }
    \subfigure[\APP{\OPT{\mu}{\fsym_{1/2}}}.]{
        \renumbermu{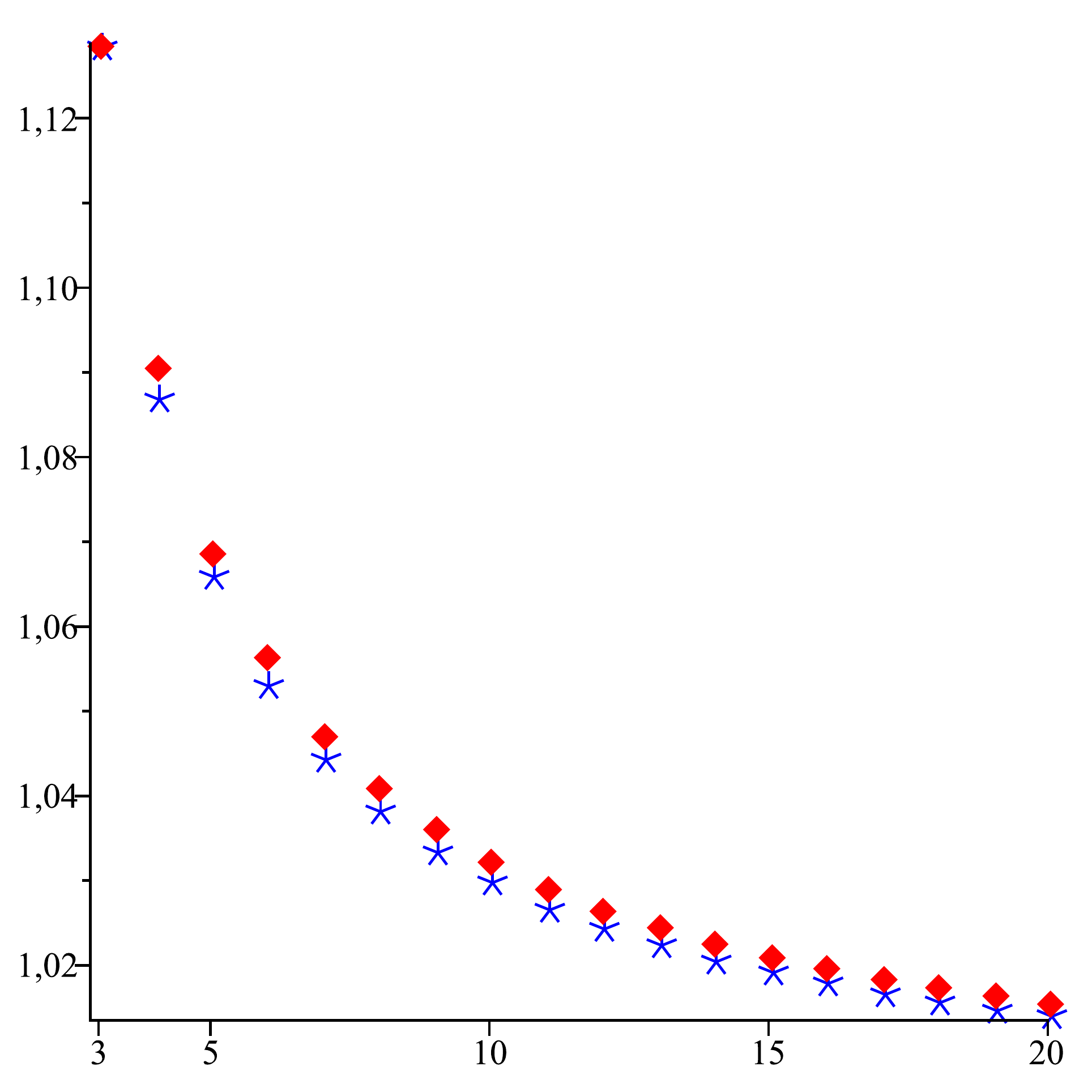}
    }
    \subfigure[\APP{\OPT{\mu}{\fsym_{1}}}.]{
        \renumbermu{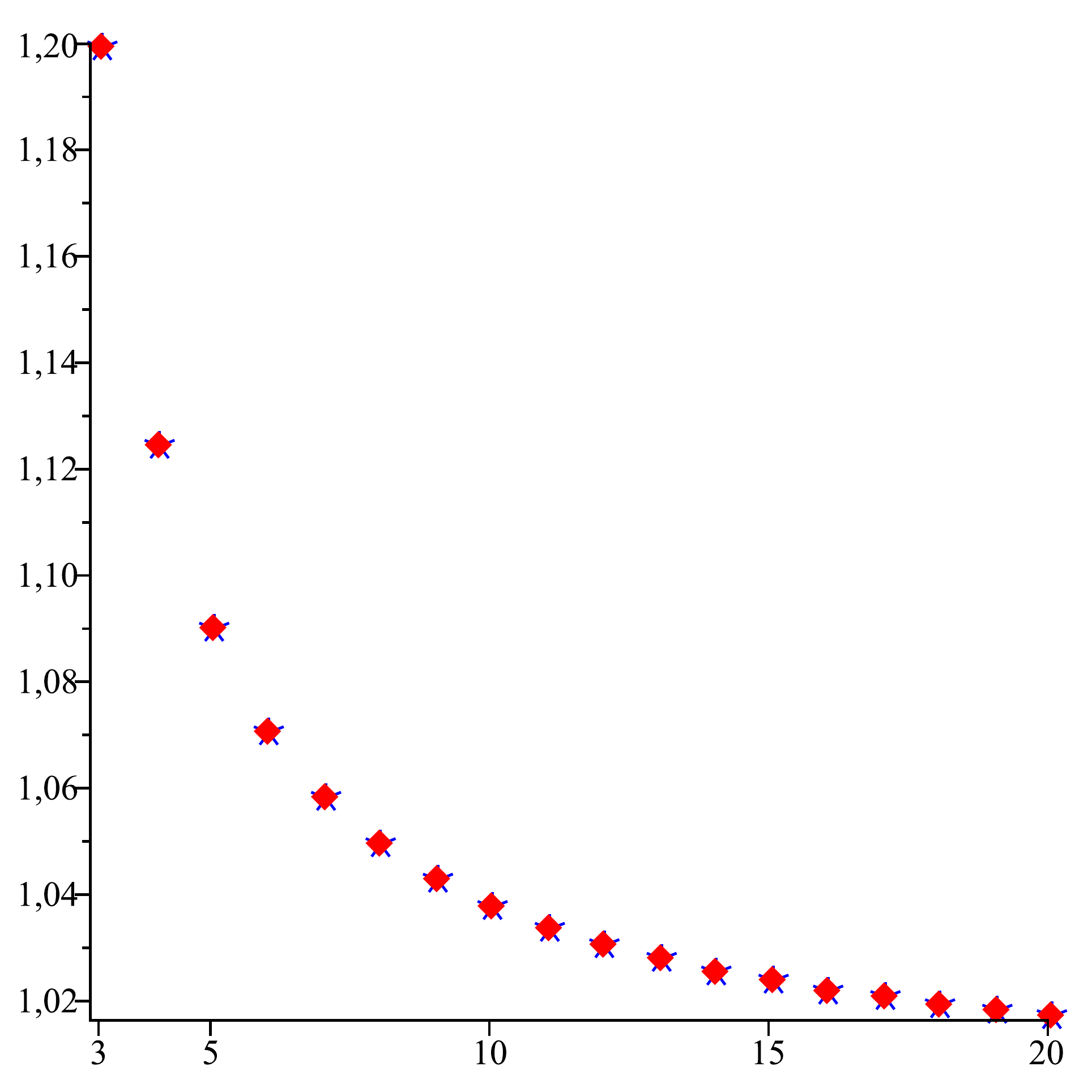}
    }
    \subfigure[\APP{\OPT{\mu}{\fsym_{2}}}.]{
        \renumbermu{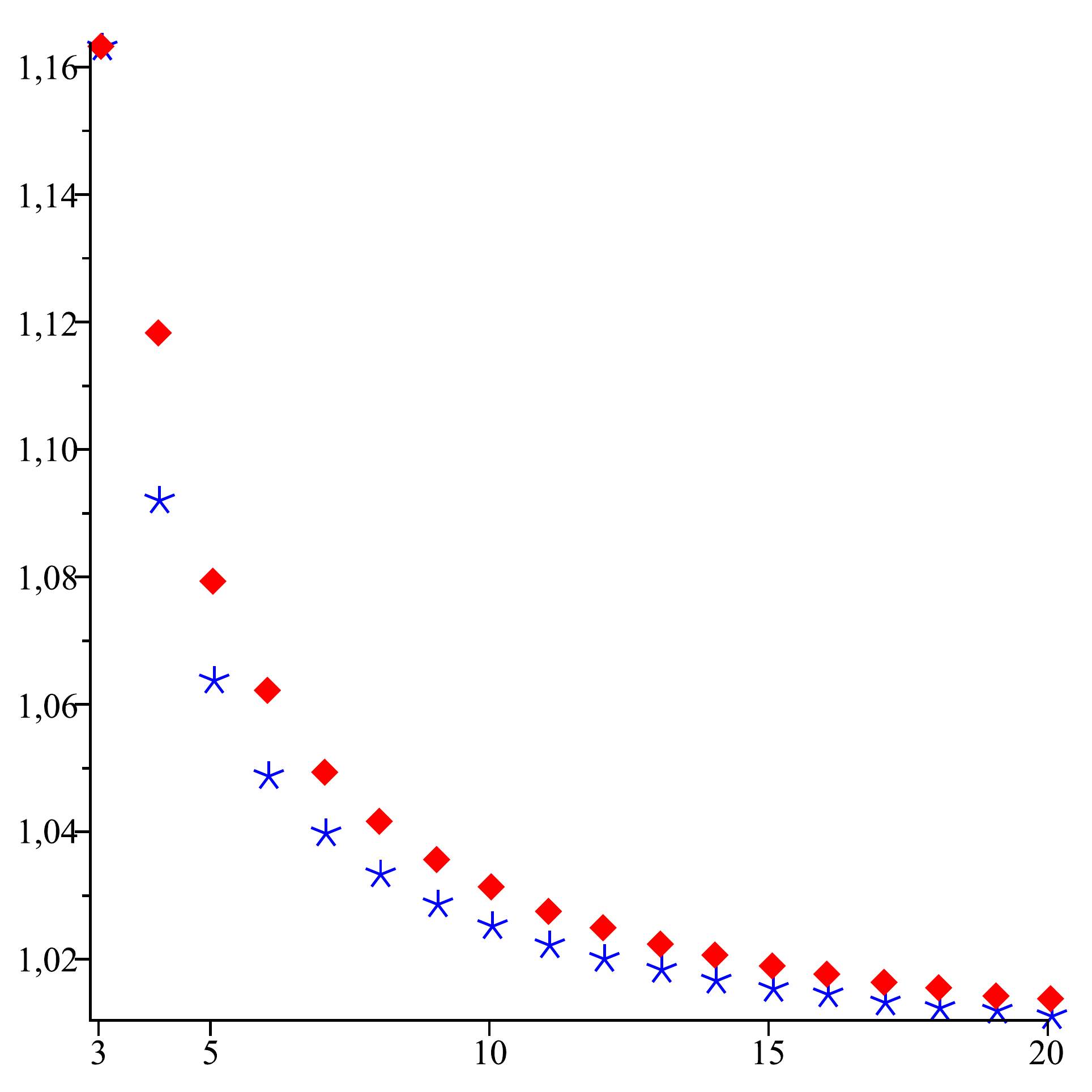}
    }
    \subfigure[\APP{\OPT{\mu}{\fsym_{3}}}.]{
        \renumbermu{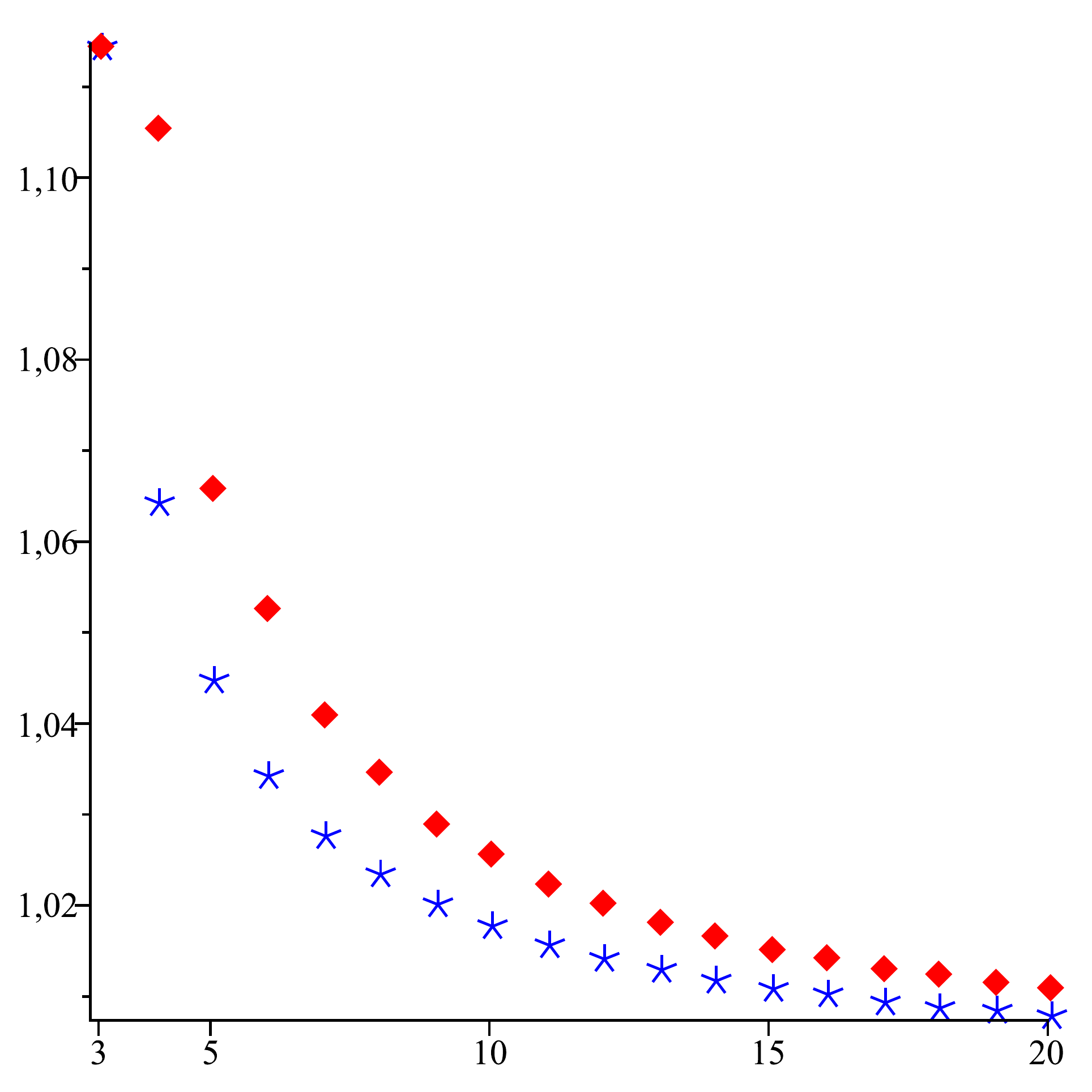}
    }\\
    \centering
    \subfigure[\APP{\OPT{\mu}{\fasy_{1/3}}}.]{
        \renumbermuu{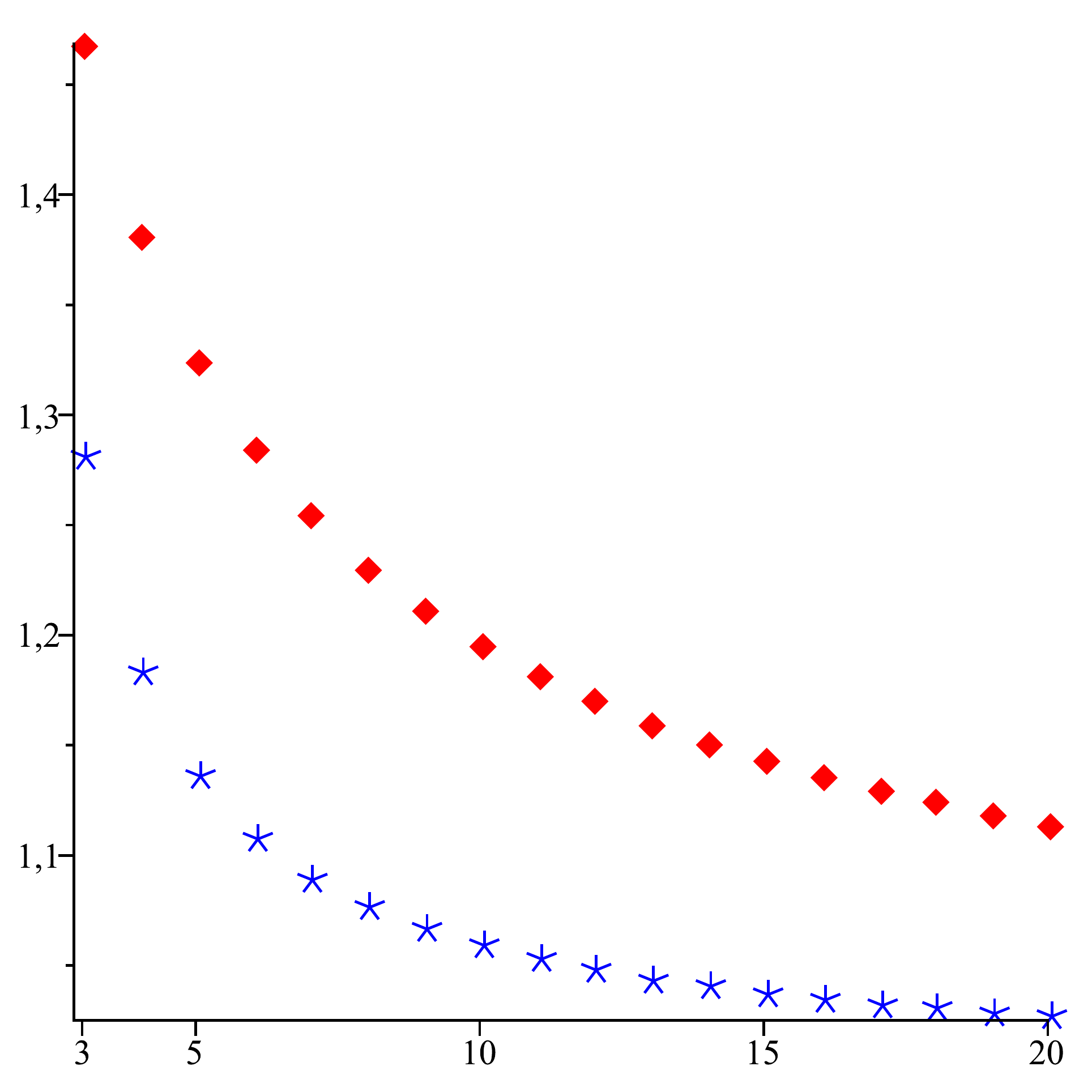}
        \label{fig:mu:asy:13}
    }
    \subfigure[\APP{\OPT{\mu}{\fasy_{1/2}}}.]{
        \renumbermuu{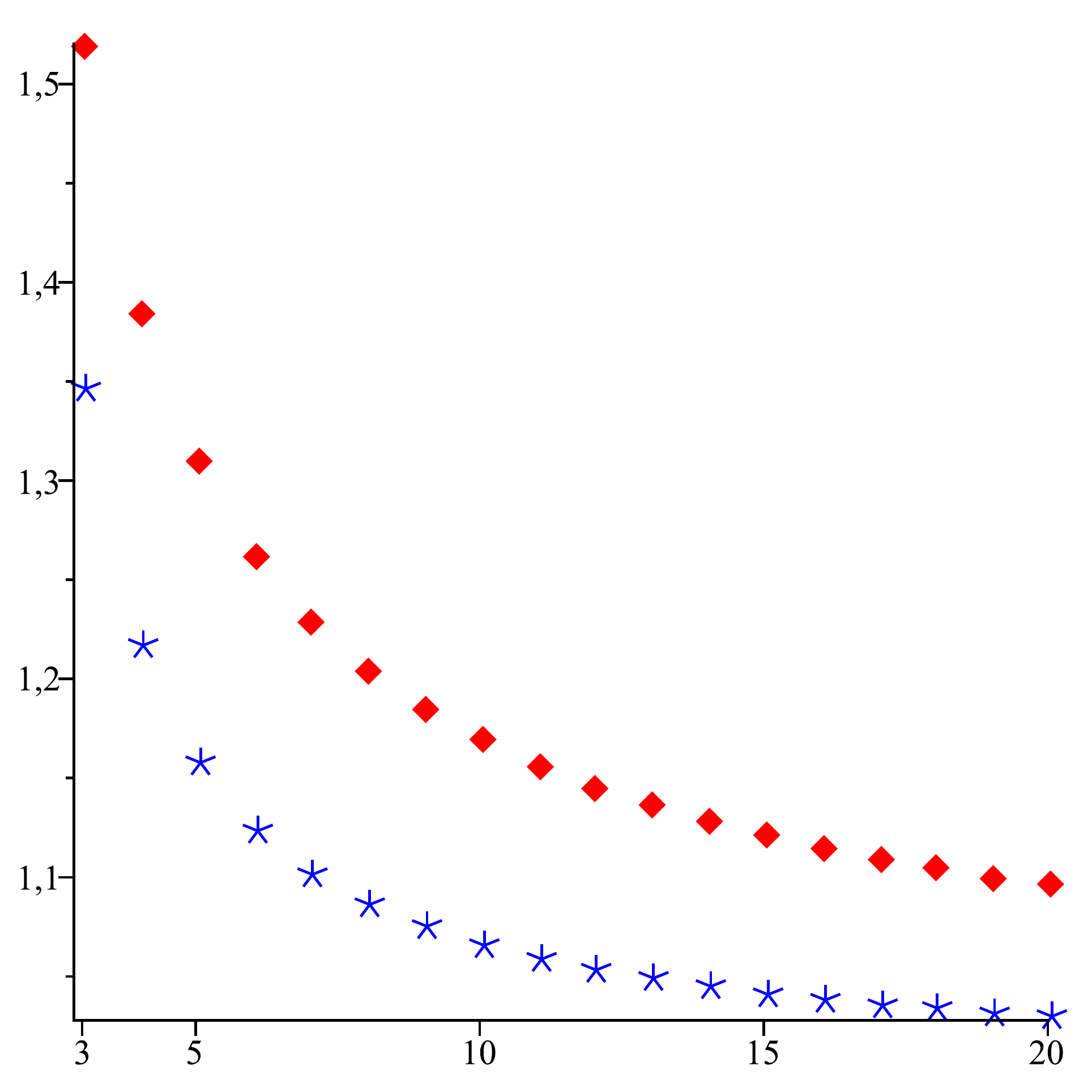}
        \label{fig:mu:asy:12}
    }
    \subfigure[\APP{\OPT{\mu}{\fasy_{1}}}.]{
        \renumbermu{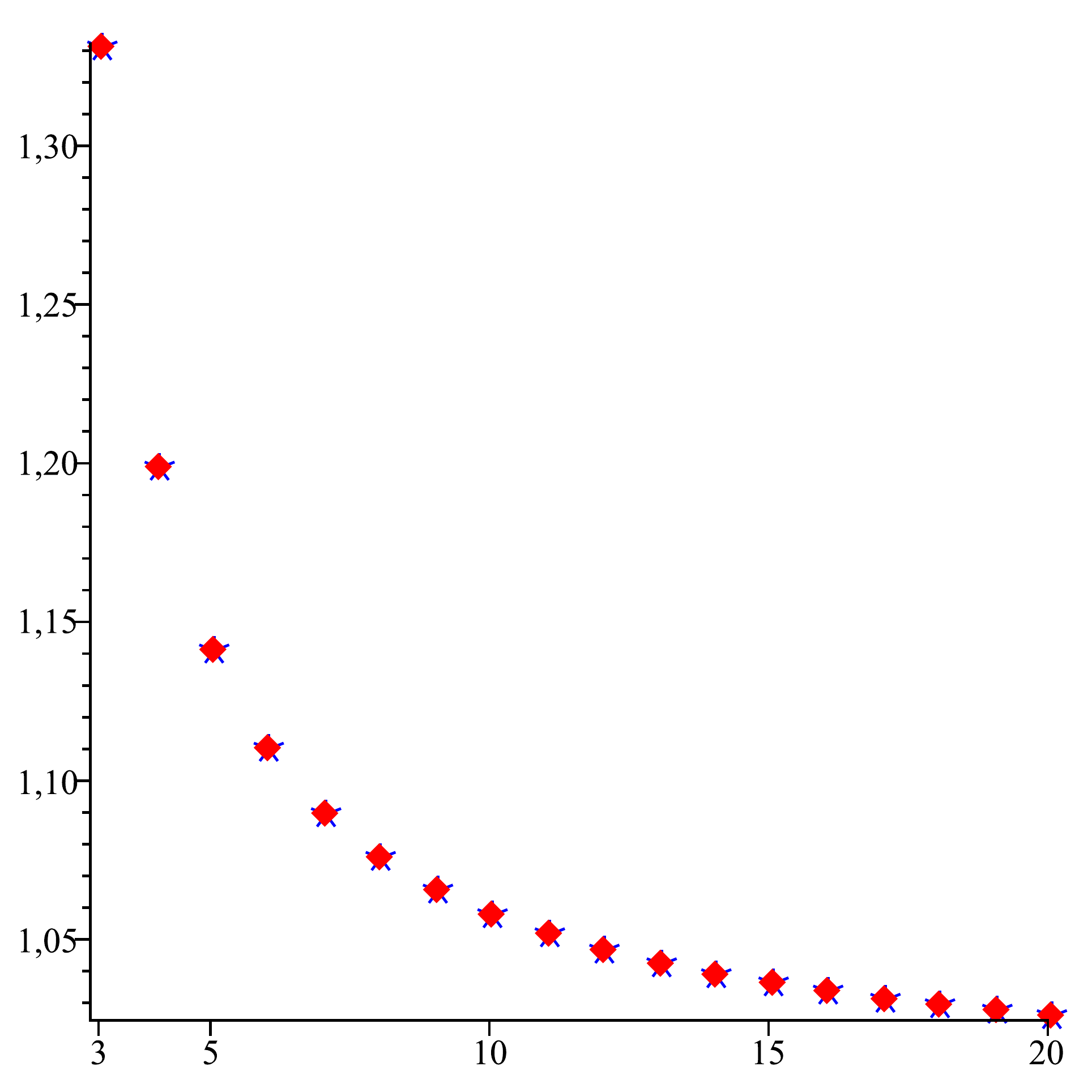}
    }
    \subfigure[\APP{\OPT{\mu}{\fasy_{2}}}.]{
        \renumbermu{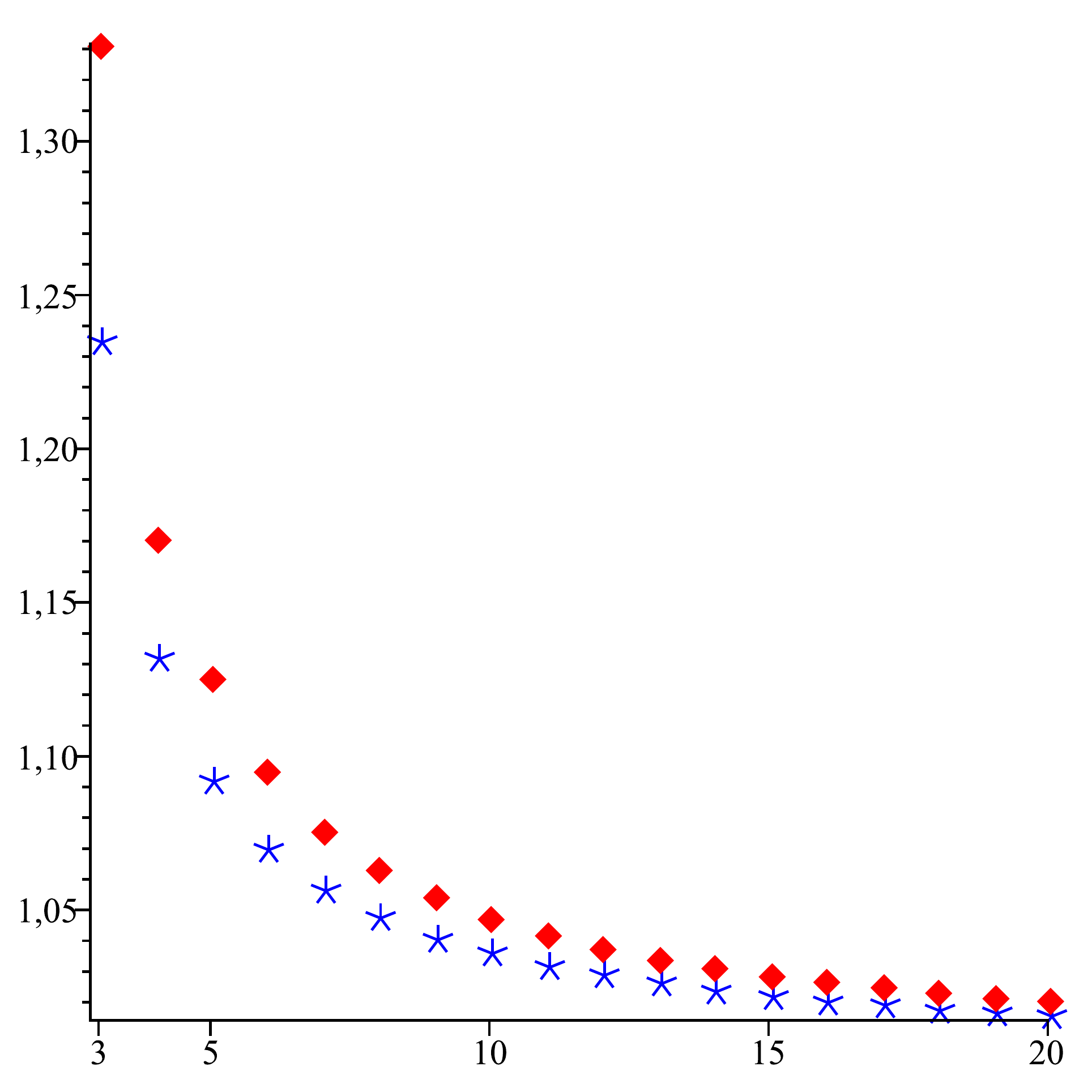}
    }
    \subfigure[\APP{\OPT{\mu}{\fasy_{3}}}.]{
        \renumbermu{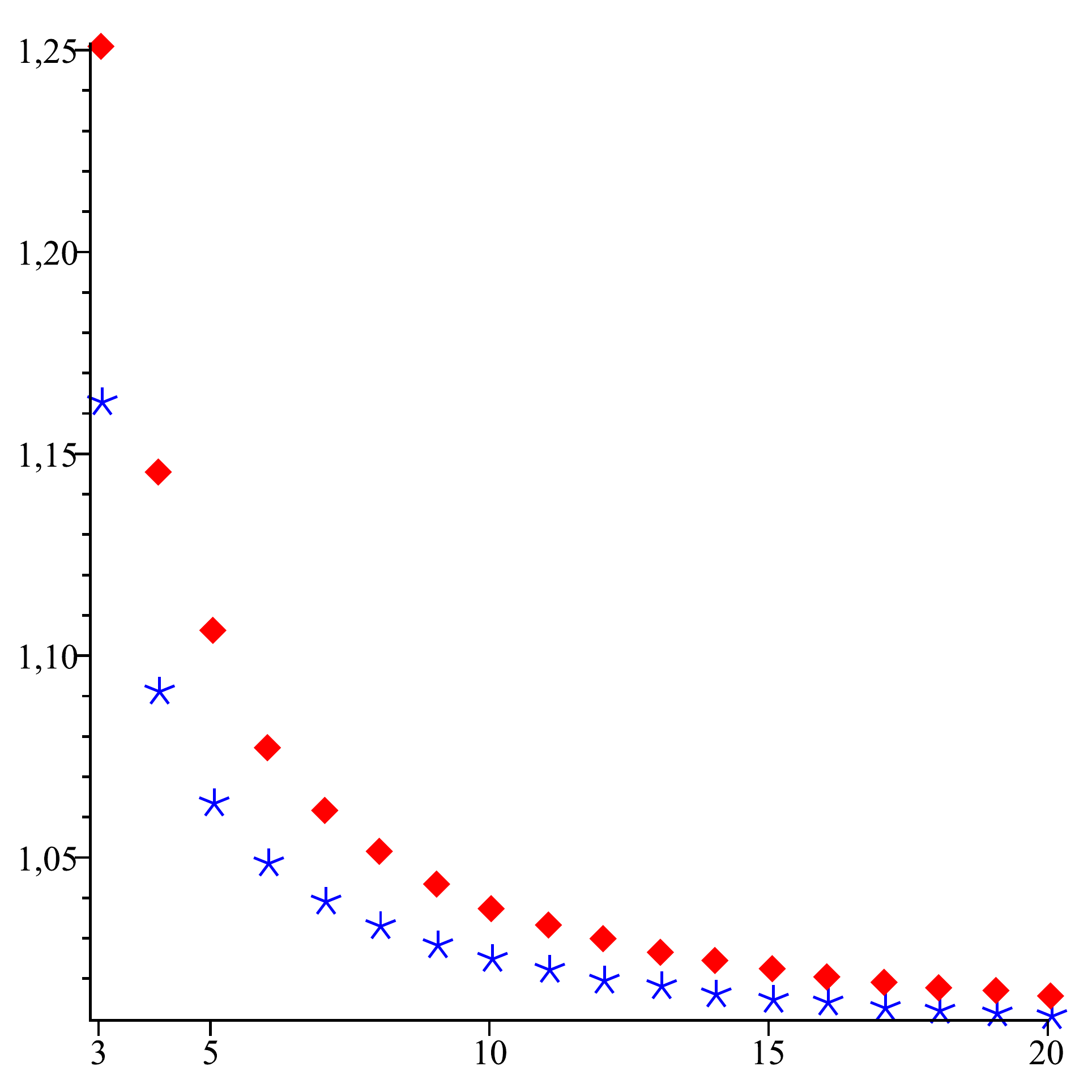}
    }
    \caption{Approximation ratio 
             of the optimal hypervolume distribution
             (\protect\includegraphics[width=2mm,clip]{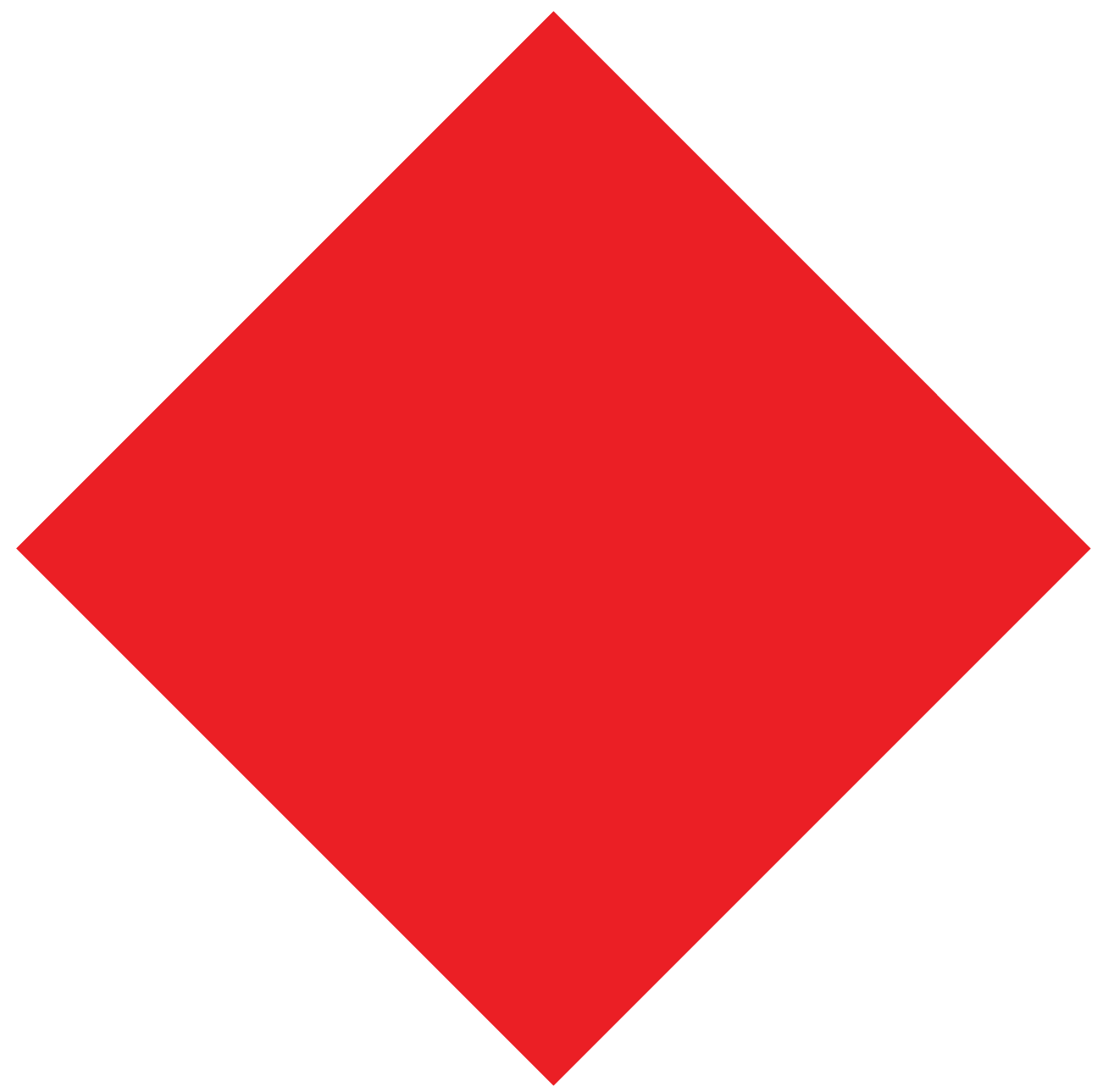})
             and  
             the optimal approximation distribution
             (\protect\includegraphics[width=2mm,clip]{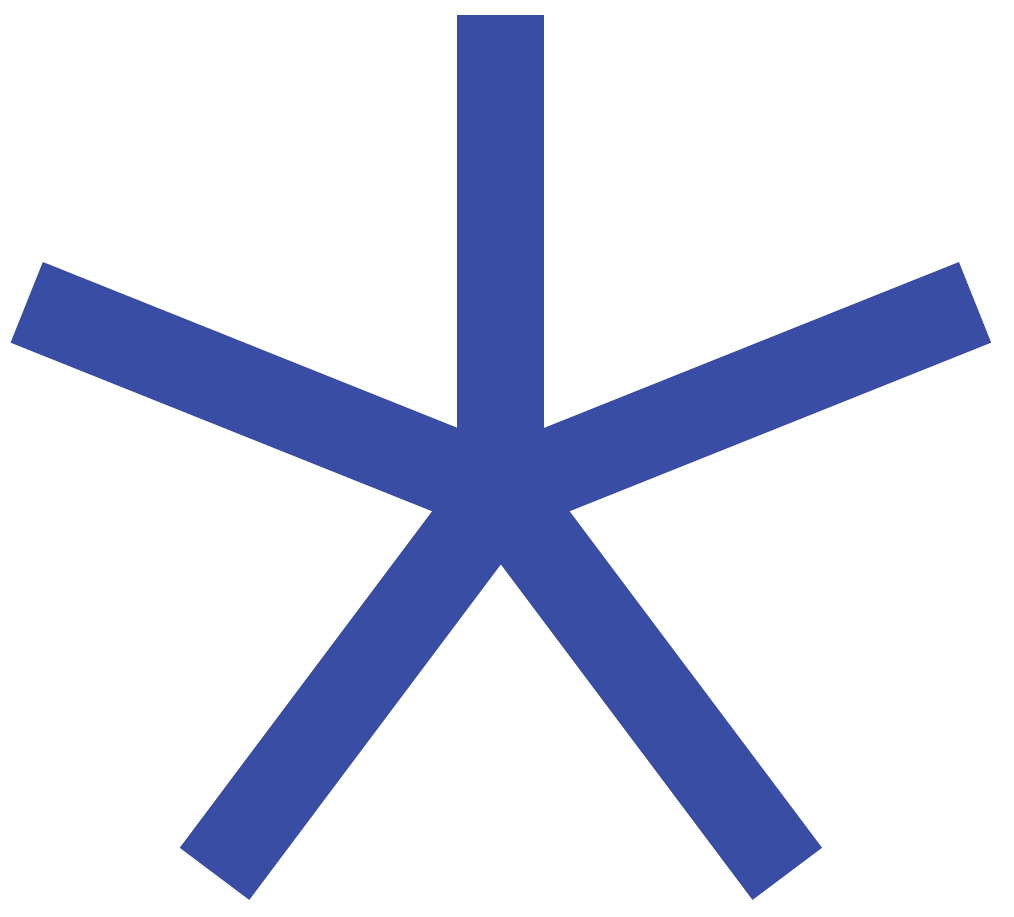})
             \emph{depending on the number of points $\mu$}
             for symmetric and asymmetric fronts $f_p$
             and different parameters $p$ (cf.\ \defref{OPT}).
             We omit the values of the $y$-axis as we are
             only interested in the relative comparison 
             (\protect\includegraphics[width=2mm,clip]{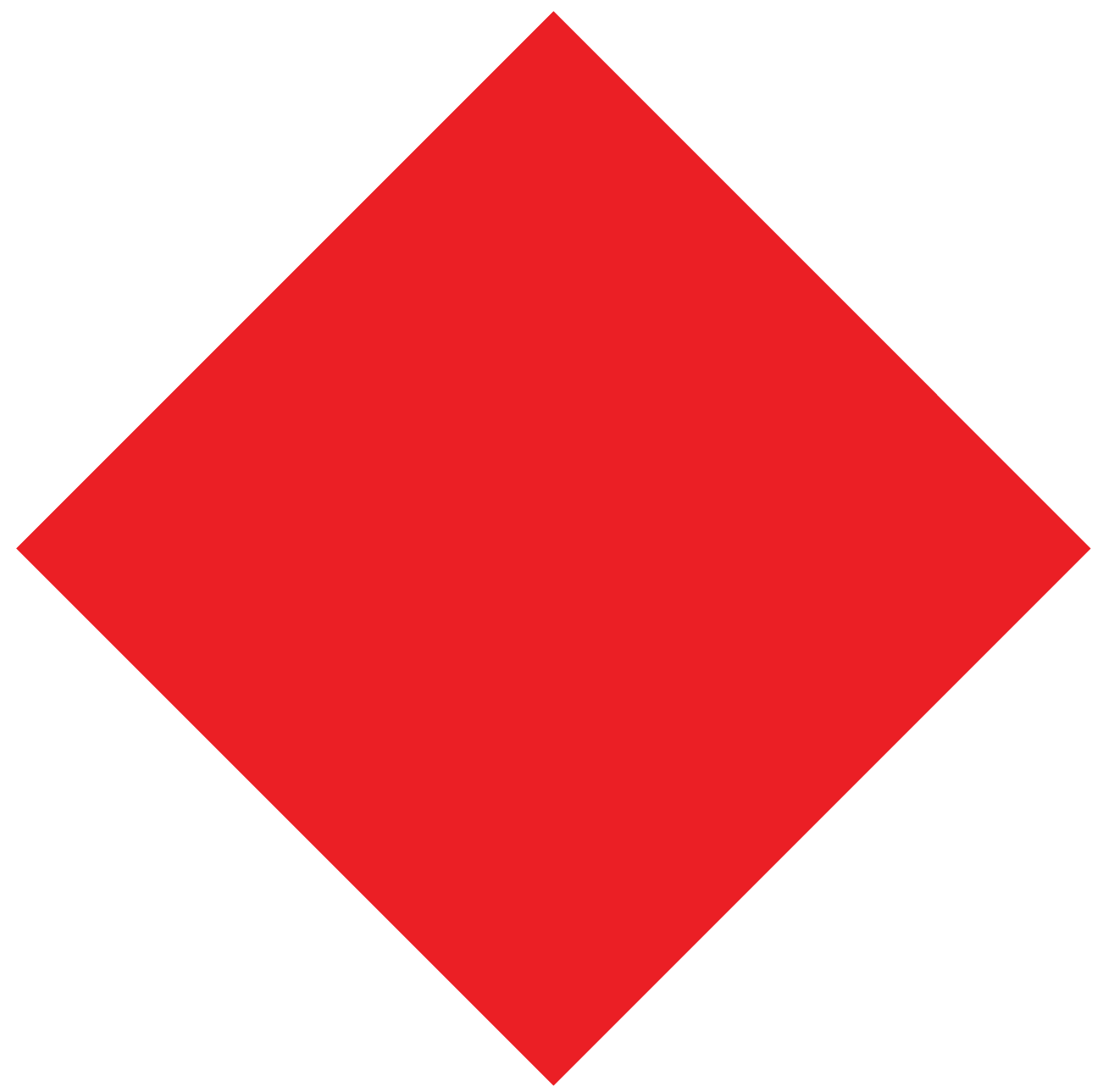}
             vs.~\protect\includegraphics[width=2mm,clip]{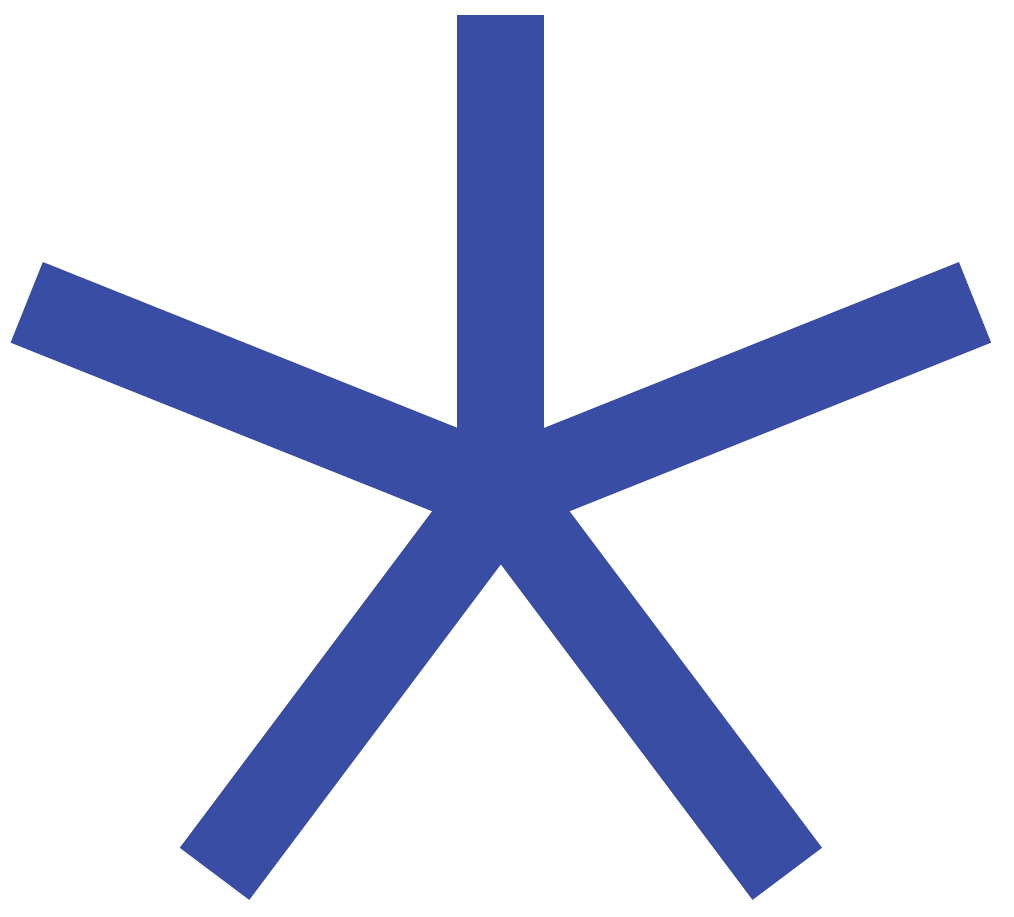})
             for each front $f_p$.
             Note that (c) and (h) show that the approximation ratio of the optimal hypervolume distribution
             \APP{\OPTHYP{\mu}{\fsym_{1}}} and
             the optimal approximation distribution
             \APP{\OPTHYP{\mu}{\fsym_{1}}} 
             are equivalent for all examined $\mu$.
             That maximizing the hypervolume yields the optimal approximation ratio
             can also be observed for all symmetric $\fsym_{p}$ with $\mu=3$ in
             (a)--(e).
             }
    \label{fig:mu}
\end{figure*}

\begin{figure*}[tbp]
    \centering
    \subfigure[\APP{\OPT{3}{\fsym_{p}}}.]{
        \renumberp{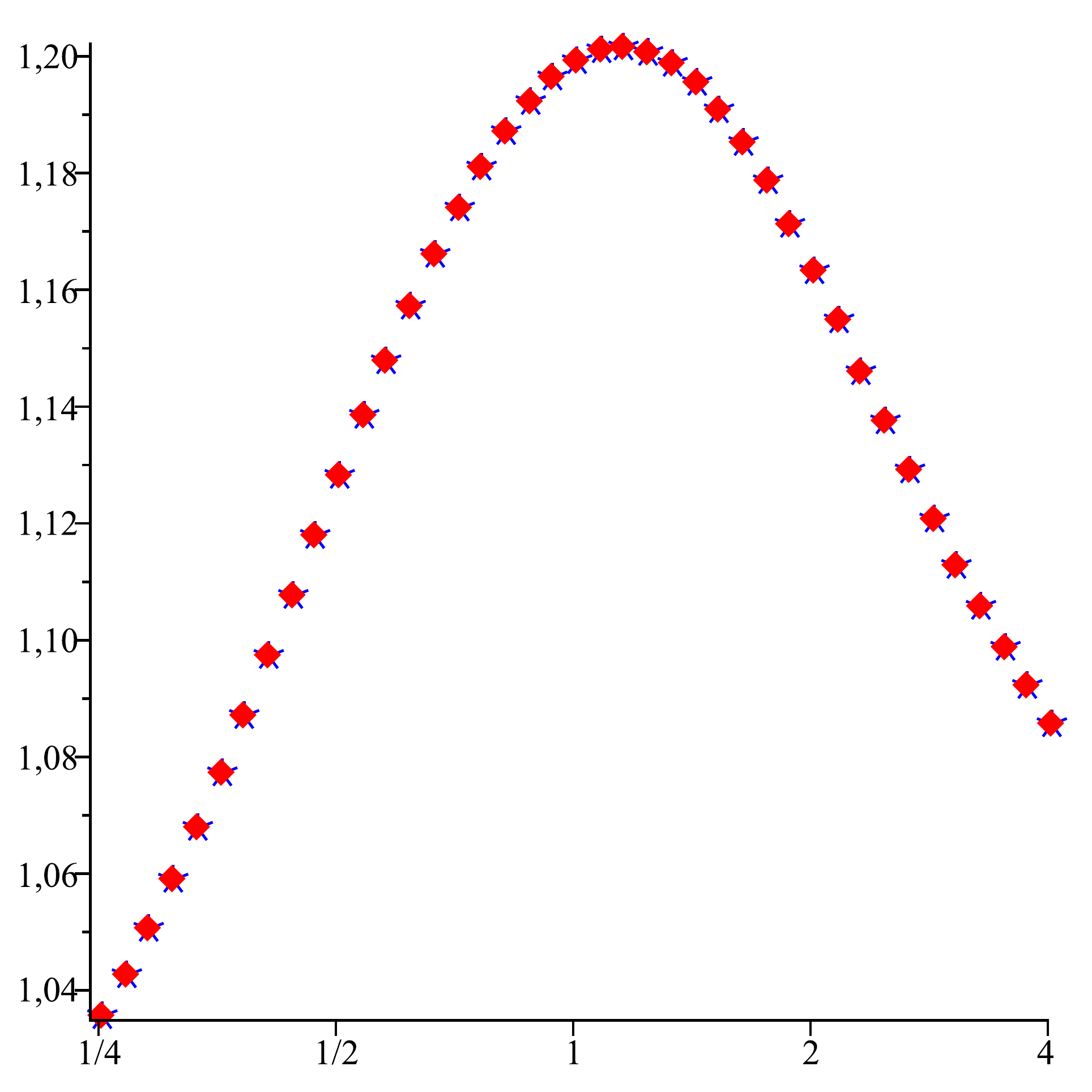}
        \label{fig:p:sym:mu1}
    }
    \hspace*{.02\textwidth}
    \subfigure[\APP{\OPT{4}{\fsym_{p}}}.]{
        \renumberp{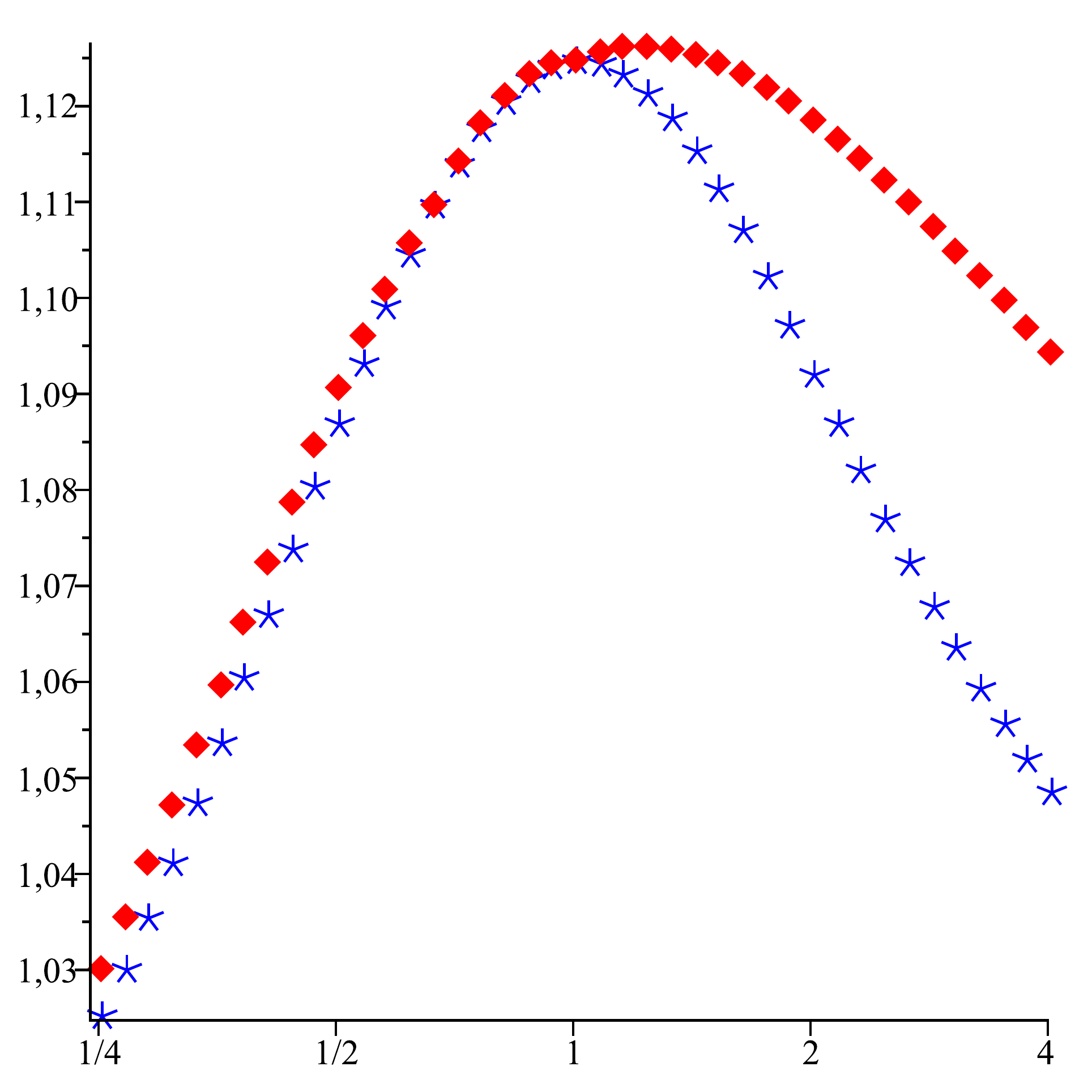}
        \label{fig:p:sym:mu2}
    }
    \hspace*{.02\textwidth}
    \subfigure[\APP{\OPT{5}{\fsym_{p}}}.]{
        \renumberp{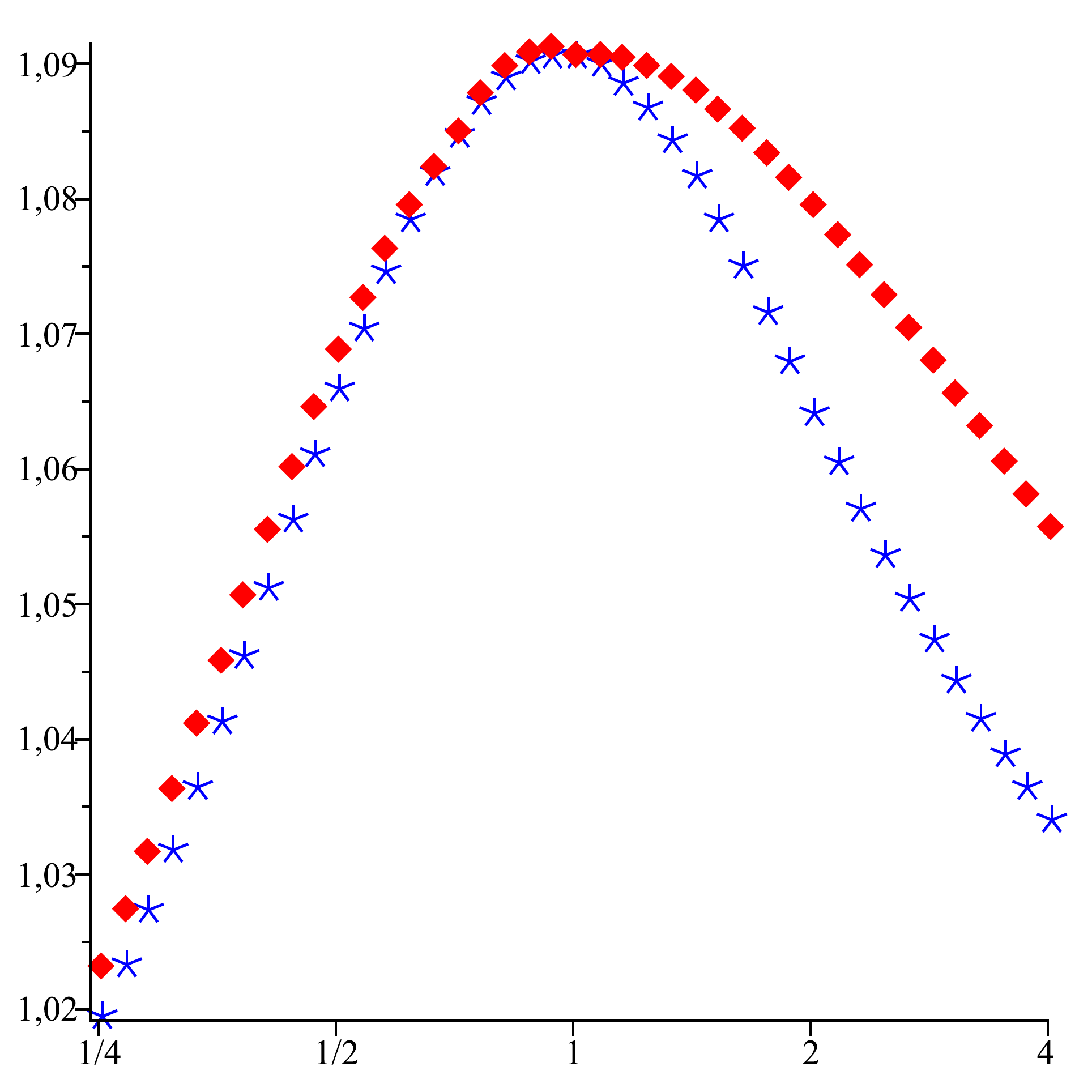}
        \label{fig:p:sym:mu3}
    }
    \hspace*{.02\textwidth}
    \subfigure[\APP{\OPT{6}{\fsym_{p}}}.]{
        \renumberp{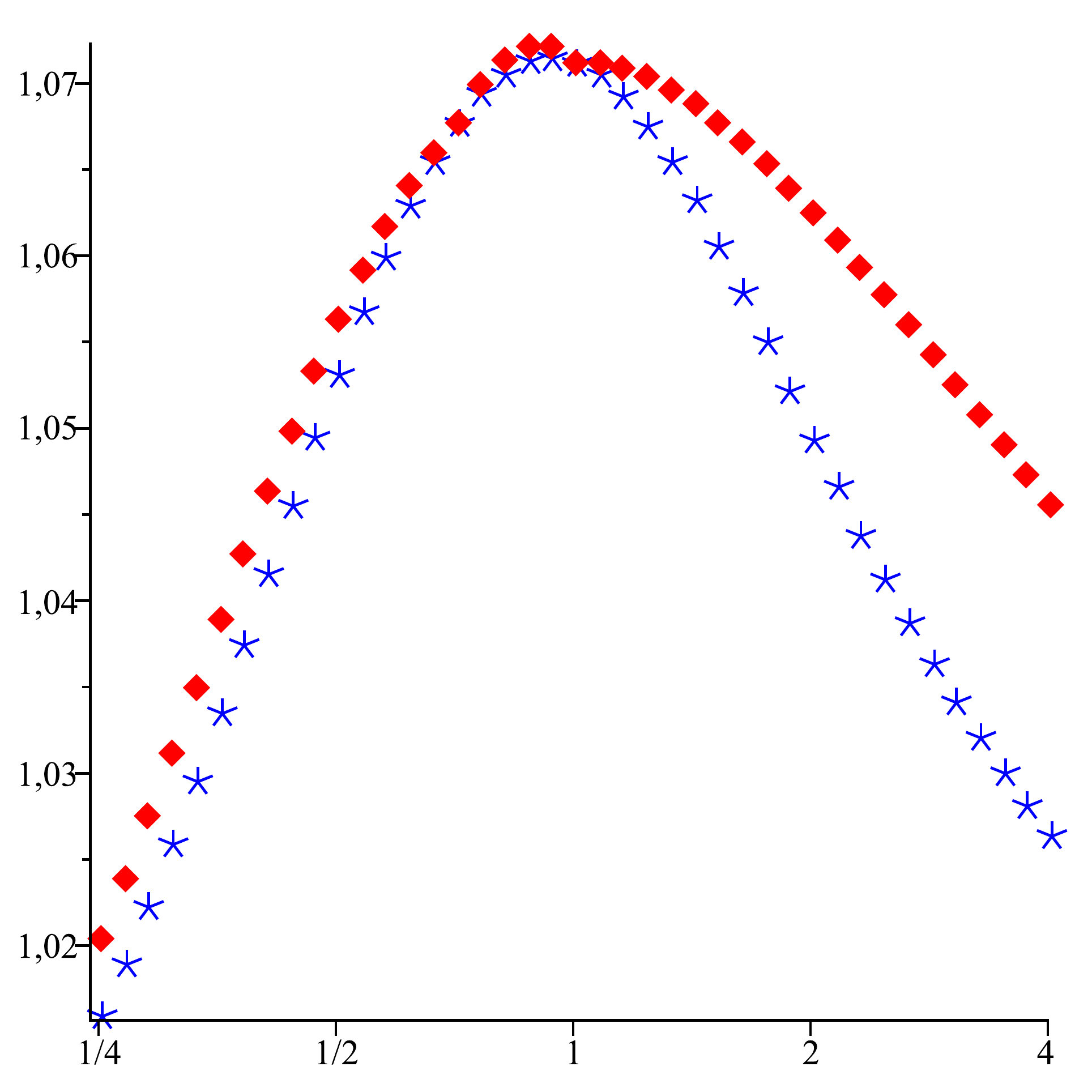}
        \label{fig:p:sym:mu4}
    }\\
    \centering
    \subfigure[\APP{\OPT{3}{\fasy_{p}}}.]{
        \renumberpp{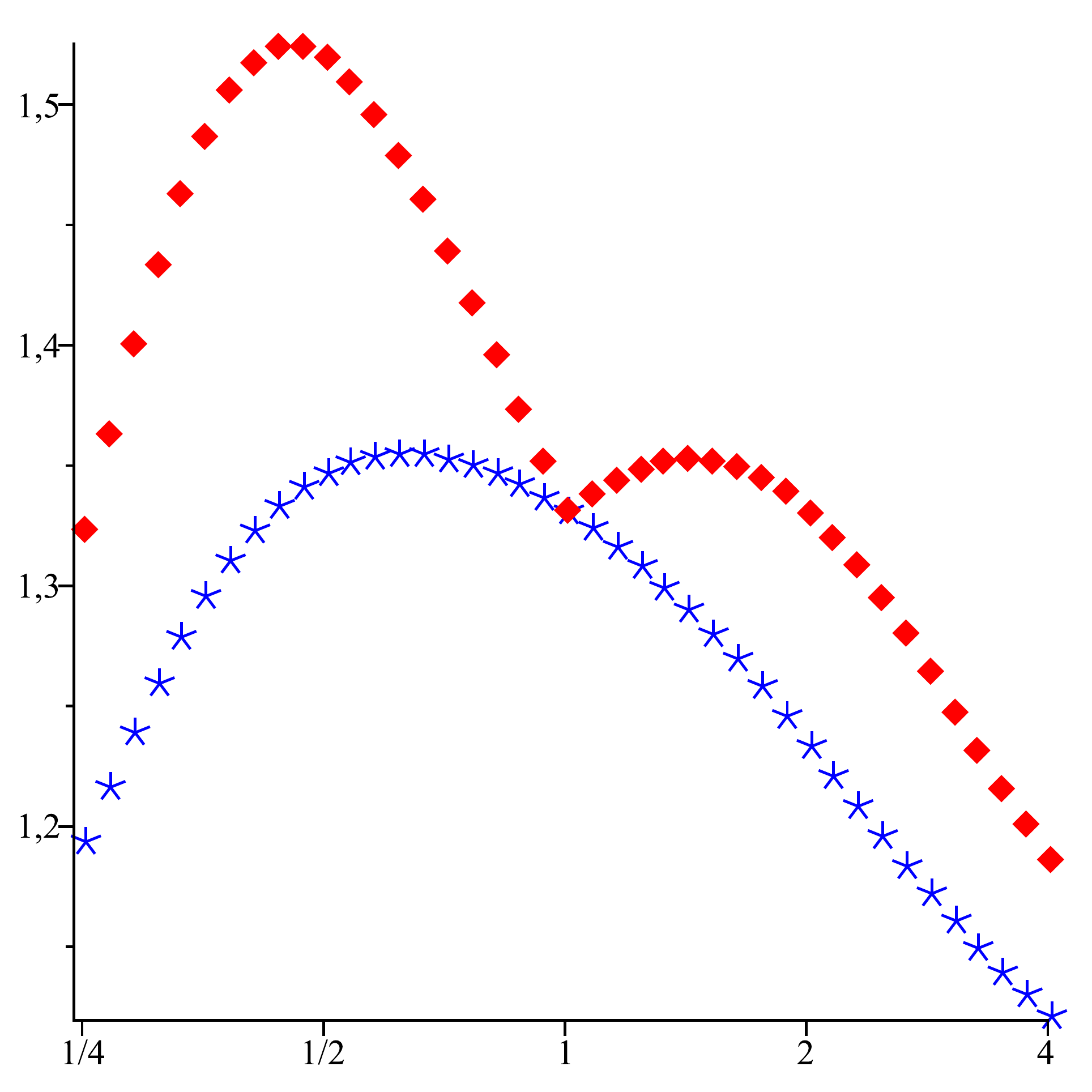}
        \label{fig:p:asy:mu1}
    }
    \hspace*{.02\textwidth}
    \subfigure[\APP{\OPT{4}{\fasy_{p}}}.]{
        \renumberpp{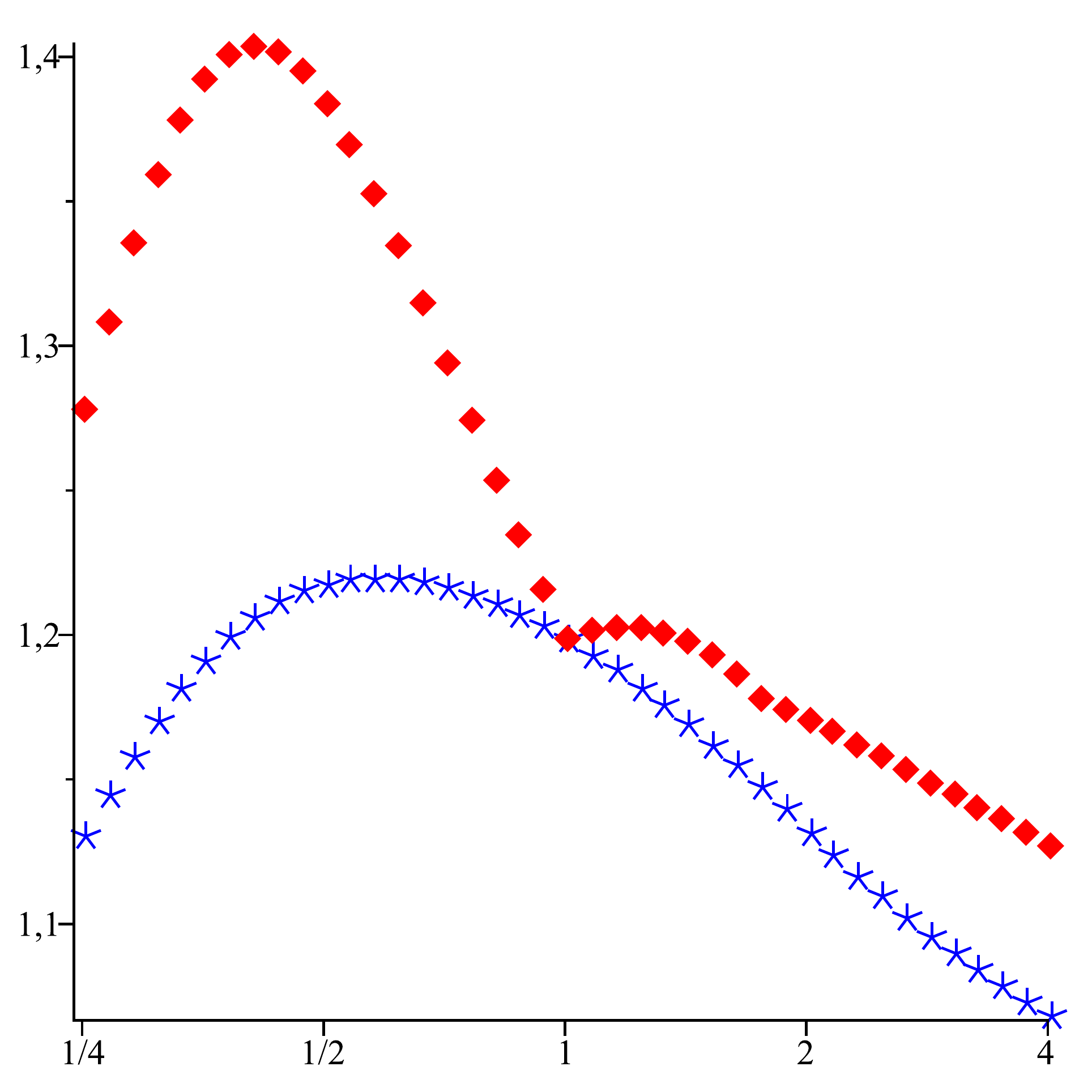}
        \label{fig:p:asy:mu2}
    }
    \hspace*{.02\textwidth}
    \subfigure[\APP{\OPT{5}{\fasy_{p}}}.]{
        \renumberp{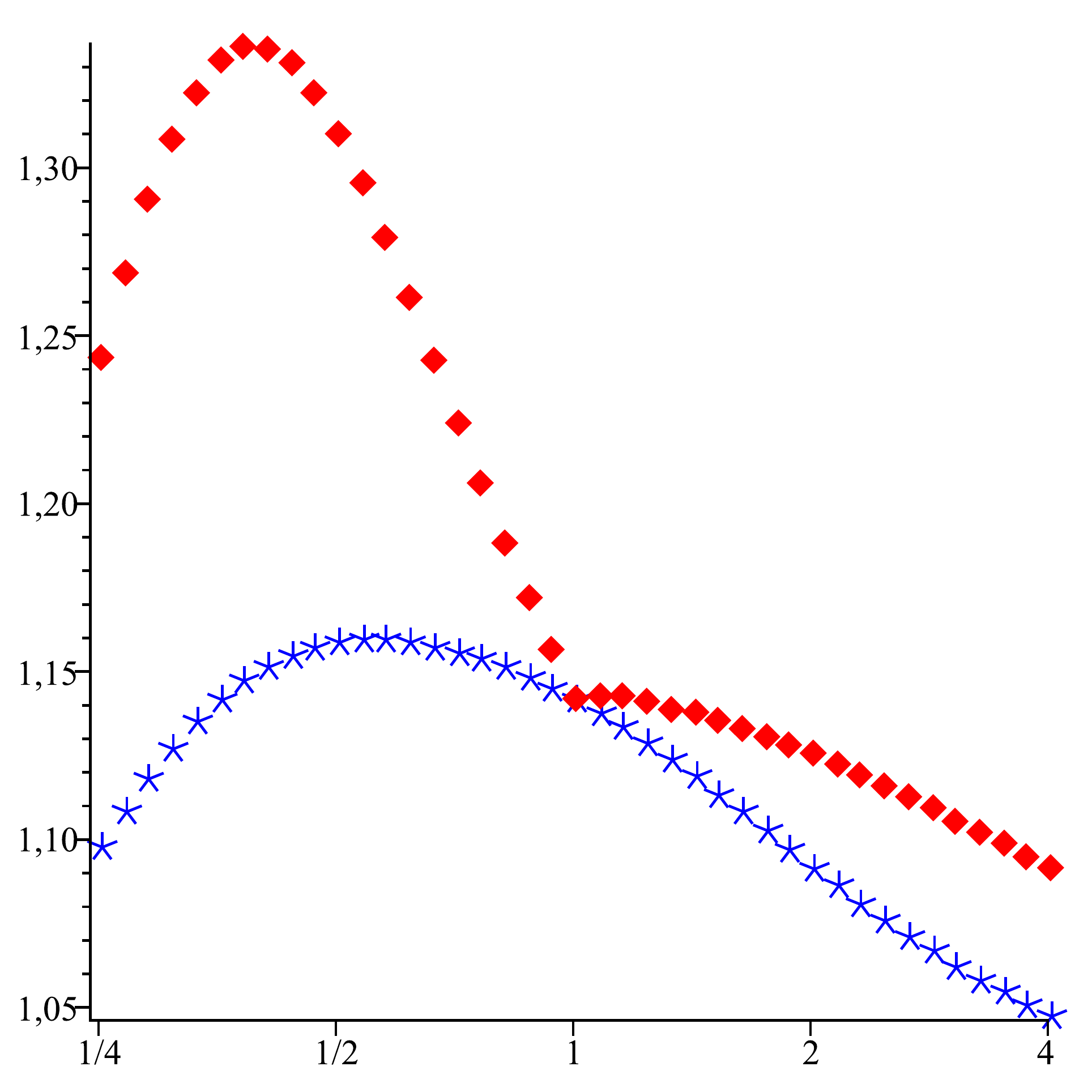}
        \label{fig:p:asy:mu3}
    }
    \hspace*{.02\textwidth}
    \subfigure[\APP{\OPT{6}{\fasy_{p}}}.]{
        \renumberp{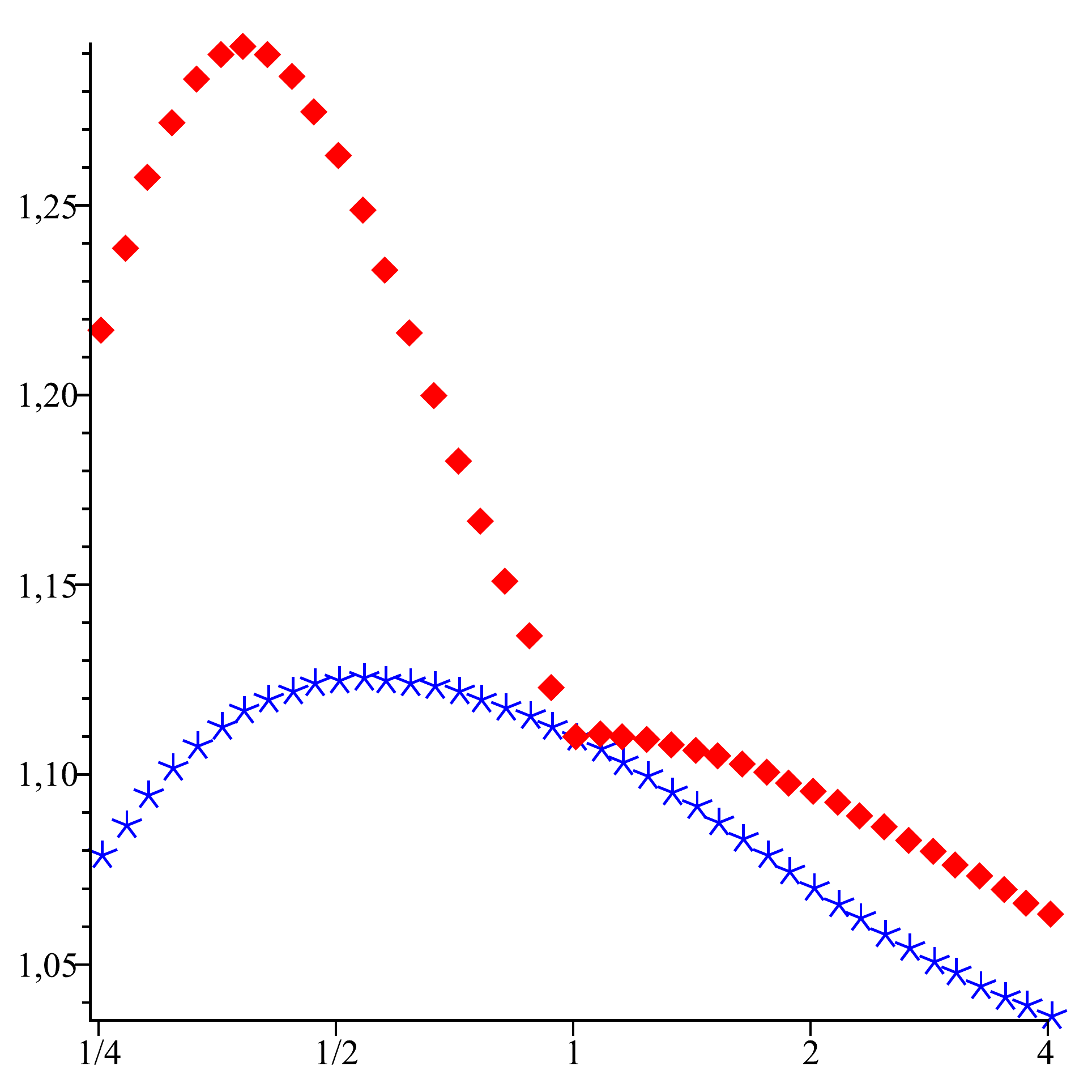}
        \label{fig:p:asy:mu4}
    }
    \caption{Approximation ratio 
             of the optimal hypervolume distribution
             (\protect\includegraphics[width=2mm,clip]{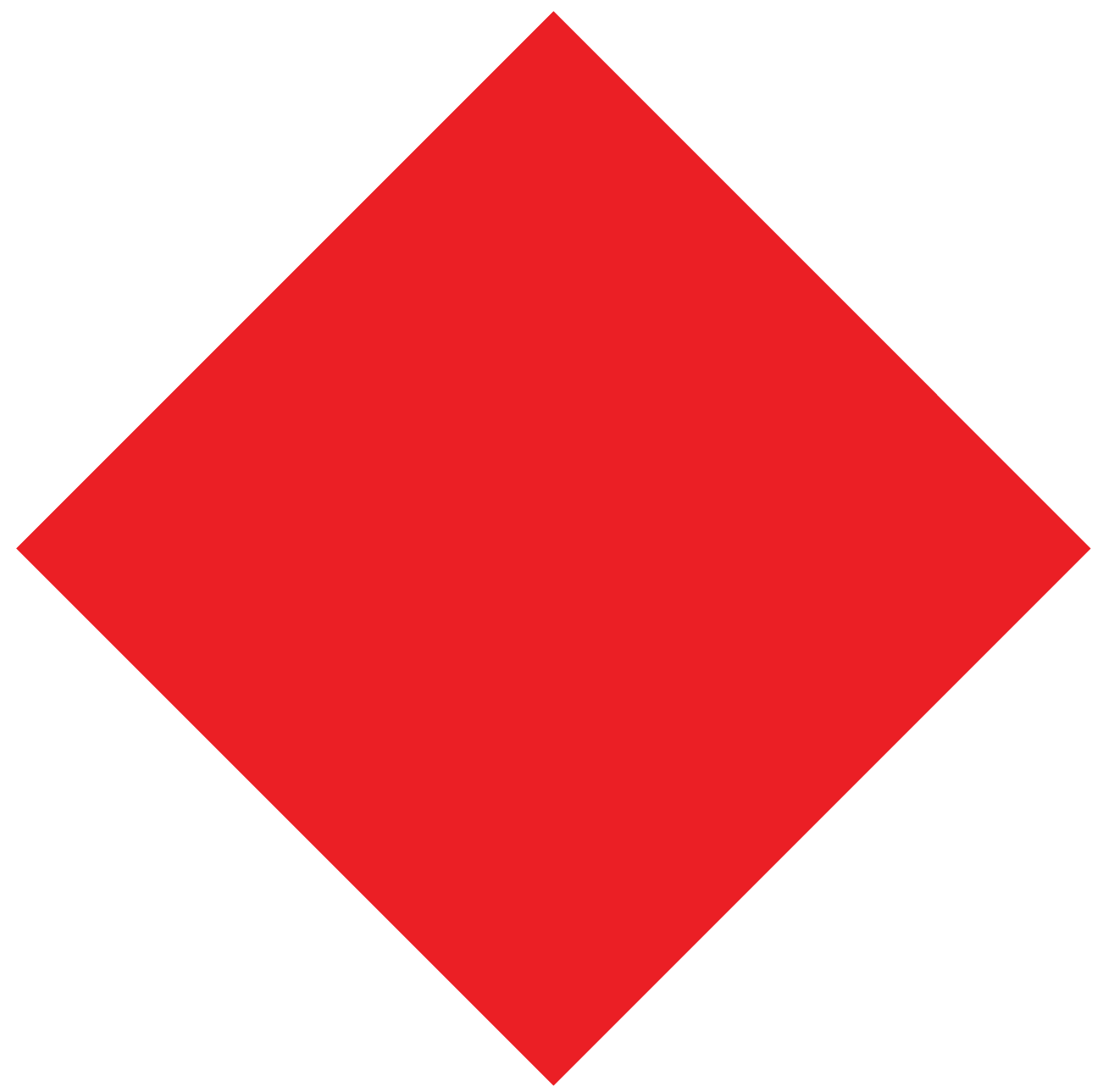})
             and  
             the optimal approximation distribution
             (\protect\includegraphics[width=2mm,clip]{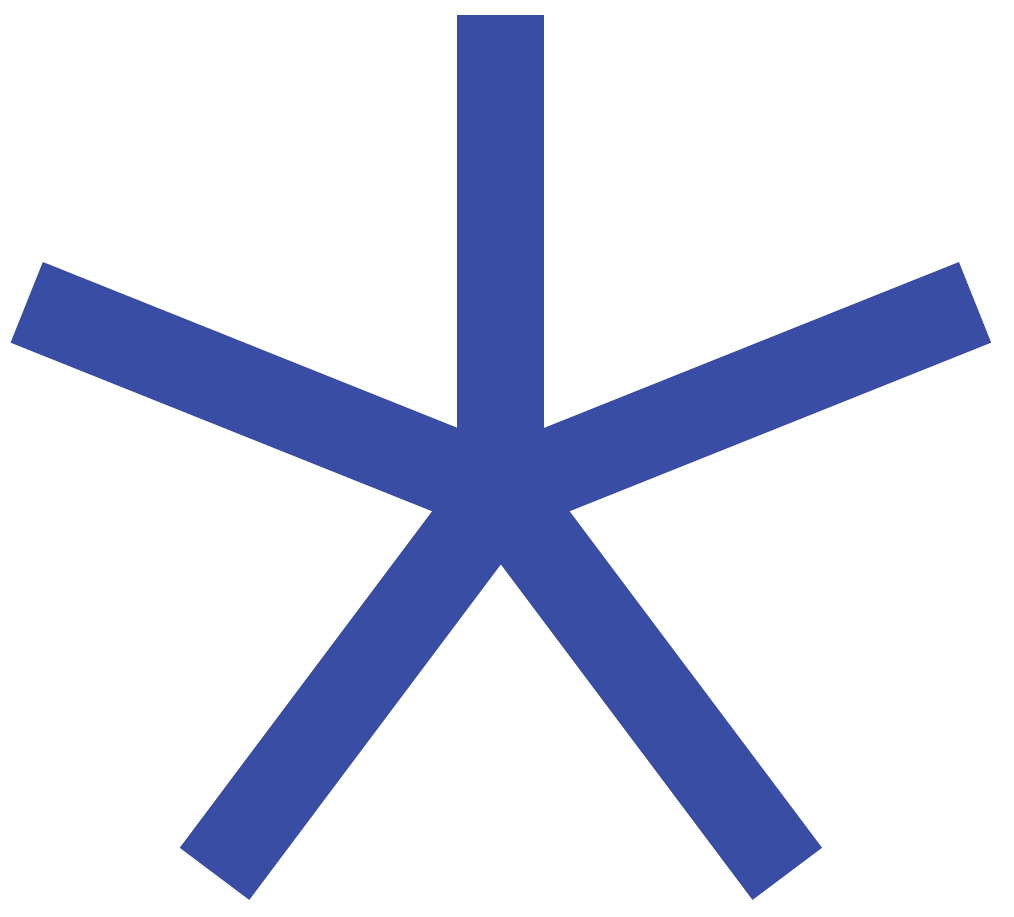})
             \emph{depending on the convexity/concavity parameter $p$}
             for symmetric and asymmetric fronts $f_p$
             and different population sizes $\mu$ (cf.\ \defref{OPT}).
             The $x$-axis is scaled logarithmically.
             We omit the values of the $y$-axis as we are
             only interested in the relative comparison 
             (\protect\includegraphics[width=2mm,clip]{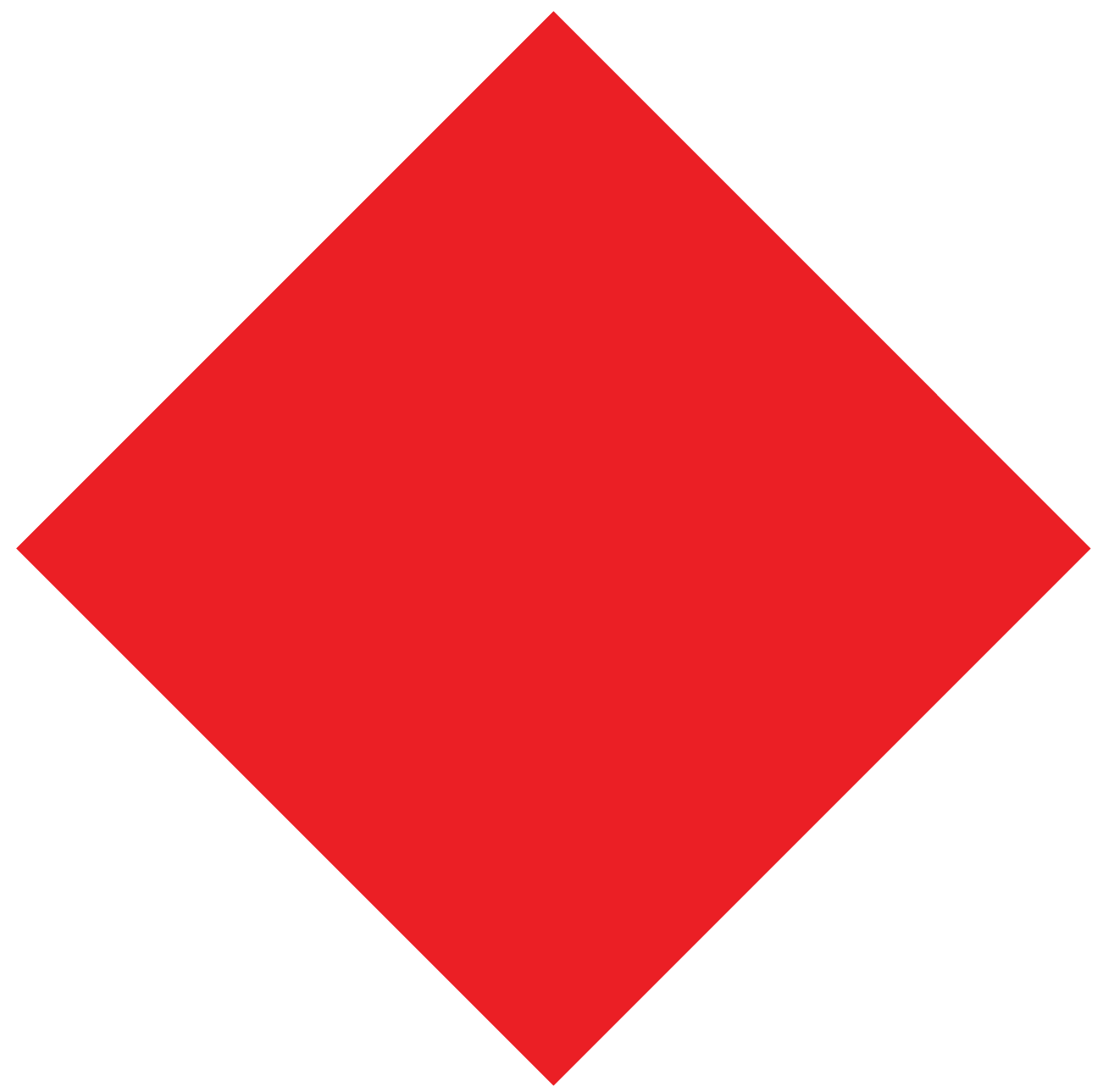}
             vs.~\protect\includegraphics[width=2mm,clip]{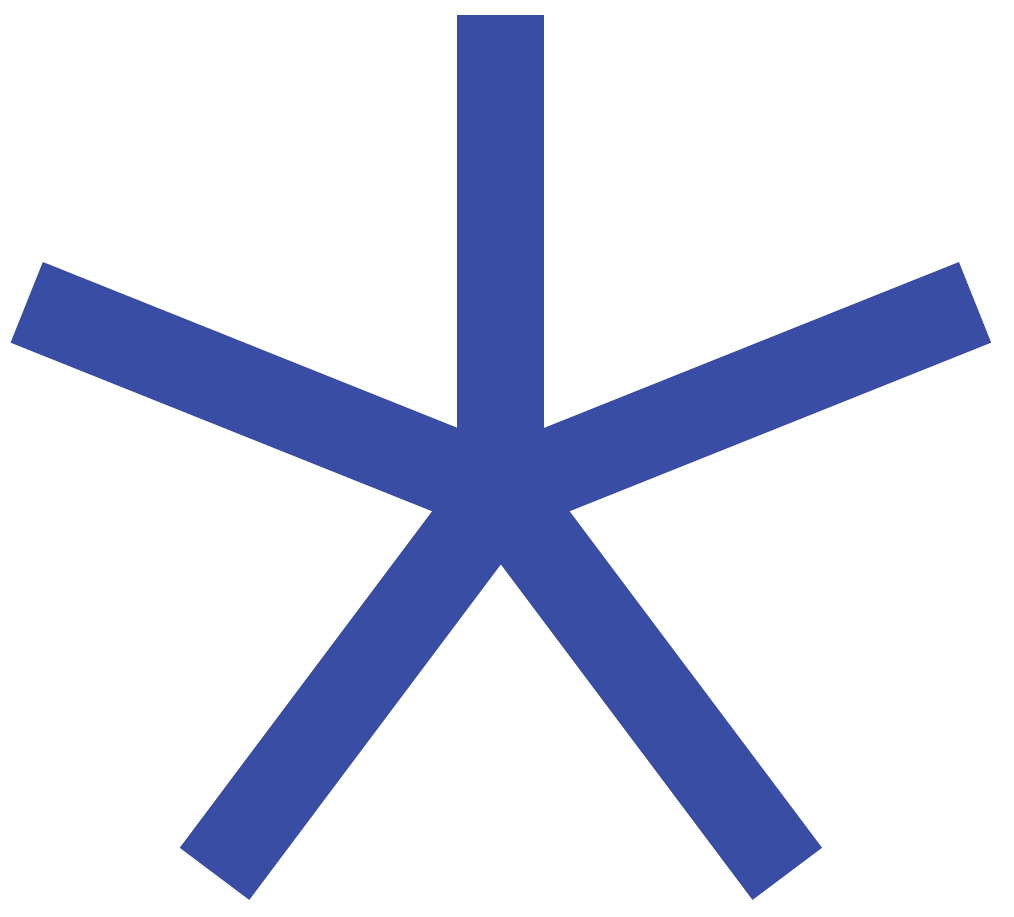})
             for each front $f_p$ and population size $\mu$.
             Note that (a) shows that the approximation ratio of the optimal hypervolume distribution
             $\APP{\OPTHYP{3}{\fsym_{p}}}$ and
             the optimal approximation distribution
             $\APP{\OPTAPP{3}{\fsym_{p}}}$ 
             are equivalent for all examined $p$.
             }
    \label{fig:p}
\end{figure*}

We have already seen that scaling the function has a high impact on the optimal approximation distribution but not on the optimal hypervolume distribution. We want to investigate this effect in greater detail. 
The influence of scaling the parameter $x_{\mu} \geq 2$ of different functions $f_{p} \colon [1,x_{\mu}] \to [1,2]$ is depicted in \figref{x3} for $p=1/3,1/2,1,2,3$. For fixed $\mu=3$ it shows the achieved approximation
ratio. As expected, the larger the asymmetry ($x_\mu$) the larger the approximation
ratios. For concave fronts ($p>1$) the approximation ratios seem to converge
quickly for large enough $x_\mu$. The approximation of $f_2$ tends towards
the golden ratio $\sqrt{5}-1\approx1.236$ for the optimal approximation
and $4/3\approx1.333$ for the optimal hypervolume. For $f_3$ they tend towards
1.164 and 1.253, respectively. Hence, for $f_2$ and $f_3$ the hypervolume
is never more than 8\% worse than the optimal approximation.
This is different for the convex fronts ($p<1$). There, the ratio between the
hypervolume and the optimal approximation appears divergent.

Another important question is how the choice of the population size influences the relation between an optimal approximation and the approximation achieved by an optimal hypervolume distribution.
We investigate the influence of the choice of $\mu$ on the approximation behavior in greater detail.
\figref{mu} shows the achieved approximation ratios depending on the 
number of points $\mu$. 
For symmetric $f_p$'s with $(x_1,y_1)=(y_\mu,x_\mu)$ and $\mu=3$
the hypervolume achieves an optimal approximation distribution for all $p>0$. The same
holds for the linear function $f_1$ independent of the scaling
implied by $(x_1,y_1)$ and $(y_\mu,x_\mu)$.

For larger populations, the approximation ratio of the hypervolume distribution and the optimal distribution decreases.
However, the performance of the hypervolume measure is especially poor even for larger $\mu$
for convex asymmetric fronts, that is, $\fasy_p$ with $p<1$
(e.g.\ \figrefs{mu:asy:13}{mu:asy:12}). Our investigations show that the 
approximation of an optimal hypervolume distribution may differ significantly 
from an optimal one depending on the choice of $p$. 
An important issue is whether the front is convex or concave~\cite{LizarragaLizarraga08}.
The hypervolume was thought to prefer convex regions to concave regions~\cite{ZitzlerThiele98}
while \cite{Augetal2009} showed that the density of points only depends on the slope of the front and
not on convexity or concavity.
To illuminate the impact of convex vs.\ concave further,
\figref{p} shows the approximation ratios depending on $p$.
As expected, for $p=1$ the hypervolume calculates the optimal
approximation. However, the influence of the $p$ is very different
for the symmetric and the asymmetric test function.
For $\fsym_p$ the convex ($p<1$) fronts 
are much better approximated
by the hypervolume than the concave ($p>1$) fronts (cf.~\figref{p}(a)--(d)).
For $\fasy_p$ this is surprisingly the other way around (cf.~\figref{p}(e)--(h)).


\section{Influence of the reference point}\label{sec:depref}

In all previous investigations, we have not considered the impact of the reference point. To allow a fair comparison we assumed that the optimal approximation distribution and the optimal hypervolume distribution have to include both extreme points. This is clearly not optimal when considering the optimal approximation distribution. Therefore, we relax our assumption and allow any set consisting of $\mu$ points and raise the question how the optimal approximation distribution looks in this case. 
Considering the hypervolume indicator, the question arises whether this optimal approximation distribution can be achieved by choosing a certain reference point. Therefore, the goal of this section is to examine the impact of the reference point for determining optimal approximation distributions.

For this we have to redefine parts of the notation.  We mark all variables
with a hat (like \ $\widehat\ $\ )  
to make clear that we do not require the extreme points to be
included anymore.

Consider again a Pareto front $f$.  We now let $\cX_r(\mu,f)$ be the set of all subsets of 
\[
	\left\{(x,f(x)) \mid x \in [x_{\min},x_{\max}]\right\}
\]
of cardinality $\mu$, where we do \emph{not} assume
that $(x_{\min},f(x_{\min}))$ and $(x_{\max},f(x_{\max}))$ have to be necessarily contained.
We also have to redefine the notion of optimal hypervolume distributions
and optimal approximation distribution.

\begin{definition}\label{def:OPTr}
    The optimal hypervolume distribution
    \begin{equation*}
    	\OPTHYPr{\mu,r}{f} := \argmax_{X\in\cX_r(\mu,f)}{\HYP_r(X)}
    \end{equation*}
    consists of $\mu$ points that maximize the hypervolume with respect to $f$.
    The optimal approximation distribution
    \begin{equation*}
    	\OPTAPPr{\mu}{f} := \argmin_{X\in\cX_r(\mu,f)}{\APP{X}}
    \end{equation*}
    consists of $\mu$ points that minimize the approximation ratio with respect to $f$.
\end{definition}

\subsection{Optimal approximations}

Similar to Lemma~\ref{lem:optapprox}, the following lemma states conditions for an optimal approximation distribution which does not have to contain the extreme points.

\begin{lemma}
\label{lem:optapprox2}
Let $f \colon [x_{\min}, x_{\max}] \to \R$ be a Pareto front and $X = \{x_{1}, \dots, x_{\mu}\}$ a solution set with $x_{i} < x_{i + 1}$ for all $1 \leq i < \mu$.
If there is a ratio $\ratio > 1$ and a set $Z = \{z_{1}, \dots, z_{\mu - 1}\}$ with $x_{i} < z_{i} < x_{i + 1}$ for all $1 \leq i < \mu$ such that
\begin{itemize}
\setlength{\itemsep}{0pt}
\setlength{\parskip}{0pt}
\item $z_{i} = \ratio \cdot x_{i}$ for all $1 \leq i \leq \mu$ (where $z_{\mu} = x_{\max}$) and
\item $f(z_{i}) = \ratio \cdot f(x_{i + 1})$ for all $0 \leq i < \mu$ (where $z_{0} = x_{\min}$)
\end{itemize}
then $X = \OPTAPPr{\mu}{f}$ is the optimal approximation distribution with approximation ratio $\ratio$.
\end{lemma}

\begin{proof}
For each $i$, $1 \leq i \leq \mu-1$, $z_i$ is the worst approximated point in the interval $[x_i, x_{i+1}]$. Furthermore, $z_0=x_{\min}$ is the worst approximated point in the interval $[x_{\min}, x_1]$ and $z_{\mu}=x_{\max}$ is the worst approximated point in the interval $[x_{\mu}, x_{\max}]$.
This implies that the approximation ratio of $X$ is 
$$\ratio = \max \left\{\frac{f(z_0)}{f(x_1)}, \frac{z_{\mu}}{x_{\mu}},  \frac{z_i}{x_i}, \frac{f(z_i)}{f(x_{i+1})} \mid 1\leq i \leq \mu-1 \right\} .$$ 

Assume there is a different solution set $X' = \{x'_{1}, \dots, x'_{\mu}\}$ with $x'_{i} < x'_{i + 1}$ for all $1 \leq i < \mu$ and approximation ratio at most $\ratio$.

Since $X' \neq X$ there is an index $i$ with $x'_{i} \neq x_{i}$.
Consider the smallest such index.
We distinguish the two cases $x'_{i} < x_{i}$ and $x'_{i} > x_{i}$.

Assume $x'_{i} < x_{i}$.
Consider the point $z_i' = \ratio \cdot x'_{i}$.
Since $z_i'= \ratio \cdot x'_{i} < \ratio \cdot x_{i} = z_{i}$, we derive $x'_{i + 1} < x_{i + 1}$ as otherwise $\frac{f(z'_i)}{f(x_{i+1}')} \geq \frac{f(z'_i)}{f(x_{i+1})} > \ratio$  would contradict our assumption that $X'$ achieves an approximation ratio of at most $\ratio$.
Repeating the argument $(\mu - i)$-times leads to $x'_{\mu} < x_{\mu}$, which gives $\ratio \cdot x'_{\mu} < \ratio \cdot x_{\mu} = x_{\max}$.
This implies that the approximation of $x_{\max}$ by $X'$ is $\frac{ x_{\max}}{x'_{\mu}} > \ratio$ which contradicts the assumption that $X'$ achieves an approximation ratio of at most $\delta$.

Assume $x'_{i} > x_{i}$.
Then all points within $(z_{i - 1}, f^{-1}(\delta \cdot f(x'_{i})))$ are not $\ratio$-approximated.
The interval is not empty since $f^{-1}(\ratio \cdot f(x'_{i})) > x'_{i} > x_{i} > z_{i - 1}$ due to $\ratio > 1$ and $f$ strictly monotonically decreasing.
We have another contradiction.

Altogether, we get that $X = \OPTAPPr{\mu}{f}$  is the unique set achieving an approximation ratio of at most $\ratio$ and therefore an optimal approximation distribution.
\end{proof}

The previous lemma can be used to compute the overall optimal approximation distribution of $\mu$ for a given function describing the Pareto front. In the following, we will use this to compare it to the optimal hypervolume distribution depending on the chosen reference point. Again we consider the class of linear fronts and the class of convex fronts given in Section~\ref{sec3}.


\newcommand{\renumberlinr}[1]{
\hspace*{-3mm}
\begin{tikzpicture}
	\node (full) at (0,0) [inner sep=0pt,above right]
	            {\includegraphics[width=\subfigwidth,viewport=32mm 79mm 192mm 208mm,width=\doubledistfigwidth,clip]{pdfs/#1.pdf}}; 
	\foreach \x in {0,0.5,...,2}{
	\pgfmathparse{0.035+2.7*(\x)}
	\node at (\pgfmathresult,.1) [below] {\footnotesize\pgfmathprintnumber{\x}};
	}
	\foreach \y in {0,0.5,...,2}{
	\pgfmathparse{0.05+2.13*(\y)}
	\node at (.1,\pgfmathresult) [left] {\footnotesize\pgfmathprintnumber{\y}};
	}
\end{tikzpicture}
\hspace*{-3mm}
}

\newcommand{\renumberlinzoom}[1]{
\hspace*{-3mm}
\begin{tikzpicture}
	\node (full) at (0,0) [inner sep=0pt,above right]
	            {\includegraphics[width=\subfigwidth,viewport=32mm 79mm 192mm 208mm,width=\doubledistfigwidth,clip]{pdfs/#1.pdf}}; 
	\foreach \x in {0.92,0.94,...,1.01}{
	\pgfmathparse{0.51+59.98*(\x-0.92)}
	\node at (\pgfmathresult,.1) [below] {\footnotesize\pgfmathprintnumber{\x}};
	}
	\foreach \y in {0.92,0.94,...,1.01}{
	\pgfmathparse{0.425+47.3*(\y-0.92)}
	\node at (.1,\pgfmathresult) [left] {\footnotesize\pgfmathprintnumber{\y}};
	}
\end{tikzpicture}
\hspace*{-3mm}
}

\subsection{Analytic results for linear fronts}

We first consider linear fronts.  The optimal multiplicative approximation factor
can be easily determined with \lemref{optapprox2} as shown in the following theorem.

\begin{theorem}\label{thm:optlinr}
	Let $f \colon [1,(1-d)/c] \to [1,c+d]$ 
	be a linear function $f(x) = c \cdot x + d$ where $c < 0$ and $d > 1 - c$ are arbitrary constants.
	Then
	\begin{equation*}
		\OPTAPPr{\mu}{f} = \{x_1, \ldots, x_{\mu} \},
	\end{equation*}
	where $x_i = 
	\frac{d\,(\mu c - i\,(c+d-1))}{c\,(c+(\mu+1)\,d-1)}$,
	$1 \leq i \leq \mu$, and
	\[
	\APP{\OPTAPPr{\mu}{f}} = \frac{c+(\mu+1)\,d-1}{\mu d}.
	\]
\end{theorem}
\begin{proof}
Using Lemma~\ref{lem:optapprox2},
we get the following system of $2\,(\mu+1)$ linear equations
\begin{align*}
z_0 &= 1,\\
z_i &= \delta \, x_i \text{\qquad\qquad for $i = 1 \dots \mu$},\\
z_\mu &= (1-d)/c,\\
c \,y_i +d &= \delta \,( c\, x_{i+1} +d) \text{\quad for $i = 0 \ldots \mu-1$}.
\end{align*}
The unique solution of this system of linear equations is
\begin{align*}
\delta &= \frac{c+(\mu+1)\,d-1}{\mu d},\\
x_i &= \frac{d\,(\mu \, c - i\,(c+d-1))}{c\,(c+(\mu+1)\,d-1)}  \text{\qquad for $i = 1 \dots \mu$},\\
z_i &= 1-i \, \frac{c+d-1}{\mu \, c} \text{\qquad\qquad for $i = 0 \dots \mu$},
\end{align*}
which proves the claim.
\end{proof}

It remains to analyze the approximation factor achieved by an optimal hypervolume distribution. 
The impact of the reference point for the class of linear functions has been 
investigated by Brockhoff in~\cite{Brockhoff2010}. Using his results, we can conclude the 
following theorem.

\begin{theorem}\label{thm:linr}
	Let $f \colon [1,(1-d)/c] \to [1,c+d]$ 
	be a linear function $f(x) = c \cdot x + d$ where $c < 0$ and $d > 1 - c$ are arbitrary constants.
	Let $\mu\geq2$ and
	\begin{align*}
	M_1 &:= \min\left\{
		c+d-r_1,\frac{\mu}{\mu-1}\cdot(c+d-1),\right.\\
		&
		\hspace*{1cm}
		\left.(d-1)+\frac{d+(\mu+1)\,c}{\mu}+ \frac{r_2+d-1}{\mu c}+ \frac{d-1}{\mu c^2}
		\right\}, \\
	M_2 &:= \min\left\{
            	\frac{1-d}{c}-r_2,
		\frac{\mu}{\mu-1} \cdot \frac{1-c-d}{c},\frac{1-c-d}{c}+\frac{r_1-c-d}{\mu c}
                        \right\}.
	\end{align*}
	Then the optimal hypervolume distribution with respect to the reference point $r$ is
	\[
		\OPTHYPr{\mu,r}{f} = \{x_1, \ldots, x_{\mu} \}
	\]
	where
	\[
	 x_i =  \frac{M_1-d+1}{c}  + i\cdot \frac{d-M_1+(M_2+1)\,c-1}{c\,(\mu+1)} .
	\]
\end{theorem}

\thmref{lin} follows immediately from Theorem~3 of Brockhoff~\cite{Brockhoff2010}
by translating their minimization setting into our maximization setting.
Knowing the set of points which maximize the hypervolume, we can now determine
the achieved approximation depending on the chosen reference point.

\begin{theorem}\label{thm:lina}
	Let $f \colon [1,(1-d)/c] \to [1,c+d]$ 
	be a linear function $f(x) = c \cdot x + d$ where $c < 0$ and $d > 1 - c$
	are arbitrary constants.
	Let $\mu\geq2$ and $M_1$ and $M_2$ defined as in \thmref{linr}, then
	\[
		\APP{\OPTHYPr{\mu,r}{f}} = \max\{A_\ell, A_c, A_r \}
	\]
	where
	\begin{align*}
	 A_\ell :=  \frac{(c+d)\,(\mu+1)}{\mu+c+d+M_1\mu+M_2 c},\\
	 A_c := \frac{d\,(\mu+1)}{d \mu +c+2 d -1-M_1+M_2 c},\\
	 A_r := \frac{(1-d)\,(\mu+1)}{c\mu -d+1+M_1+ M_2 c \mu}.
	\end{align*}
\end{theorem}
\begin{proof}
	We want to determine the approximation ratio of the optimal hypervolume
	distribution $\OPTHYPr{\mu,r}{f} = \{x_1, \ldots, x_{\mu} \}$ as defined
	in \thmref{linr}. For this, we distinguish between three cases.
	The approximation ratio of the inner points $\tilde{x}$ with
	$x_1\leq\tilde{x}\leq x_\mu$ can be determined
	as in the proof of \thmref{lin}.  It suffices to plug
	the definition of $x_i$ and $x_{i+1}$ from \thmref{linr}
	into equation~\ref{eqn:approx}. Let $\tilde{x}$ be the solution of this linear
	equation.  Then the inner approximation factor is \[
	A_c=\frac{\tilde{x}}{x_i}
	= \frac{d\,(\mu+1)}{d \mu +c+2 d -1-M_1+M_2 c},
	\]
	which is independent of $i$.
	
	It remains to determine the outer approximation factors.  The approximation
	factor of the points $\tilde{x}$ with $1\leq\tilde{x}\leq x_1$ is maximized
	for $\tilde{x}=1$.
	The left approximation factor is therefore
	\[
	A_\ell=\frac{c+d}{f(x_1)}
	= \frac{(c+d)\,(\mu+1)}{\mu+c+d+M_1\mu+M_2 c}.
	\]
	The approximation factor 
	 of the points $\tilde{x}$ with $x_{\mu}\leq\tilde{x}\leq (1-d)/c$ is maximized
	 for $\tilde{x}=x_\mu$.
	 The right approximation factor is therefore 
	\[
	A_r=
	\frac{1-d}{c x_\mu}
	=\frac{(1-d)\,(\mu+1)}{c\mu -d+1+M_1+ M_2 c \mu}.\]
	The overall approximation factor is then the largest approximation factor
	of the three parts, that is, 
	$\max\{A_\ell, A_c, A_r \}$.
\end{proof}

\newcommand{\renumberconr}[1]{
\hspace*{-3mm}
\begin{tikzpicture}
	\node (full) at (0,0) [inner sep=0pt,above right]
	            {\includegraphics[width=\subfigwidth,viewport=32mm 79mm 192mm 208mm,width=\doubledistfigwidth,clip]{pdfs/#1.pdf}}; 
	\foreach \x in {0,0.5,...,2}{
	\pgfmathparse{0.035+2.7*(\x)}
	\node at (\pgfmathresult,.1) [below] {\footnotesize\pgfmathprintnumber{\x}};
	}
	\foreach \y in {0,0.5,...,2}{
	\pgfmathparse{0.05+2.13*(\y)}
	\node at (.1,\pgfmathresult) [left] {\footnotesize\pgfmathprintnumber{\y}};
	}
\end{tikzpicture}
\hspace*{-3mm}
}

\newcommand{\renumberconzoom}[1]{
\hspace*{-3mm}
\begin{tikzpicture}
	\node (full) at (0,0) [inner sep=0pt,above right]
	            {\includegraphics[width=\subfigwidth,viewport=32mm 79mm 192mm 208mm,width=\doubledistfigwidth,clip]{pdfs/#1.pdf}}; 
	\foreach \x in {0.94,0.96,...,1.01}{
	\pgfmathparse{1.238+67.4*(\x-0.94)}
	\node at (\pgfmathresult,.1) [below] {\footnotesize\pgfmathprintnumber{\x}};
	}
	\foreach \y in {0.94,0.96,...,1.01}{
	\pgfmathparse{0.99+53.1*(\y-0.94)}
	\node at (.1,\pgfmathresult) [left] {\footnotesize\pgfmathprintnumber{\y}};
	}
\end{tikzpicture}
\hspace*{-3mm}
}

\subsection{Analytic results for a class of convex fronts}

\begin{figure}[tb]
    \centering
        \includegraphics[viewport=21mm 73mm 192mm 208mm,width=.4\textwidth,clip]{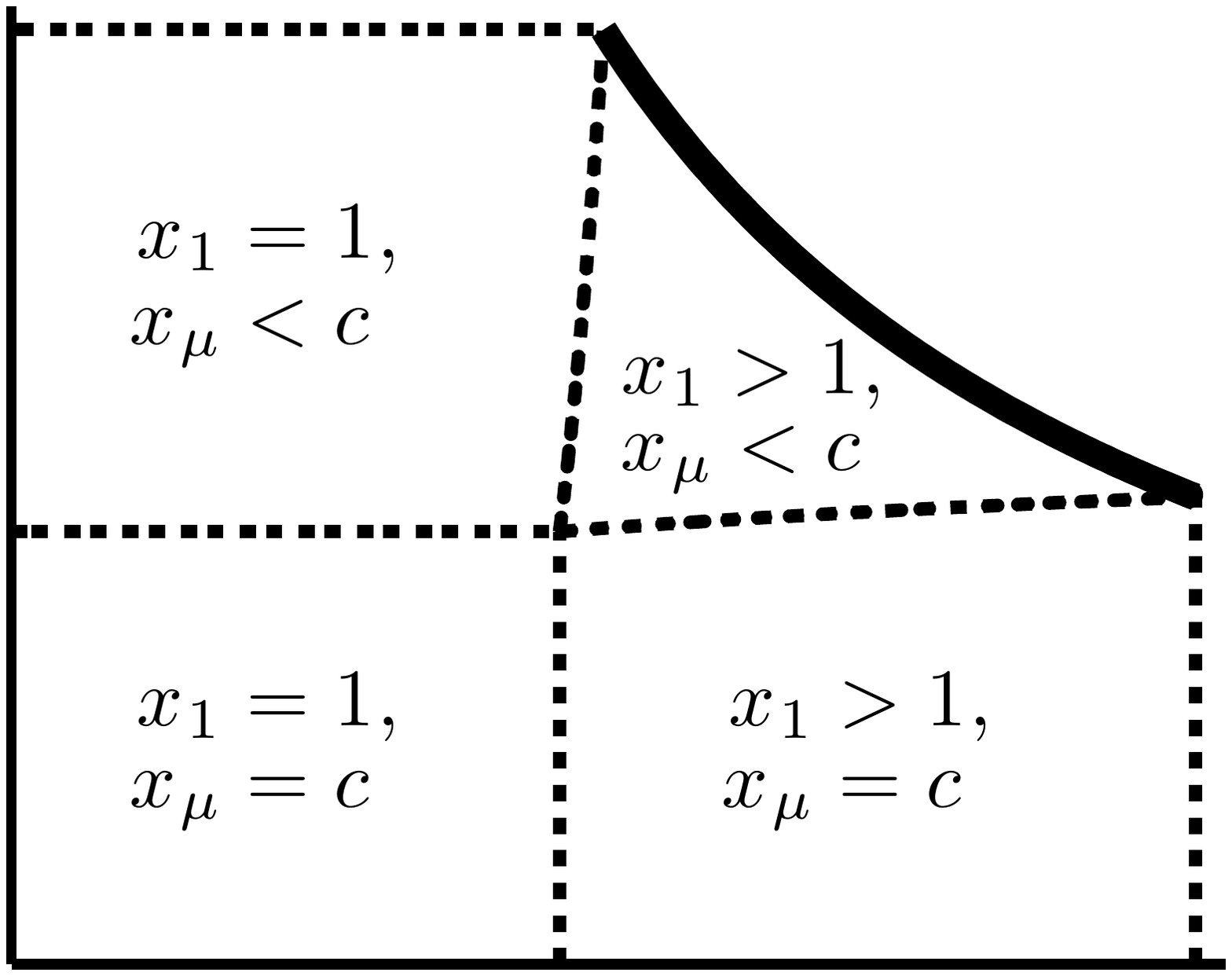}
    \hspace*{.05\textwidth}
    \caption{Position of the extremal points of optimal hypervolume distributions depending on the reference point for convex functions
    $f \colon [1,c] \to [1,c]$ with $f(x) = c/x$ and $c > 1$.
       Depending on the position of the reference point, the leftmost point of a optimal
       hypervolume distribution $x_1$ is either at the border ($x_1=x_{\min}=1$) or
       inside the domain ($x_1>x_{\min}=1$).  Similarly, the rightmost point $x_{\mu}$
       is either at the border ($x_{\mu}=x_{\max}=c$) or inside the domain ($x_{\mu}<x_{\max}=c$).
       (Note that the figure looks very similar for linear functions.)
       }
    \label{fig:hyprefcon}
\end{figure}

We now consider convex fronts and investigate the overall optimal multiplicative approximation first which does not have to include the extreme points.
The following theorem shows how such an optimal approximations looks like and will serve later for the comparison to an optimal hypervolume distribution in dependence of the chosen reference point.

\begin{theorem}\label{thm:optcon}
	Let $f \colon [1,c] \to [1,c]$ be a convex front with $f(x) = c/x$ where $c > 1$ is an arbitrary constant.
	Then
	\begin{equation*}
		\OPTAPPr{\mu}{f} = \{x_1, \ldots, x_{\mu} \},
	\end{equation*}
	where $x_i = c^{\frac{2i-1}{2 \mu}}$, $1 \leq i \leq \mu$, and  $\APP{\OPTAPPr{\mu}{f}} = c^{\frac{1}{2\mu}}$.
\end{theorem}

\begin{proof}
Using Lemma~\ref{lem:optapprox2},
we have $z_0 = x_{\min} = 1$ and $z_{\mu}=x_{\max}=c$. Furthermore, 
\begin{eqnarray*}
& & f(z_0) = \delta f(x_1)\\
& \Leftrightarrow & c = \delta \cdot \frac{c}{x_1}\\
& \Rightarrow & x_1 = \delta.
\end{eqnarray*}
 We have $z_i = \delta x_i$, $1 \leq i \leq \mu$, and 
 \begin{eqnarray*}
 & & f(z_i) = \delta \cdot f(x_{i+1})\\
 & \Leftrightarrow & \frac{c}{z_i} = \delta \frac{c}{x_{i+1}}\\
 & = & \frac{c}{\delta x_i} = \delta \frac{c}{x_{i+1}}\\
 & \Rightarrow & \frac{x_{i+1}}{x_i} = \delta^2
 \end{eqnarray*}
for $1\leq i < \mu$. This implies $x_i = \delta^{2i-1}$, $1 \leq i \leq \mu$. 
Furthermore, 
\begin{eqnarray*}
& & y_{\mu} = \delta \ x_{\mu}\\
& \Leftrightarrow & c = \delta^{2\mu}\\
& \Rightarrow & \delta = c^{\frac{1}{2\mu}}
 \end{eqnarray*}
This implies  $x_i = c^{\frac{2i-1}{2 \mu}}$, $1 \leq i \leq \mu$ and  $\APP{\OPTAPPr{\mu}{f}} = c^{\frac{1}{2\mu}}$ which completes the proof.
\end{proof}

Now, we consider the optimal hypervolume distribution depending on the choice of the reference point and compare it to the optimal multiplicative approximation.

\begin{theorem}\label{thm:hyprefcon}
	Let $f \colon [1,c] \to [1,c]$ be a convex front with $f(x) = c/x$ where $c > 1$ is an arbitrary constant.
	Then
	\begin{equation*}
		\OPTHYPr{\mu,r}{f} = \{x_1, \ldots, x_{\mu} \},
	\end{equation*}
	where $x_i$, $1 \leq i \leq \mu$, depend on the choice of the reference point $r$
	as follows:
	\begin{enumerate}
	\item
	If $r_1\leq c^{-1/(\mu-1)}$ and $r_2\leq c^{-1/(\mu-1)}$,
	then $x_1=1$, $x_\mu=c$, 
	$
		x_i  = c^{(i-1)/(\mu-1)}
	$
	for $1 \leq i \leq \mu$,
	and  \[
	\APP{\OPTHYPr{\mu,r}{f}} = c^{\frac{1}{2\mu-2}}.
	\]
	
	\item 
	If $r_1\leq c \,r_2^\mu$ and $r_2\leq c\, r_1^\mu$,
	then $x_1>1$, $x_\mu<c$,
	$
	x_i = ( c^i \, r_1^{\mu-i+1} / r_2^i )^{1/(\mu+1)}
	$
	for $1 \leq i \leq \mu$,
	and 
	\begin{align*}
	&\APP{\OPTHYPr{\mu,r}{f}} =
\max \left\{
\left( \frac{c \cdot r_1^{\mu}}{r_2} \right)^{\frac{1}{\mu+1}}, 
\left(\frac{c}{r_1 r_2} \right)^{\frac{1}{2(\mu+1)}},
c \, \left( \frac{r_2^\mu}{c^\mu \cdot r_1} \right)^{\frac{1}{\mu+1}}
\right\}.
	\end{align*}
	
	\item
	If $r_2\geq c^{-1/(\mu-1)}$, $r_2\leq c$, and $r_2\geq c r_1^\mu$,
	then
	$x_1=1$,
	$
	x_i = (c/ r_2)^{(i-1)/\mu}
	$
	for $1 \leq i \leq \mu$,
	and 
	\[
	\APP{\OPTHYPr{\mu,r}{f}} = \max \left\{
\left(\frac{c}{r_2}\right)^{\frac{1}{2\mu}},
c\left(\frac{c}{r_2}\right)^{\frac{\mu}{\mu-1}}
\right\}.
	\]
	
	\item
	If $r_1\geq c^{-1/(\mu-1)}$, $r_1\leq c$ and $r_1\geq c \, r_2^\mu$,
	then
	$x_\mu=c$,
	$
	x_i = r_1 ( c/r_1)^{i/\mu}
	$
	for $1 \leq i \leq \mu$,
	and 
	\[
	\APP{\OPTHYPr{\mu,r}{f}} = 
	\max \left\{
\left(\frac{c}{r_1}\right)^{\frac{1}{2\mu}},
\left(c r_1^{\mu-1}\right)^{\frac{1}{\mu}}
\right\}.
\]	

	\end{enumerate}
\end{theorem}

\begin{proof}
In order to proof \thmref{hyprefcon}, we distinguish four cases, namely
whether $x_1=1$ or $x_1>1$ and
whether $x_\mu=c$ or $x_\mu<c$.
\figref{hyprefcon} gives an illustration of the four cases.

\emph{
The first case $x_1=1$ and $x_\mu=c$} corresponds to the previous situation where 
we required that both extreme points are included.  The statement
of \thmref{hyprefcon} for this case follows immediately from 
equations~\ref{eqn:hypcon} and~\ref{eqn:conapp}
in \secref{confixed}. 

\emph{The second case $x_1>1$ and $x_\mu<c$} is more involved.
First note that we consider only points that have a positive contribution with respect to the given reference point. Therefore, we assume that $r_1 < x_1$ and $r_2< f(x_{\mu})$ holds.

The hypervolume of a set of points $X=\{x_1,\dots,x_{\mu}\}$, where \wlo\ $x_1 \leq x_2 \leq \dots \leq x_{\mu}$, with respect to a reference point $r=(r_1, r_2)$ with $r_1 < x_1$ and $r_2< f(x_{\mu})$ is then given by
\begin{align*}
	\HYP(X,r) = {} & (x_1 -r_1) \cdot (f(x_1)-r_2)\\
	& + (x_2 -x_1) \cdot (f(x_2)-r_2)\\
	& \dots\\
	& +  (x_{\mu} - x_{\mu-1}) \cdot (f(x_{\mu})-r_2)\\
	 = {} & c \cdot \mu + r_1r_2 \\
	& \!-c (r_1/x_1 + x_1/x_2 + x_2/x_3 + \dots + x_{\mu-1} / x_{\mu})\\
	& \!-x_{\mu} \cdot r_2.
\end{align*}

\noindent
In order to maximize the hypervolume, we 
consider the function
\begin{align*}
h(x_1, \ldots, x_{\mu})
= c\, \left(\frac{r_1}{x_1} + \frac{x_1}{x_2} + \frac{x_2}{x_3} + \dots + \frac{x_{\mu-1}}{x_{\mu}}\right) + x_{\mu} \cdot r_2
\end{align*}
and compute its partial derivatives.

We have $1<x_1 < x_2 < \ldots < x_\mu<c$ as the equality of two points implies that one of them can be exchanged for another and thereby increases the hypervolume. 
We work under these assumptions and aim to find a set of points $X$ that minimizes the function $h$.
To do this, we consider the gradient vector given by the partial derivatives

\begin{align*}
	h'(x_1,\dots,x_{\mu})
	= \left(\frac{c}{x_2}-\frac{cr_1}{x_1^2}, \frac{c}{x_3}-\frac{cx_1}{x_2^2},\dots, \frac{c}{x_{\mu}}-\frac{cx_{\mu-2}}{x_{\mu-1}^2}, r_2 - \frac{cx_{\mu-1}}{x_{\mu}^2}\right).
\end{align*}
This implies that $h$ can be minimized by setting
\begin{equation*}
	\begin{array}{lclcl}
		x_2 & = & x_1^2/r_1 &  & \\
		x_3 & = & x_2^2/x_1 & = &  x_1^3/r_1^2\\
		x_4 & = & x_3^2/x_2 & = &  x_1^4/r_1^3\\
		 & \vdots &  \\
		x_\mu & = & x_{\mu-1}^2/x_{\mu-2} & = &  x_1^{\mu}/r_1^{\mu-1}\\
		x_{\mu}^2 & = & cx_{\mu-1}/r_2.
	\end{array}
\end{equation*}
Hence with 
\[
\frac{cx_{\mu-1}}{r_2} = \frac{cx_1^{\mu-1}}{r_1^{\mu-2}r_2} = \frac{x_1^{2\mu}}{r_1^{2(\mu-1)}}
\]
we get
\[
x_1 = \min \left\{\max \left\{ \left(\frac{cr_1^{\mu}}{r_2}\right)^{\frac{1}{\mu+1}},1 \right\},c \right\}.
\]
As we can assume $x_1\geq1$ and $x_{\mu} \leq c$,
we get for
$r_2\leq c r_1^{\mu}$ and $r_1 \leq c r_2^{\mu}$ that
\[
x_i = \left( \frac{c^i \cdot r_1^{\mu-i+1}}{r_2^i} \right)^\frac{1}{\mu+1}
\]
for $1 \leq i \leq \mu$.
It now remains to determine the achieved approximation factor.
For this, we proceed as in \thmref{con} and use that
\[
\frac{x}{x_i} = \frac{f(x)}{f(x_{i+1})}\\
\Longrightarrow x = \sqrt{x_i \cdot x_{i+1}}.
\]

\noindent
This gives an approximation factor of the inner points of
\[
\frac{x_{i+1}}{x} = \left(\frac{c}{r_1 r_2} \right)^{\frac{1}{2(\mu+1)}}.
\]

\noindent
For the upper end points the approximation is 
\[
\frac{c}{f(x_1)}
= x_1
= \left( \frac{c \cdot r_1^{\mu}}{r_2} \right)^{\frac{1}{\mu+1}}.
\]

\noindent
For the lower end points the approximation is 
\[
\frac{c}{x_{\mu}}
=
c \cdot \left( \frac{r_2^\mu}{c^\mu \cdot r_1} \right)^{\frac{1}{\mu+1}}.
\]

\noindent
Hence the overall approximation factor in the second case is
\[
\max \left\{
\left( \frac{c \cdot r_1^{\mu}}{r_2} \right)^{\frac{1}{\mu+1}}, 
\left(\frac{c}{r_1 r_2} \right)^{\frac{1}{2(\mu+1)}},
c \cdot \left( \frac{r_2^\mu}{c^\mu \cdot r_1} \right)^{\frac{1}{\mu+1}}
\right\}.
\]

\emph{The third case $x_1=1$ and $x_\mu<c$} fixes only the left end of the front.
Here, we consider the function
\begin{align*}
&h(x_2, \ldots, x_{\mu}) = c \, \left (\frac{1}{x_2} + \frac{x_2}{x_3} + \frac{x_3}{x_4} + \dots + \frac{x_{\mu-1}}{x_{\mu}}\right) + x_{\mu} \cdot r_2.
\end{align*}
Not that in contrast to the second case, $h(\cdot)$ does not depend on $r_1$.
We can assume without loss of generality that $1=x_1 < x_2 < \ldots < x_\mu<c$.
The partial derivatives are therefore
\begin{align*}
	&h'(x_2,\dots,x_{\mu})= \left(\frac{c}{x_3}-\frac{c}{x_2^2},
	\frac{c}{x_4}-\frac{c x_2}{x_3^2},
	\dots, \frac{c}{x_{\mu}}-\frac{cx_{\mu-2}}{x_{\mu-1}^2}, r_2 - \frac{cx_{\mu-1}}{x_{\mu}^2}\right).
\end{align*}
This implies that $h$ can be minimized by setting
\begin{equation*}
	\begin{array}{lclcl}
		x_3 & = & x_2^2 \\
		x_4 & = & x_3^2/x_2 & = &  x_2^3\\
		x_5 & = & x_4^2/x_3 & = &  x_2^4\\
		 & \vdots &  \\
		x_\mu & = & x_{\mu-1}^2/x_{\mu-2} & = &  x_2^{\mu-1}\\
		x_{\mu}^2 & = & \frac{cx_{\mu-1}}{r_2}.
	\end{array}
\end{equation*}
Starting with $x_\mu^2=x_\mu^2$ we get,
\[
\frac{cx_{\mu-1}}{r_2}
= \frac{cx_2^{\mu-2}}{r_2} = x_2^{2(\mu-1)}
\]
and
\[
x_2 = \min \left\{\max \left\{ \left(\frac{c}{r_2}\right)^{\frac{1}{\mu}},1 \right\},c \right\}.
\]
Again using that $x_2\geq1$ and $x_{\mu} \leq c$ and assuming that
$r_2\leq c$ and $r_2\geq c^{-\frac{1}{\mu-1}}$, we get
\[
x_i = \left(\frac{c}{r_2}\right)^{\frac{i-1}{\mu}}.
\]

\noindent
This results in an approximation factor for the inner points of
\[
\frac{x_{i+1}}{\sqrt{x_i \cdot x_{i+1}}} = \left(\frac{c}{r_2}\right)^{\frac{1}{2\mu}}.
\]

\noindent
For the upper end points the approximation is 
\[
\frac{c}{f(x_1)}
= x_1
= 1.
\]

\noindent
For the lower end points the approximation is 
\[
\frac{c}{x_{\mu}}
=
c\left(\frac{r_2}{c}\right)^{\frac{\mu-1}{\mu}}.
\]

\noindent
Hence the overall approximation factor in the third case is
\[
\max \left\{
\left(\frac{c}{r_2}\right)^{\frac{1}{2\mu}},
c\left(\frac{c}{r_2}\right)^{\frac{\mu}{\mu-1}}
\right\} .
\]

\emph{The fourth case $x_1>1$ and $x_\mu=c$} fixes the right end of the front.
We
consider the function
\begin{align*}
&h(x_1, \ldots, x_{\mu-1})
= c \left(\frac{r_1}{x_1} + \frac{x_1}{x_2} + \frac{x_2}{x_3} + \dots + \frac{x_{\mu-2}}{x_{\mu-1}} + \frac{x_{\mu-1}}{ c}\right)
\end{align*}
and compute its partial derivatives
\begin{align*}
	&h'(x_1,\dots,x_{\mu-1})=
	\left(\frac{c}{x_2}-\frac{cr_1}{x_1^2}, \frac{c}{x_3}-\frac{cx_1}{x_2^2},\dots,
	\frac{c}{x_{\mu-1}}-\frac{cx_{\mu-3}}{x_{\mu-2}^2},
	1-\frac{cx_{\mu-2}}{x_{\mu-1}^2}
	\right).
\end{align*}

This implies that $h$ can be minimized by setting
\begin{equation*}
	\begin{array}{lclcl}
		x_2 & = & x_1^2/r_1 &  & \\
		x_3 & = & x_2^2/x_1 & = &  x_1^3/r_1^2\\
		x_4 & = & x_3^2/x_2 & = &  x_1^4/r_1^3\\
		 & \vdots &  \\
		x_{\mu-1} & = & x_{\mu-2}^2/x_{\mu-3} & = &  x_1^{\mu-1}/r_1^{\mu-2}\\
		x_{\mu-1}^2 & = & cx_{\mu-2}.
	\end{array}
\end{equation*}
Setting 
\[
cx_{\mu-2} = 
cx_1^{\mu-2}/r_1^{\mu-3} =
x_1^{2\mu-2}/r_1^{2\mu-4}
\]
we get
\[
x_1 = \min \left\{\max \left\{ \left(c r_1^{\mu-1}\right)^{\frac{1}{\mu}},1 \right\},c \right\}.
\]
Using that $x_1\geq1$ and $x_{\mu-1} \leq c$ and assuming
$r_2\geq c^{-1/(\mu-1)}$ and $r_2\leq c$ gives
\[
x_i = r_1 \left( \frac{c}{r_1}\right)^\frac{i}{\mu}.
\]

\noindent
This results in an approximation factor for the inner points of
\[
\frac{x_{i+1}}{\sqrt{x_i \cdot x_{i+1}}} = \left(\frac{c}{r_1}\right)^{\frac{1}{2\mu}}
\]

\noindent
For the upper end points the approximation is 
\[
\frac{c}{f(x_1)}
= x_1
= \left(c r_1^{\mu-1}\right)^{\frac{1}{\mu}}.
\]

\noindent
For the lower end points the approximation is 
\[
\frac{c}{x_{\mu}}
=
1.
\]

\noindent
Hence the overall approximation factor for the fourth case is
\[
\max \left\{
\left(\frac{c}{r_1}\right)^{\frac{1}{2\mu}},
\left(c r_1^{\mu-1}\right)^{\frac{1}{\mu}}
\right\},
\]
which finishes the proof.
\end{proof}

\subsection{Numerical evaluation for two specific fronts}

\begin{figure*}[!h]
    \centering
    \subfigure{
        \renumberlinr{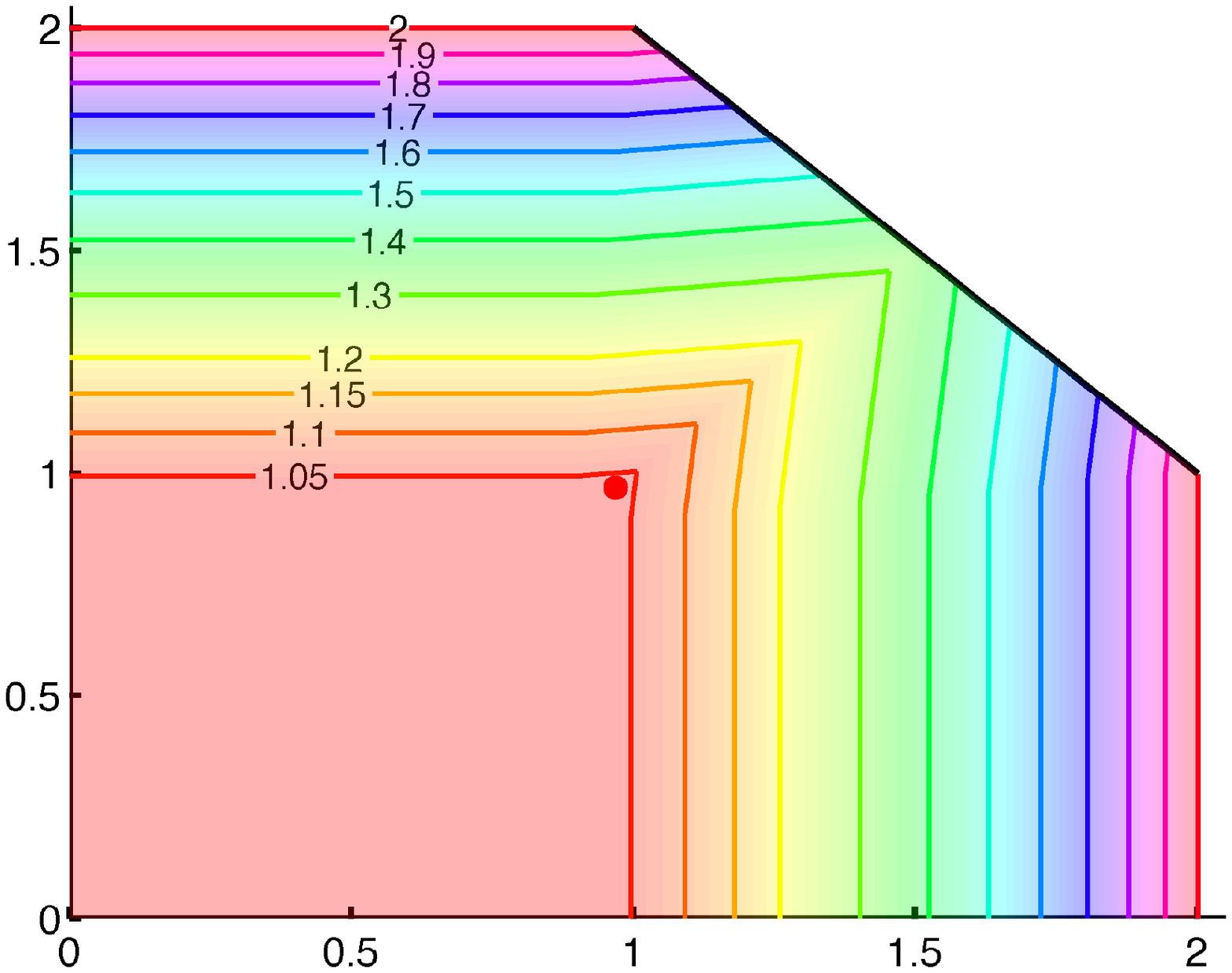}
    }
    \hspace*{.05\textwidth}
    \subfigure{
    	\renumberlinzoom{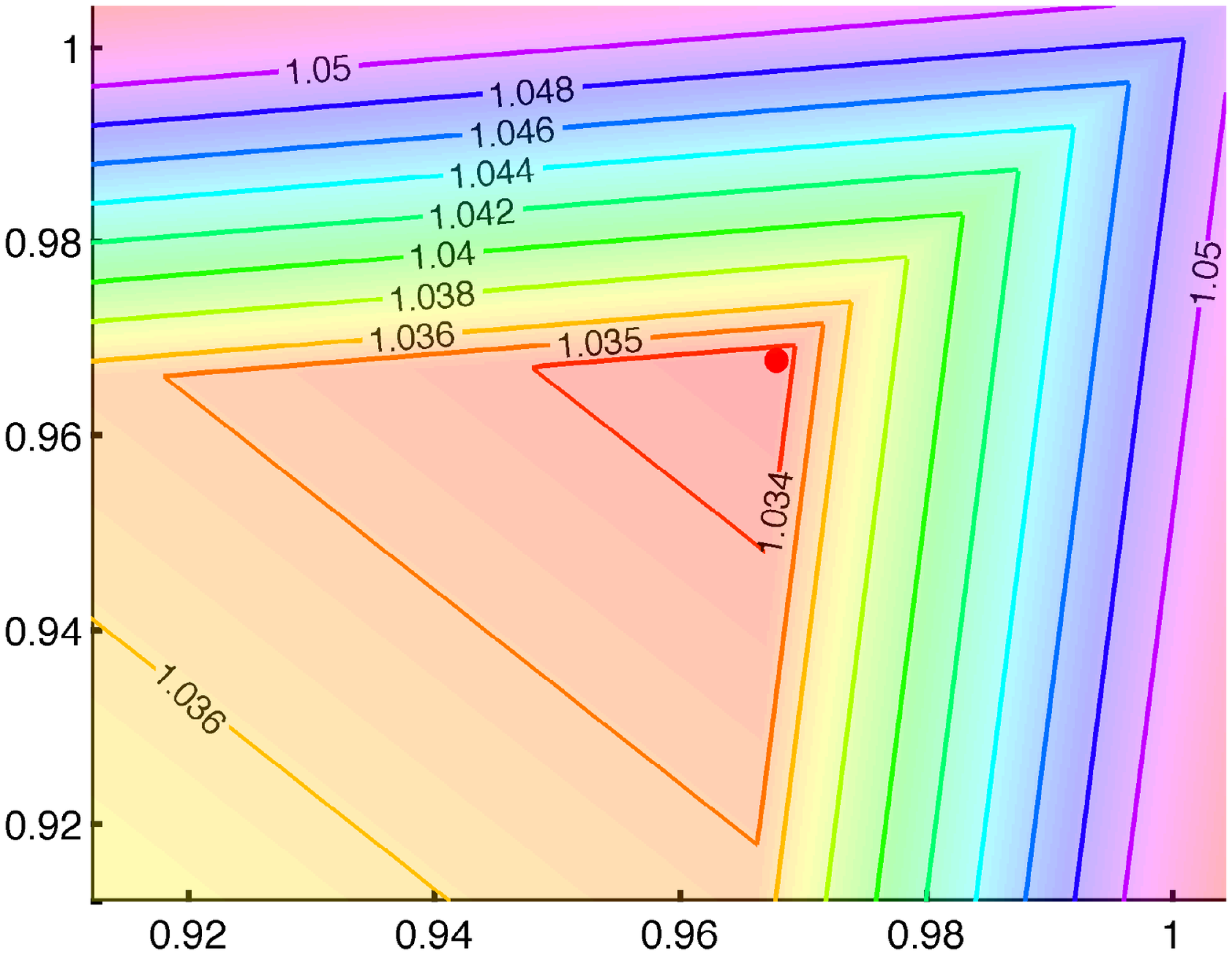}
    }
    \caption{Approximation factor of optimal hypervolume distribution depending
    on the reference point for the linear function
    $f \colon [1,2] \to [1,2]$ with  $f(x)=3-x$ for a population size of $\mu=10$.
    The right figure shows a closeup view of the area around the reference reference point with the best approximation ratio,
    which is marked with a red dot.}
    \label{fig:linr}
\end{figure*}

\begin{figure*}[!h]
    \centering
    \subfigure{
        \renumberconr{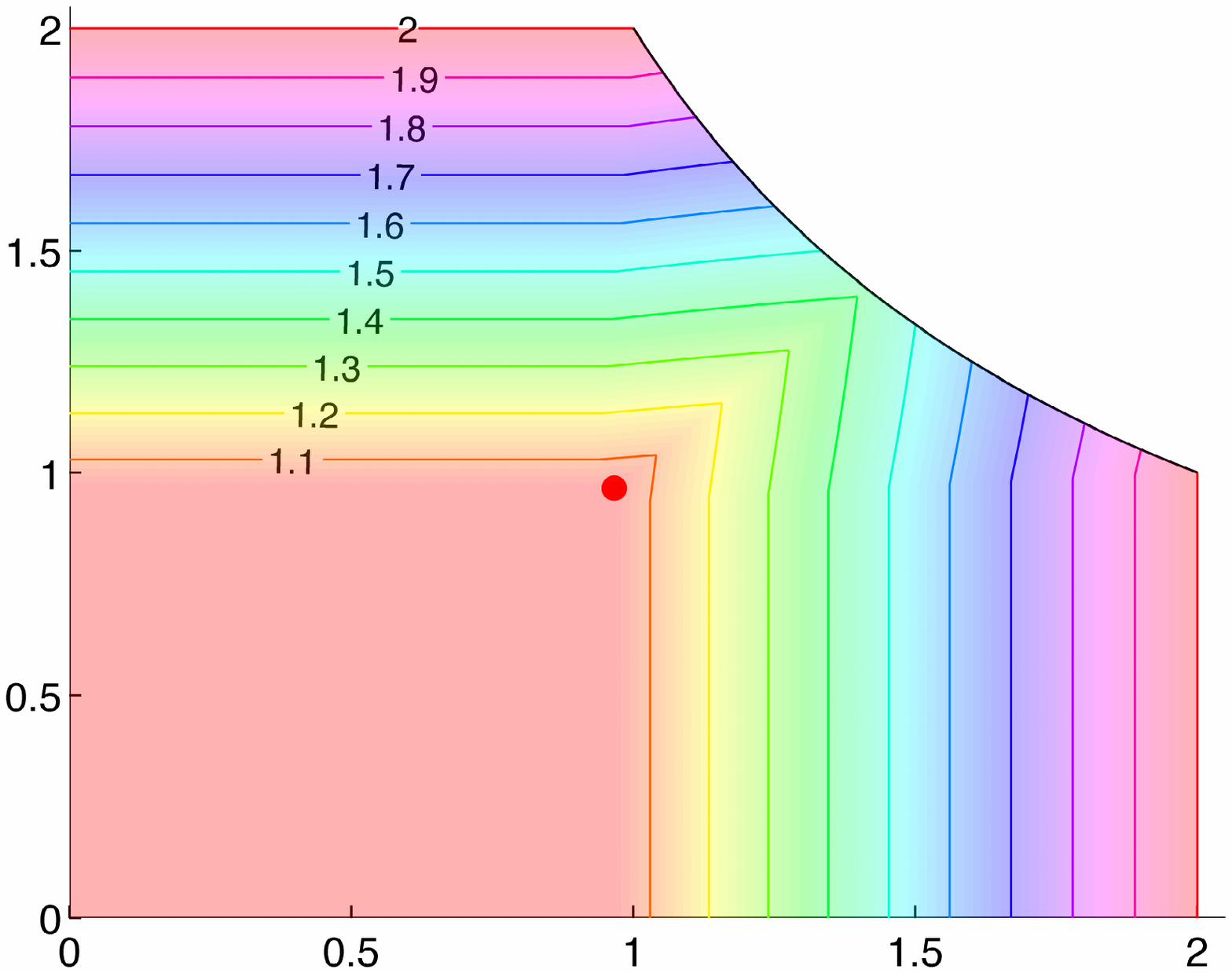}
    }
    \hspace*{.05\textwidth}
    \subfigure{
    	\renumberconzoom{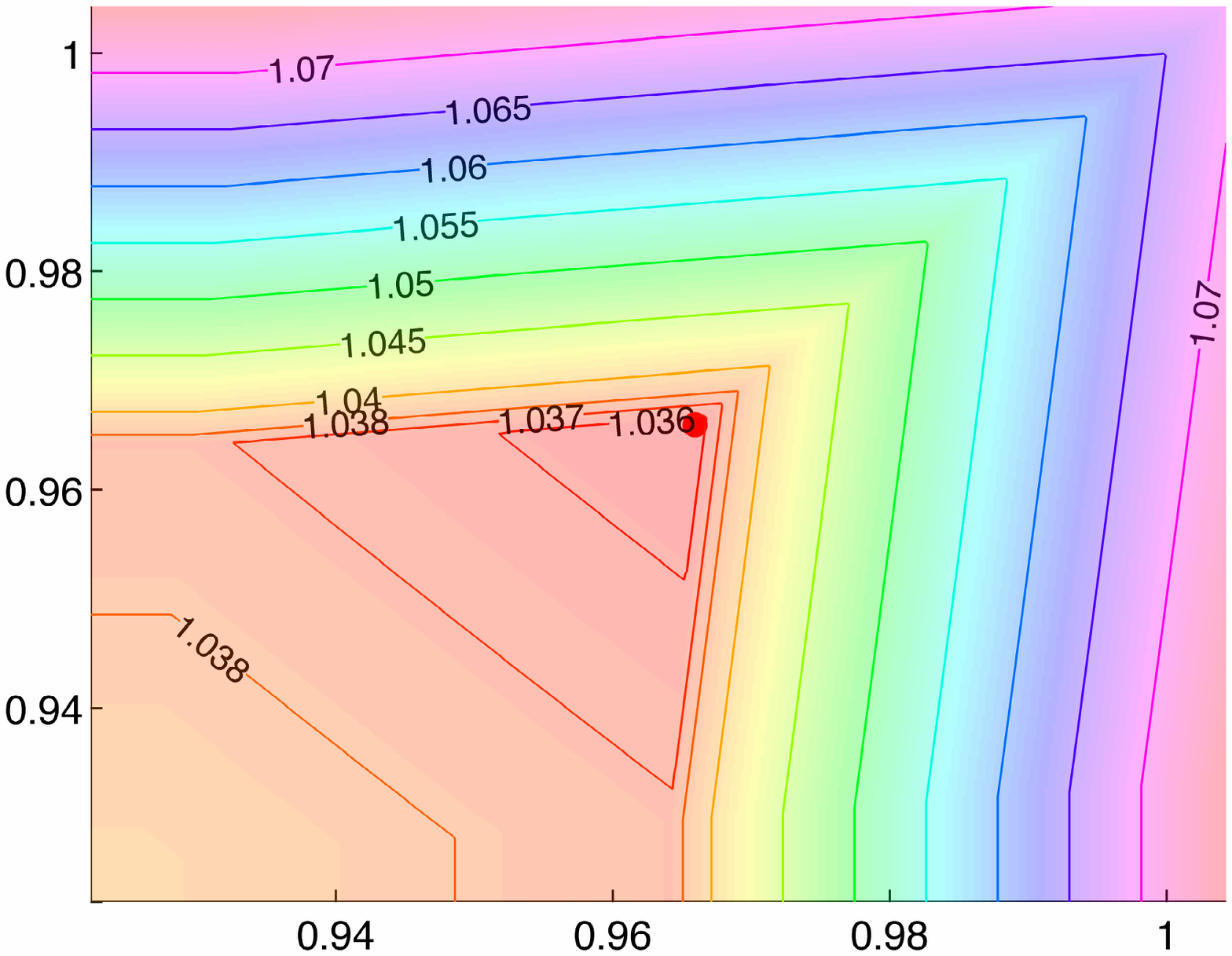}
    }
    \caption{Approximation factor of optimal hypervolume distribution depending
    on the reference point for the convex function
   $f \colon [1,2] \to [1,2]$ with $f(x) = 2/x$ for a population size of $\mu=10$.
    The right figure shows a closeup view of the area around the reference point with the best approximation ratio,
    which is marked with a red dot.}
    \label{fig:conr}
\end{figure*}

We now use the theoretical results of this \secref{depref} on the approximation
factor depending on the reference point and study two specific fronts as an example.

First, we consider 
the linear front
$f \colon [1,2] \to [1,2]$ with  $f(x)=3-x$.
A plot of this front is shown in \figref{lin}~(a).
For $\mu=10$,
\thmref{optlinr} gives that the optimal distribution
\begin{equation*}
\OPTAPPr{\mu}{f} = \left\{
\frac{33}{31},
\frac{36}{31},
\frac{39}{31},
\frac{42}{31},
\frac{45}{31},
\frac{48}{31},
\frac{51}{31},
\frac{54}{31},
\frac{57}{31},
\frac{60}{31}
\right\}
\end{equation*}
achieves the (optimal) approximation of
$
	\APP{\OPTAPPr{\mu}{f}} = 31/30.
$
With \thmref{lina} we can now also determine the approximation factor
of optimal hypervolume distributions depending on the reference point $r$.
For some specific reference points we get
\begin{eqnarray*}
\APP{\OPTHYPr{\mu,r}{f}} = &2 &\text{for $r=(2,1)$,}\\
\APP{\OPTHYPr{\mu,r}{f}} = &4/3 &\text{for $r=(3/2)$,}\\
\APP{\OPTHYPr{\mu,r}{f}} = &22/21 &\text{for $r=(1,1)$,}\\
\APP{\OPTHYPr{\mu,r}{f}} = &31/30 &\text{for $r=(30/31,30/31)$,}\\
\APP{\OPTHYPr{\mu,r}{f}} = &27/26 &\text{for $r\leq (8/9,8/9)$.}
\end{eqnarray*}
\figref{linr} shows a plot of the approximation factor depending on the
reference point.
We observe that if $r_2>32\, r_1-30$ or $r_1>32\,r_2-30$, the approximation factor
is only determined by the inner approximation factor $A_c$ (cf.~\thmref{lina}).
Moreover, for $r_1>10\,r_2-8$ and $r_2>10\, r_1-8$ the approximation factor
only depends  on $r_1$ and $r_2$, respectively.  For $r\leq (8/9,8/9)$ it is constant.
The optimal approximation factor is achieved for the reference point
\[
\Big(
\tfrac{c^2+d\,(c+1+\mu)-1}{c+d\,(\mu+1)-1},
\tfrac{c^2+d\,(c+1+\mu)-1}{c+d\,(\mu+1)-1}
\Big)=
(30/31,30/31).
\]

Let us now consider a specific convex function 
$f \colon [1,c] \to [1,c]$ with $f(x) = c/x$ and $c=2$ for a population size of $\mu=10$.
The function is shown in \figref{con2}.
\thmref{optcon} gives that the optimal distribution 
\[
\OPTAPPr{\mu}{f} = \big\{
2^{\frac{1}{20}},
2^{\frac{3}{20}},
2^{\frac{5}{20}},
\cdots
,
2^{\frac{15}{20}},
2^{\frac{17}{20}},
2^{\frac{19}{20}}
\big\}
\]
achieves the (optimal) approximation factor of
\[
\APP{\OPTAPPr{\mu}{f}} = 2^{1/20}\approx1.0353.
\]
With \thmref{hyprefcon} we can determine the approximation factor
of optimal hypervolume distributions depending on the reference point $r$.
\figref{conr} shows the behavior of the approximation factor depending on the
choice of the reference point $r$. 
We observe that for $r_2>c\,r_1^\mu$ and $r_1>c\,r_2^\mu$, the approximation factor
only depends  on $r_1$ and $r_2$, respectively.  For
\[
	r\leq (c^{1/(1-\mu)},c^{1/(1-\mu)})
	=(2^{-1/9},2^{-1/9})
	\approx (0.926,0.926)
\]
the approximation factor is invariably 
\[
c^{1/(2 \mu-2)}=2^{1/18}\approx1.0393.
\]
The optimal approximation factor
is achieved for the reference point \[
(c^{-1/2\mu},c^{-1/2\mu}) =
(2^{-1/20},2^{-1/20})\approx
(0.966,0.966).
\]


\section{Conclusions}
Evolutionary algorithms have been shown to be very successful for dealing with multi-objective optimization problems. This is mainly due to the fact that such problems are hard to solve by traditional optimization methods. The use of the population of an evolutionary algorithm to approximate the Pareto front seems to be a natural choice for dealing with these problems.
The use of the hypervolume indicator to measure the quality of a population in an 
evolutionary multi-ob\-ject\-ive algorithm has become very popular in recent years. 
Understanding the optimal distribution of a population consisting of $\mu$ 
individuals is a hard task and the optimization goal when using the hypervolume 
indicator is rather unclear. Therefore, it is a challenging task to understand 
the optimization goal by using the hypervolume indicator as a quality measure 
for a population. 

We have examined how the hypervolume indicator approximates 
Pareto fronts of different shapes and related it to the best possible 
approximation ratio. 
We started by considering the case where we assumed that the extreme points with respect to the given objective functions have to be included in both distributions.
Considering linear fronts and a class of convex 
fronts we have pointed out that the hypervolume indicator gives provably the 
best multiplicative approximation ratio that is achievable. To gain further insights into the optimal hypervolume distribution and its relation to multiplicative approximations, we carried out numerical investigations. These 
investigations point out that the shape as well the scaling of the objectives 
heavily influences the approximation behavior of the hypervolume indicator. 
Examining fronts with different shapes we have shown that the approximation 
achieved by an optimal set of points with respect to the hypervolume may differ 
from the set of $\mu$ points achieving the best approximation ratio.

After having obtained these results, we analyzed the impact of the reference points on the hypervolume distribution and compared the multiplicative approximation ratio obtained by this indicator to the overall optimal approximation that does not have to contain the extreme points. Our investigations show that also in this case the hypervolume distribution can lead to an overall optimal approximation when the reference point is chosen in the right way for the class of linear and convex functions under investigation. Furthermore, our results point out the impact of the choice of the reference point with respect to the approximation ratio that is achieved as shown in Figures~\ref{fig:linr} and \ref{fig:conr}. 

Our results provide insights into the connection of the optimal hypervolume distribution and approximation ratio for special classes of functions describing the the Pareto fronts of multi-objective problems having two objectives.
For future work, it would be interesting to obtain results for broader classes of functions as well as problems having more than $2$ objectives.



\bibliography{../../hypervol}

\end{document}